\documentclass{article}
\usepackage{fullpage}

\usepackage{natbib}
\usepackage[utf8]{inputenc} % allow utf-8 input
\usepackage[T1]{fontenc}    % use 8-bit T1 fonts
\usepackage{hyperref}       % hyperlinks
\usepackage{url}            % simple URL typesetting
\usepackage{booktabs}       % professional-quality tables
\usepackage{nicefrac}       % compact symbols for 1/2, etc.
\usepackage{microtype}      % microtypography
\usepackage{xcolor}         % colors
\usepackage{amsthm,amsfonts,amssymb}
\usepackage{xspace}
\usepackage{bbm}
\usepackage{subcaption}
\usepackage{etoolbox}% http://ctan.org/pkg/etoolbox
\usepackage{enumitem}
\usepackage{authblk}
\usepackage[noend]{algorithmic}
\usepackage[ruled,vlined]{algorithm2e}
\usepackage{comment}
\usepackage{amsmath}
\usepackage{accents}
\usepackage{graphicx} % Allows including images
\usepackage{caption}
\usepackage{pgfplots}
\usepackage{thmtools} 
\usepackage{thm-restate}
\usepackage{enumitem}
\usepackage{mathtools}
\usepackage{tikz}
\pgfplotsset{compat=1.10}
\usetikzlibrary{quotes,positioning}
\allowdisplaybreaks

\newtheorem{lemma}{Lemma}[section]

\newtheorem{definition}{Definition}[section]
\newtheorem{remark}{Remark}[section]

\newtheorem{ass}{Assumption}[section]

\newcommand{\sw}[1]{}
\newcommand{\ed}[1]{}
\newcounter{numrellocal}% Local counter for numering relations
\newcommand{\at}[1]{}
\newcommand{\ar}[1]{}
\newcommand{\jwd}[1]{}

\newcommand{\cD}{\mathcal{D}}

\renewcommand{\thenumrellocal}{\arabic{numrellocal}}% Counter numrellocal uses lowercase roman numerals
\newcounter{numrelglobal}% Global counter for numering relations
\makeatletter
\newcommand{\numrel}[2]{% Relation numbering
  \stepcounter{numrellocal}% Increment local counter
  \refstepcounter{numrelglobal}% Increment global counter and create correct reference hook, with label-text from local counter
  \ltx@label{#2}% Label numrel counter
  \overset{(\thenumrellocal)}{#1}% Print counter + relation
}
\makeatother
\title{Correcting Underrepresentation and Intersectional Bias for Classification}

\author[1]{Emily Diana}
\author[2]{Alexander Williams Tolbert}
\affil[1]{Toyota Technological Institute at Chicago}
\affil[2]{Emory University}
\begin{document}

\maketitle

\begin{abstract}
We consider the problem of learning from data corrupted by underrepresentation bias, where positive examples are filtered from the data at different, unknown rates for a fixed number of sensitive groups. We show that with a small amount of unbiased data, we can efficiently estimate the group-wise drop-out rates, even in settings where intersectional group membership makes learning each intersectional rate computationally infeasible. Using these estimates, we construct a reweighting scheme that allows us to approximate the loss of any hypothesis on the true distribution, even if we only observe the empirical error on a biased sample. From this, we present an algorithm encapsulating this learning and reweighting process along with a thorough empirical investigation. Finally, we define a bespoke notion of PAC learnability for the underrepresentation and intersectional bias setting and show that our algorithm permits efficient learning for model classes of finite VC dimension.

%In particular, we provide strong PAC-style guarantees that, with high probability, our estimate of the risk of the hypothesis over the true distribution will be arbitrarily close to the true risk. 

%\ed{Add some words on showing that it is agnostically PAC-learnable and experiments.}

\end{abstract}

\section{Introduction}

Intersectionality is a concept that was introduced in the 1980s by the legal scholar \citet{crenshaw1989demarginalizing} to describe how multiple forms of oppression, such as racism, sexism, and classism, intersect to create unique experiences of discrimination for individuals who are members of multiple marginalized groups. However, there are different models on which intersectionality can be interpreted. \citet{curry2018killing} identifies at least two types that have dominated the literature: \textit{interactive} and \textit{additive} intersectionality. The \textit{additive} model posits ``a person with two or more intersecting identities experiences the distinctive forms of oppression associated with each of his or her subordinate identities summed together. The more devalued identities a person has, the more cumulative discrimination he or she faces'' \citep[p.~25-26]{curry2018killing}.\footnote{Our formal model is more closely aligned with the \textit{additive} model of intersectionality.} The \textit{interactive} model argues ``each person’s subordinate identities interact synergistically. People experience these identities as one, thus contending with discrimination as a multiply marginalized other'' \citep[p.~25-26]{curry2018killing}. These models aimed to predict that people with multiple subordinate identities would be subjected to more prejudice and discrimination than those with a single subordinate identity, supported by findings on various economic and social indicators \citep[p.~25-26]{curry2018killing}.

Research has shown that black males have been unfairly incarcerated on drug charges at significantly higher rates than white males. This example illustrates intersectional bias based on the intersection of race and gender, as black males experience discrimination that is distinct from both black females and white males \citep{alexander2011new,curry2017man}. These biased outcomes are often the result of historical data that reflect discriminatory patterns or biased decision-making processes. When these biased data are used to train algorithms, they perpetuate and amplify existing inequalities, further disadvantaging individuals who belong to multiple marginalized groups.  This reality confronts us with the philosophical distinction between the world as it appears in our datasets—distorted by bias—and the unobserved world that represents the true state of affairs. The philosophy of science distinguishes between \textit{observables} that we can see and the ground truth that we cannot observe \citep{bird1998philosophy}.

Building on this premise, our approach in this paper is to harness the insight that while we can readily observe certain distributions, they may be distorted by the biases inherent in data-generating processes. In particular, we consider the problem of learning from data corrupted by underrepresentation bias, where positive examples are filtered from the data at different, unknown rates for a fixed number of sensitive groups. Interestingly, there are many situations that arise in which one may actually have access to a small amount of unbiased data in addition to the much larger set of biased data (which we later refer to as a 'two-batch setting'). For example, consider the CUNY open admissions program, which ``would guarantee every high school graduate a seat in a community college, to be phased in from 1971 to 1975'' \cite{cuny}. During the years with a more traditional college admission process for CUNY, before and after this program, many potentially successful applicants from disadvantaged groups may have been filtered out or never applied. Therefore, their success could not have been observed (thus creating a biased dataset). However, a smaller, unbiased data set could be curated from the open admissions period and potentially used to model the biased data-generating mechanism.

Understanding this mechanism of bias and explicitly modeling it becomes a powerful tool, enabling us to use the biased data to approximate and eventually recover the true distribution which we can then use to train machine learning models. We show that with a small amount of unbiased data, we can efficiently estimate the group-wise drop-out rates, even in settings where intersectional group membership makes learning each intersectional rate computationally infeasible. Using these estimates, we construct a reweighting scheme that allows us to approximate the loss of any hypothesis on the true distribution, even if we only observe the empirical error on a biased sample. From this, we present an algorithm encapsulating this learning and reweighting process along with a thorough empirical investigation. Finally, we define a bespoke notion of PAC learnability for the underrepresentation and intersectional bias setting and show that our algorithm permits efficient learning for model classes of finite VC dimension.

\subsection{Related Work}
Significant research has been conducted on machine learning techniques to predict outcomes in the face of corrupted data \citep{angluin1988learning,cesa1999sample,diakonikolas2019distribution}. In addition, many approaches have been developed to correct data imbalances for machine learning applications. Over-sampling methods used to re-balance data include ADASYN \cite{He2008ADASYNAS}, MIXUP \cite{mixup}, and SMOTE \cite{smote} (including several adaptations \cite{Nguyen2009BorderlineOF, Batista2003BalancingTD, 10.1145/1007730.1007735, Douzas_2018, Han2005BorderlineSMOTEAN, Menardi2012TrainingAA}).  
Another line of research \cite{CondensedNearestNeighbor, Wilson1972AsymptoticPO, 4309523, kNN, tomek2} studies cluster-based approaches to under-sampling classes that are overrepresented in the data.  Propensity score re-weighting \cite{propensityScoreReweighting} is also a popular approach used in the causal literature to account for group size differences. Finally, generative AI methods have recently been used for data augmentation and to address data-imbalance concerns, as studied in \cite{10.1007/978-3-030-90963-5_39, 8869910}. Each of these techniques, however, requires knowledge of the target distribution over these different groups, whereas our approach uses the existence of a small sample of data from the true distribution to learn this distribution and then reweight the biased dataset during training to better mimic the true distribution. Another closely related area of study is domain adaptation, which broadly studies distribution mismatches between train and test sets in statistical and machine learning applications. For example, \cite{Lipton2018} provide a re-weighting algorithm in the case of label shift bias. \cite{NEURIPS2019_b4189d9d} also later extended work on representation learning in domain adaptation to study fairness under different types of parity.

In addition to data imbalances, several approaches have been devised to address performance imbalances in machine learning algorithms. Some of the most common approaches involve incorporating demographic constraints into the learning process, which define the standards to which a fair classifier should adhere \citep{chouldechova2018frontiers}. Commonly, these approaches segment communities based on relevant characteristics, and the classifier is expected to exhibit comparable performance across all protected demographic groups \citep{chouldechova2017fair}, \citep{hardt2016equality}, \citep{kamiran2009classifying}. Ensuring equality in predictive outcomes based on group membership can be pursued through fairness metrics or definitions, which encompass various aspects of parity constraints. Several definitions have been proposed in the literature to capture different facets of fairness \citep{chouldechova2017fair}, \citep{dwork2012fairness}, \citep{hardt2016equality}. A substantial amount of research has been dedicated to exploring the interrelationships between these various fairness definitions and identifying conflicts among them \citep{kleinberg2016inherent} \citep{chouldechova2017fair}. Several recent works in the machine learning community have also begun to explicitly study intersectionality in machine learning from an empirical perspective, including \cite{10.1145/3593013.3593979} and \cite{10.1145/3531146.3533101}.  A developing area of research within the domain of fair machine learning, which bears significant relevance to our study, revolves around the concept that observed data does not fully capture the underlying unobserved data. This line of inquiry explores how enforcing fairness can aid machine learning models in mitigating biases \citep{kleinberg2018selection, blum2019recovering}. \cite{NEURIPS2021_8b0bb3ef} also defines a notion of fair-PAC learning with respect to different group fairness notions, and \cite{Sicilia2023LearningTG} study correction techniques for under-representation bias in text generation algorithms.

\subsection{Contributions and Novelty}
In their research surveying the role of intersectionality in quantitative studies, \citet{bauer2021intersectionality} discovered that the engagement with the fundamental principles of intersectionality was frequently superficial. They reported that 26.9\% of the papers did not define intersectionality, 32.0\% did not cite the pioneering authors, and 17.5\% of the papers utilized ``intersectional'' categories that were not explicitly linked to social power. However, linking intersectionality to social power has been critiqued in the general philosophy of science and philosophy of social science literature due to a purported lack of empirical falsifiability and potential ad-hoc or arbitrary hypotheses \cite{bright2016causally, curry2018killing}.  

Our model contributes to this discourse by providing a formal mathematical framework from which we can derive strong theoretical and empirical results. In particular, we clearly define a previously vague problem and problem setting, and we present a refined methodological approach to tackle it with provable guarantees. (This is in contrast to presenting a more technically sophisticated or assumption-light solution to an existing problem.) Our contribution's essence is the formulation of a bias model that eschews assumptions of uniform base rates and non-overlapping groups used in earlier works. This choice involves the trade-off of assumptions. Unlike Blum and Stangl's analysis, which hinges on the uniformity of base rates and the separation of groups for validity, our model forgoes these constraints in favor of an independence assumption. This crucial adjustment permits examining overlapping groups without requiring uniform base rate assumptions.

One benefit of the PAC-style techniques that we rely upon in our analysis is that they are very popular within the theoretical fairness research community, so our analysis and model have the potential to be easily built upon by others. Our research distinctively enhances the PAC learning model by deriving theoretical bounds for data categorized into two specific types of batches: those with bias and those without. This nuanced application goes beyond traditional PAC learning by meticulously analyzing how these different data sources—each with its own characteristics of bias or lack thereof—affect the learning process. Our approach not only addresses the complexity of bias in machine learning datasets but also expands the PAC framework in a subtle yet important way, showcasing a deeper, more detailed exploration of data bias through innovative theoretical contributions. 

Finally, contrary to many of the works cited, which primarily focus on experimental approaches without establishing a comprehensive mathematical framework or clear definitions for intersectional and underrepresentation biases, our research marks a departure from these methodologies. Drawing upon the theoretical foundations laid by Blum and Stangl, which have spurred a considerable volume of related literature, our study endeavors to advance these theoretical underpinnings. We propose a theoretical framework that not only rigorously defines underrepresentation bias within an intersectional context but also systematically addresses these biases. This approach allows our research to stand out by providing both theoretical depth and a methodical framework for examining intersectional and underrepresentation biases (verified experimentally), thereby contributing novel insights to the field.

Our research fundamentally diverges from the cited literature, which only superficially resembles our focus on bias or underrepresentation. Our work uniquely characterizes dropout bias through rigorous mathematical formulation, distinct from label shift or methodologies in the referenced papers, which largely lack mathematical rigor in models, definitions, or theorems. Our reweighting strategy and theoretical contributions, such as deriving bounds within the PAC model for different batch settings and specifically addressing biased and unbiased data, set our research apart. Our approach, especially in contrast to methods like those in generative adversarial networks, underscores a novel, mathematically grounded methodology with theoretical implications not explored in the suggested literature. Our study distinctively advances a separate strand of literature, focusing on a rigorous mathematical treatment of dropout bias, in contrast to the primarily experimental works cited. While those works offer valuable insights, our work is different.

\section{Preliminaries and Model Overview}
\label{sec:prelims}
In this section, we introduce the mathematical framework that underpins our model. We aim to capture the intersectionality of biases within different groups, recognizing that biases can be multifaceted and interconnected. We begin by defining the spaces over which our model operates. Let $\mathcal{X}$ be the feature space, representing the set of all possible $d$-dimensional feature vectors $\textbf{x}$. The label space, $\mathcal{Y}$, consists of binary labels $y \in \{0,1\}$, representing two distinct classes or outcomes. The joint distribution of feature-label pairs $(\textbf{x},y)$ is denoted by the distribution $\mathcal{D}(\textbf{x}, y): \mathcal{X} \times \mathcal{Y} \nonumber$.

Importantly, our data domain consists of $k$ distinct groups $\{G_i\}_{i=1}^{k}$. For simplicity, we can imagine that each group $G_i$ comes with an inclusion function $g_i:\mathcal{X} \rightarrow \{0,1\}$ to indicate whether or not a sample $x$ is in the group $g_i$. Therefore, we can write $G_i = \{x \in \mathcal{X}: g_i(x) = 1\}$ and $x \in G_i$ if $g_i(x) = 1$. These groups can intersect, meaning that a sample $\textbf{x}$ might belong to multiple groups. This leads to $2^k$ possible subsets of $\{G_i\}_{i=1}^{k}$, capturing the concept of intersectionality.\footnote{Our approach to frame the analysis around $k$ intersectional groups and estimating marginal retention parameters while considering the combinatorial effects of multiple group memberships enables us to circumvent the computational complexity of estimating an exponential number of parameters for every possible $2^k$ group combinations. This methodology remains efficient and scalable, particularly for larger k, ensuring practical applicability without compromising on the accuracy of estimates and corrections as the number of intersectional groups increases.} The function $G(\textbf{x})$ is used to indicate the set of groups to which a feature vector $\textbf{x}$ belongs, and we use the notation $\lvert G(\textbf{x}) \rvert$ to indicate the number of groups for which $\textbf{x}$ is a member. Formally, for each $\textbf{x} \in \mathcal{X}$, $G(\textbf{x}) \stackrel{\text{def}}{=} \{ i | \textbf{x} \in G_i \}$. We use $p_i \stackrel{\text{def}}{=} \Pr[y = 1 | \textbf{x} \in G_i]$ to denote the base positive rate, or the inherent likelihood of a positive outcome in group $G_i$, and we use $p_0 \stackrel{\text{def}}{=} \Pr[y = 1]$ to denote the positive rate in the population. We will write $\mathbf{p} \stackrel{\text{def}}{=}(p_0, p_1, p_2, \ldots, p_k)$ to indicate the vector consisting of individual group base positive rates as well as the population base positive rate. Finally, we define \( m_i \) as the number of samples in \( S \) from group \( G_i \)

\subsection{Inclusion of Non-Member Samples and Independence Assumptions}

In addition to the $k$ groups $\{G_i\}_{i=1}^{k}$, we also assume the existence of a set $C$ composed of samples that are not members of any of the $k$ groups. This allows us to make the following assumptions about the independence of group membership:

\begin{ass}[Independence of Group Membership]
\label{ass:indep}
We assume that group membership is mutually independent. That is, for $I \subseteq [k]$:
\[
\Pr[\textbf{x} \in \cap_{i \in I} G_i] = \prod_{i \in I} \Pr[\textbf{x} \in G_i]
\]
\end{ass}

This assumption asserts that the probability of a sample belonging to an intersection of groups is the product of the probabilities of belonging to each individual group. It ensures that the groups are treated independently of each other.

\paragraph{Justification of the Independence Assumption:}
While this assumption may not perfectly capture the complexities of real-world interactions among different group memberships, it is a necessary simplification for several reasons:

\begin{itemize}
    \item \textbf{Model Simplicity:} It simplifies the mathematical model, making it more tractable and manageable. Modeling the interactions between all possible group memberships would significantly increase the complexity and computational demands of the model.
    \item \textbf{Logical in Certain Contexts:} In cases like sex and race, this assumption is logical for analytical purposes. Knowing an individual’s sex does not inherently provide information about their race, and vice versa, making these variables independent in many statistical analyses.
    \item \textbf{Empirical Basis:} The true nature of the relationships between different group memberships is ultimately an empirical question. This assumption serves as an initial hypothesis to be tested and refined based on data.
    \item \textbf{Foundation for Future Research:} This assumption provides a starting point for further exploration. It enables initial analysis and understanding, which can be built upon with more complex models that consider interdependencies.
\end{itemize}

We emphasize that it would be nice to eventually remove or weaken some of these assumptions, and we hope that others can build on this work to do so. Our choice to assume independence of group membership essentially allowed us to write the probability that an individual was filtered out to be proportional to the drop-out probabilities for each group he or she was a member of. Then, we would only have to estimate the drop-out parameter for each group separately, rather than estimating drop-out probabilities for each of the $2^k$ possible combination of group memberships (in the worst case). Making this assumption also allowed us to drop several other assumptions from previous works and calculate rigorous PAC-style bounds (which, while a more traditional technique, does allow us to make meaningful statements about the behavior of our estimates with high probability). Therefore, while recognizing its limitations, we utilize this assumption as a pragmatic approach to begin our exploration of intersectional biases. It is intended as an initial step, inviting future research to empirically test and potentially revise this assumption. Please see the Appendix for additional discussion on this assumption.  

Next, in Assumption~\ref{ass:cond_indep}, we extend the assumption of independence to the conditional case where labels are positive, ensuring that the groups are treated independently even when conditioned on this outcome. 
\begin{ass}[Conditional Independence of Group Membership]
\label{ass:cond_indep}
Given that $y=1$, we assume that group membership is mutually independent. Formally, for $I \subseteq [k]$:
\[
\Pr[\textbf{x} \in \cap_{i \in I} G_i|y=1] = \prod_{i \in I} \Pr[\textbf{x} \in G_i|y=1]
\]
\end{ass}

\paragraph{Justification of Conditional Independence Assumption:}
The assumption of conditional independence for positive outcomes (\(y=1\)) is pivotal for several reasons:

\begin{itemize}
    \item \textbf{Simplification for Probability Calculation:} This assumption allows for the simplified calculation of \(\Pr(y = 1|\textbf{x} \in \bigcap_{i \in I} G_i)\), the probability of a positive outcome given membership in intersecting groups. Without this assumption, the probability would need to factor in the complex interdependencies between groups, which would complicate the model significantly.
 
    \item \textbf{Empirical Testing Basis:} While it is a simplifying assumption, it forms an empirical hypothesis that should be tested with real-world data. 
\end{itemize}

This assumption, by simplifying the interaction model between groups for positive outcomes, provides a manageable approach for initial analysis. It is crucial, however, to acknowledge that this simplification may not capture all the nuances of real-world group dynamics. Future empirical research should investigate this assumption's validity and explore the complexities of group interdependencies, particularly in scenarios of positive outcomes. 

Using Assumptions~\ref{ass:indep} and ~\ref{ass:cond_indep}, we can now derive Lemma~\ref{lem:prod_pos_rate}, which expresses the probability that an example is positive given its group membership. The proof for this lemma, and all other proofs, can be found in the Appendix.
\begin{restatable}{lemma}{prodPosRate}
\label{lem:prod_pos_rate}
The positive rate of samples belonging to a specific intersection of groups can be calculated from the marginal positive rates of those groups and the overall positive rate as, $\forall I \subseteq [k]$:
\[\Pr[y=1|\textbf{x} \in \bigcap_{i \in I} G_i] = \Pr[y=1]^{1-\lvert I \vert } \prod_{i \in I} p_i \]
\end{restatable}

\begin{remark}
It is worth noting that when a sample is a member of only one group, this expression reduces to $\Pr[y=1|\textbf{x} \in G_i]=p_i$ as expected. The presence of the overall positive rate in the expression for Lemma \ref{lem:prod_pos_rate} essentially normalizes each groupwise positive rate, preventing the overall expression from becoming too small.
\end{remark}
\subsection{Bias Parameter and Biased Training Dataset}

We now introduce the concept of bias within the training dataset. Let $\mathcal{S}$ be an unbiased dataset of size $m$ drawn i.i.d. from the distribution $\mathcal{D}$. Now let $S_{\beta} = {(\textbf{x}_1,y_1),\ldots,(\textbf{x}_{m_{\beta}},y_{m_{\beta}})} \subset S$ be a biased training dataset of size $m_{\beta} \leq m$ and $m_{\beta_i}$ be the number of samples in \( S_\beta \) from group \( G_i \). To model underrepresentation bias, we associate each group $G_i$ with a bias parameter $\beta_i > 0$, representing the probability that a positive sample $(\textbf{x},y) \in S$ will be retained in $S_\beta$, and we let $\beta_0$ indicate the overall retention rate of positive samples in $S$.
%\ed{More substantial discussion -- maybe use an example.}
\begin{equation*}
    \beta_i \stackrel{\text{def}}{=} \Pr[(\textbf{x},y) \in S_\beta | \textbf{x} \in G_i, y=1], \; \beta_0 \stackrel{\text{def}}{=} \Pr[(\textbf{x},y) \in S_\beta | y=1]
\end{equation*}

We will write $\mathbf{\beta} = (\beta_0, \beta_1,...\beta_k)$ to indicate the vector consisting of both the population retention rate and individual group retention rates. Importantly, \textit{negative samples are always retained}, so $\Pr[(\textbf{x},y) \in S_\beta | y=0] = 1$. Note that the bias parameter $\beta_i$ quantifies the extent to which positive samples from group $G_i$ are retained in the biased training dataset $S_\beta$. Therefore, the quantity $1-\beta_i$ can be interpreted as the rate at which positive samples from group $G_i$ are filtered out when moving from $S$ to $S_\beta$. We now define the base positive rate observed for each group in the biased sample, which will allow us to ultimately derive an expression that we can use to efficiently estimate $\frac{1}{\beta}$.

\begin{definition}[Biased Base Positive Rate for Group $i$]
\label{def:biased_base_rate}
The biased base positive rate for group \( G_i \) in the biased dataset \( S_\beta \) is given by $p_{\beta_i} \stackrel{\text{def}}{=}  \Pr[y=1 | (\textbf{x},y) \in S_\beta, \textbf{x} \in G_i]$.
\end{definition}

Next, we define the biased base positive rate for the entire population, $p_{\beta_0}$. We will use $\mathbf{p_{\beta}} \stackrel{\text{def}}{=}(p_{\beta_0}, p_{\beta_1}, \ldots, p_{\beta_k})$ to indicate the vector of biased positive rates.
\begin{definition}[Biased Base Positive Rate for Population]
\label{def:biased_base_pos_rate}
The biased base positive rate for the population is given by $p_{\beta_0} \stackrel{\text{def}}{=}  \Pr[y=1 | (\textbf{x},y) \in S_\beta]$.
\end{definition}

With these terms defined, we can now express the inverse retention rate for each group solely in terms of the biased and unbiased positive rates.
\begin{restatable}{lemma}{inverseBeta}
\label{lem:inverseBeta}
The inverse of each bias parameter can be calculated solely using the unbiased positive rate and biased positive rate for the respective group as follows:
\begin{align*}
\beta_i^{-1} = p_i \left( 1 - p_{\beta_i} \right) p_{\beta_i}^{-1}\left( 1 - p_i \right)^{-1}
\end{align*}
\end{restatable}

% &\implies \beta_i p_i \left( 1 - p_i^* \right) = p_i^*(1 - p_i) \\
% &\implies \beta_i = \frac{p_i^*(1 - p_i)}{p_i \left( 1 - p_i^* \right)} \\
\begin{ass}[Conditional Independence of Group Membership in Biased Sample]
\label{ass:cond_indep_beta}
Given that $y=1$ and $(\textbf{x},y) \in S_\beta$, we assume that group membership is mutually independent. Formally, for $I \subseteq [k]$:
\begin{align*}
&\Pr[\textbf{x} \in \cap_{i \in I} G_i|(\textbf{x},y) \in S_\beta, y=1] = \prod_{i \in I} \Pr[\textbf{x} \in G_i|(\textbf{x},y) \in S_\beta, y=1]
\end{align*}
\end{ass}

Using Assumptions~\ref{ass:cond_indep} and ~\ref{ass:cond_indep_beta}, we can derive Lemma~\ref{lem:prod_beta}, which provides an expression for the probability that a positive example is included in $S_\beta$ given its group membership.

\begin{restatable}{lemma}{prodBeta}
\label{lem:prod_beta}
The probability that a positive example is included in $S_\beta$ given its group membership can be calculated solely from the bias parameters as follows, where $I \subseteq [k]$:
\begin{align*}
&\Pr[(\textbf{x},y) \in S_\beta|\textbf{x} \in \cap_{i \in I} G_i, y=1] = \beta_0^{1-\lvert I\rvert} \prod_{i \in I} \beta_i
\end{align*}
\end{restatable}

\begin{remark}
Notice that if a sample $(\textbf{x},y)$ is only in one group, $G_i$, this simplifies to $\Pr[(\textbf{x},y) \in S_\beta|\textbf{x} \in \cap_{i \in I} G_i, y=1] = \beta_i$, as desired. The term $\beta_0$ plays a similar role to the overall positive rate in Lemma~\ref{lem:prod_pos_rate}, essentially normalizing the drop-out rates to keep the overall expression from becoming too small.
\end{remark}

\subsection{The Biased Distribution}

We now define $\cD_{\beta}(\textbf{x}, y): \mathcal{X} \times \mathcal{Y}$ as the distribution induced on \(S_{\beta}\) by the filtering process. This distribution is influenced by the bias parameters \(\beta\), reflecting the underrepresentation of certain groups in the biased sample. To define $\cD_\beta$, we write the accompanying probability density function, $p_{\cD_\beta}$, as follows:

\begin{align*}
p_{\cD_{\beta}}(\textbf{x},y) &\stackrel{\text{def}}{=} \Pr[(\textbf{x},y)=(\textbf{X},Y)|(\textbf{x},y) \in S_{\beta}] \\& = \frac{\Pr[(\textbf{x},y) \in S_{\beta}|(\textbf{x},y)=(\textbf{X},Y)]\Pr[(\textbf{x},y)=(\textbf{X},Y)]}{\sum_{(\textbf{X},Y)}\Pr[(\textbf{x},y) \in S_{\beta}|(\textbf{x},y)=(\textbf{X},Y)]\Pr[(\textbf{x},y)=(\textbf{X},Y)]}\\&=\frac{\Pr[(\textbf{x},y) \in S_{\beta}|(\textbf{x},y)=(\textbf{X},Y)]p_\cD(\textbf{x},y)}{\sum_{(\textbf{X},Y)}\Pr[(\textbf{x},y) \in S_{\beta}|(\textbf{x},y)=(\textbf{X},Y)]p_\cD(\textbf{X},Y)}\\
&=\frac{\Pr[(\textbf{x},y) \in S_{\beta}|(\textbf{x},y)=(\textbf{X},Y)]p_\cD(\textbf{x},y)}{\mathbb{E}_\cD [\Pr[(\textbf{x},y) \in S_{\beta}|(\textbf{x},y)=(\textbf{X},Y)]]}
\end{align*}
For ease of notation in relating the original distribution \(\cD\) and the biased distribution \(\cD_\beta\), we introduce a \textit{reweighting function} \(w(\textbf{x},y)\), representing the inverse of the probability that a given sample is retained:

\begin{align*}
w(\textbf{x},y) &\stackrel{\text{def}}{=}  \frac{1}{\Pr[(\textbf{x},y) \in S_{\beta}|(\textbf{x},y)=(\textbf{X},Y)]}\\& = \mathbb{I}(y=0) + \frac{\mathbb{I}(y=1)}{\sum_{I \subseteq [k]}\mathbb{I}(G(\textbf{x})=I)\Pr[(\textbf{x},y) \in S_\beta|\textbf{x} \in \cap_{i \in I} G_i, y=1]} \\& =\mathbb{I}(y=0) + \mathbb{I}(y=1)\beta_0^{\lvert G(\textbf{x})\rvert - 1} \prod_{i \in G(\textbf{x})} \frac{1}{\beta_i}
\end{align*}

This allows us to relate $\cD$ and $\cD_\beta$ as follows:

\begin{restatable}{theorem}{thmDistributions}\label{thm:biased_pmf}
The joint probability mass function of the biased distribution \(\cD_\beta\) is related to \(\cD\) as follows:
\begin{equation*}
p_{\cD_{\beta}}(\textbf{x},y) = \frac{p_\cD(\textbf{x},y) \mathbb{E}_{\cD_\beta}[w(\textbf{x},y)]}{w(\textbf{x},y)}.
\end{equation*}
\end{restatable}

The reweighting factor, $w$, serves as a crucial tool that allows us to approximate the loss of any hypothesis on the ground truth distribution, $\cD$, even when we only have access to the empirical error computed from the biased sample. In the context of our model, understanding the relationship between the original distribution \(\cD\) and the biased distribution \(\cD_\beta\) is crucial for analyzing how biases in the training data affect the learning process. We do so by establishing a mathematical connection between these distributions, mediated by a reweighting function \(w(\textbf{x},y)\), which reflects the biases in the training data.\footnote{Due to space limitations, the proof for Theorem~\ref{thm:biased_pmf} is included in the Appendix.}

\subsection{Learning from Biased Data}
In the previous subsections, we have defined the original and biased distributions, and we have derived the relationship between them. Now, we turn our attention to the learning problem, where we aim to train a model on the biased data while approximating the true error on the original distribution. This requires us to define various loss functions and reweighting mechanisms that take into account the biases in the data.

\begin{definition}[True Loss]
The true loss of hypothesis \( h: \mathcal{X} \rightarrow \{0,1\} \) is defined as:
$ L_{\cD}(h) \stackrel{\text{def}}{=}  P_{(\textbf{x},y) \sim \cD} [h(\textbf{x}) \neq y]$.
\end{definition}

\begin{remark}
This measures how likely \( h \) is to make an error when labeled points are randomly drawn according to \( \mathcal{D} \). For the empirical risk on the unbiased sample \( S \), we have 
$L_{S}(h) \stackrel{\text{def}}{=}  1/m \sum_{i=1}^{m} \mathbb{I}(h(\textbf{x}_i) \neq y_i)$.
\end{remark}

We also consider the biased sample \( S_{\beta} \subset S \), where \( S_{\beta} = ((\textbf{x}_1, y_1),\dots, (\textbf{x}_{m_\beta}, y_{m_\beta}))\). We define the biased empirical risk to be the average loss over \( S_{\beta} \). Now we will define the reweighted biased empirical risk (RBER), which will be our approximation of the empirical risk on \( S \). We begin with a reweighting that assumes perfect knowledge of $\beta$.

\begin{definition}[RBER with True \( \beta\)]
The RBER for \( S_\beta \) reweighted by the true inverse of \( \beta \) is:
\begin{equation*}
L_{S_{\beta} \beta^{-1}}(h) \stackrel{\text{def}}{=}  \mathbb{E}_{(\textbf{x},y)\sim S_{\beta}} \left[w(\textbf{x},y)\mathbb{I}(h(\textbf{x}) \neq y)\right] = \left(1/m_\beta \right) \sum_{i=1}^{m_\beta} w(\textbf{x}_i, y_i) \mathbb{I}(h(\textbf{x}_i) \neq y_i).
\end{equation*}
\end{definition}

By the very nature of our problem, however, we do not know the true bias parameters -- instead, we must estimate them. We give a precise formulation for the estimates of $\frac{1}{\beta}$ and $\beta_0$ in Algorithm~\ref{algo:biased_learning} and refer to the estimates here as $\widehat{\frac{1}{\beta}}$ and $\widehat{\beta_0}$. These estimates allow us to calculate the reweighted biased empirical risk using only the observed data.

\begin{definition}[RBER with Estimated \( \beta \)]
The RBER for \( S_\beta \) reweighted by our estimate of \( \beta^{-1} \) is:
\begin{equation*}
L_{S_{\beta} \widehat{\beta^{-1}}}(h) \stackrel{\text{def}}{=}  \mathbb{E}_{(\textbf{x},y)\sim S_{\beta}} \left[\hat{w}(\textbf{x},y) \mathbb{I}(h(\textbf{x}) \neq y)\right] = \frac{1}{m_{\beta}} \sum_{i=1}^{m_{\beta}} \hat{w}(\textbf{x}_i, y_i) \mathbb{I}(h(\textbf{x}_i) \neq y_i),
\end{equation*}
where 
%\[
$\hat{w}(\textbf{x},y) \stackrel{\text{def}}{=}  \mathbb{I}(y=0) + \mathbb{I}(y=1) \widehat{\beta_0}^{\lvert G(\textbf{x}) \rvert - 1}\prod_{i \in G(\textbf{x})} \widehat{\frac{1}{\beta_i}}
$
%\]
\end{definition}

This section has formalized the learning problem in the presence of biases, defining the true error, empirical risk, and reweighted biased empirical risks. These definitions lay the groundwork for developing algorithms that can learn effectively from biased data, compensating for the biases through reweighting mechanisms.

\subsection{Summary}
Using the tools laid thus far, we are able to develop an algorithm operationalizing this learning and reweighting process, culminating in a powerful tool for mitigating intersectional bias in machine learning. By incorporating the bias parameters, we emphasize the intersectionality of biases, accounting for the collective impact of multiple group memberships on the biases in the dataset. The construction of $S_{\beta}$ enables us to explicitly consider and analyze the intersectional biases in machine learning models. Importantly, this allows us to train models on a large amount of biased data with only a small amount of unbiased data needed, accounting for intersectionality with limited resources. 

\section{Algorithm Overview}

Our algorithm, described in detail as Algorithm~\ref{algo:biased_learning}, aims to mitigate the biases present in the dataset. We start with an unbiased training set $S$ and a biased training set $S_{\beta}$. The biased training set $S_{\beta}$ is a function of the product of $\beta_0$ and the inverse of the group-specific $\beta$'s, which we estimate by the product of the $\widehat{\beta^{-1}}$'s to capture the intersectional biases. Note that by learning $\beta_0$ for the population and each $\beta_i^{-1}$ individually, we can estimate the product accurately and efficiently without having to estimate a bias parameter for each unique intersection of groups. We can then apply the intersectional bias learning algorithm to obtain a hypothesis $h$ that minimizes the risk of the learned model while considering the biases.\footnote{See the Appendix for a diagram of this process.} The algorithm proceeds as follows:

\begin{enumerate}
    \item We first determine the sample sizes required for training a model. These sizes are chosen to balance the complexity of the hypothesis class $\mathcal{H}$ and the desired confidence level in the learned model -- they depend on the logarithm of the VC dimension of the hypothesis class size, target error $\epsilon$, and confidence parameter $\delta$. 
    \item We then collect two training sets, $S$ and $S_\beta$. The unbiased training set $S$ is gathered by randomly selecting $m$ examples according to the unbiased distribution $\cD$, while the biased training set $S_\beta$ is formed by selecting $m_{\beta}$ examples from the biased distribution $\cD_{\beta}$, both with an appropriate number of samples per group. %This random selection ensures that we have representative samples from all groups in both datasets. 
    \item Next, we estimate the overall positive rate $p_0$ and the positive rate $p_i$ for each group $G_i$ according to the underlying distribution $\cD$. We continue by estimating the rate $p_{\beta_i}$ of positive examples from group $G_i$ appearing in the \textit{biased} dataset $S_{\beta}$ as well as the overall retention rate $p_{\beta_0}$.
    \item Using the estimates for $\mathbf{p}$ and $\mathbf{p_{\beta}}$, we can then estimate $\beta_i^{-1}$ for each group $G_i$ and $\beta_0$ for the population.  
    \item Finally, we apply the empirical risk minimization algorithm to train a model using just the biased training set $S_\beta$. The algorithm minimizes the risk associated with the learned model while considering the biases captured by the intersectional bias parameters. The algorithm's output is the hypothesis $h$, representing the learned model.
\end{enumerate}

By incorporating intersectional biases in the learning process, our algorithm aims to develop high-performing models that account for the complexities of bias arising from multiple group memberships. It allows us to learn models addressing the specific challenges of intersectionality, thereby promoting equity in machine-learning applications. 
 
\begin{algorithm}
\caption{Intersectional Bias Learning Algorithm}
\label{algo:biased_learning}
\textbf{Input:} Unbiased training set $S$ of size $m$ with $m_i$ samples in group $G_i$ for $i=1,..,k$, Biased training set $S_{\beta}$ of size $m_{\beta}$ with $m_{\beta_i}$ samples in group $G_i$ for $i=1,...,k$
%\textbf{Output:} Hypothesis $h \in \mathcal{H}$ such that $L_D(h) \leq \text{OPT} + \epsilon$, with probability at least $1-\delta$
\begin{enumerate}
\item Estimate $\widehat{p}_i = \frac{1}{m_i}\sum_{i=1}^{m_i} \mathbb{I}(y_i = 1)$ for each group $i$ from $S$ and $\widehat{p_0} = \frac{1}{m}\sum_{i=1}^{m} \mathbb{I}(y_i = 1)$ for the population.
\item Estimate $\widehat{p_{\beta_i}} = \frac{1}{m_{\beta_i}}\sum_{i=1}^{m_{\beta_i}} \mathbb{I}(y_i = 1)$ for group $i$ from $S_{\beta}$ and $\widehat{p_{\beta_0}}= \frac{1}{m_{\beta}}\sum_{i=1}^{m_{\beta}}\mathbb{I}(y_i = 1)$ for the population.
\item Let $\widehat{\frac{1}{\beta_i}} = \frac{\widehat{p_i}}{\widehat{p_{\beta_i}}}\frac{\left(1 - \widehat{p_{\beta_i}}\right)}{\left(1 - \widehat{p_i}\right)}$ be the estimated bias for group $i$ and $\widehat{\beta_0}=\frac{\widehat{p_{\beta_0}}}{\widehat{p}}$ be the estimated population retention rate.
\item Use ERM to return a hypothesis $h$ that minimizes the empirical risk on $\frac{m_{\beta}}{\sum_{i=1}^{m_{\beta}} \hat{w}(\textbf{x}_i, y_i)}L_{S_\beta \widehat{\beta^{-1}}}(h)$
\end{enumerate}
\end{algorithm}

\subsection{Theorem Overview}
The main theorem states that given a hypothesis class $\mathcal{H}$ and an unknown distribution $\cD$ over feature space $\mathcal{X} \times \{0,1\}$, if we have an unbiased sample $S$ of size $m$ and a biased sample $S_{\beta}$ of size $m_{\beta}$, with the marginal probability of a positive example in group $i$ of $S_{\beta}$ being $p_{\beta_i}$, then running Algorithm \ref{algo:biased_learning} with appropriate sample sizes ensures that, with high probability, the algorithm outputs a hypothesis $h$ that has a low error on the distribution $\cD$.

\begin{restatable}{theorem}{mainTheorem}
\label{theorem:mainTheorem}
Let $\delta > 0$, $0 < \epsilon < \frac{1}{9}$, $\mathcal{H}$ be a hypothesis class with VC-dimension $|\mathcal{H}|$, and let $\cD$ be an unknown distribution over $\mathcal{X} \times \{0,1\}$, where $\mathcal{X}$ is a feature space. Let $S$ be an unbiased sample of $m$ examples drawn i.i.d. from $\cD$, and let $S_{\beta}$ be a biased sample of $m_{\beta}$ examples drawn i.i.d. from $\cD_{\beta}$. If Algorithm \ref{algo:biased_learning} is run with sample sizes 
\begin{align*}
    &m_\beta \geq \frac{11^2 \left(\beta_0^{k-2} \min_i {\beta_i}^{1-k} - 1\right)^4}{{2\epsilon^2}}\ln{\frac{4|H|(k+1)}{\delta}},\; m_{\beta_i} \geq \frac{3 \cdot 11^2 \left(\beta_0^{k-2} \min_i {\beta_i}^{1-k} - 1\right)^2}{p_{\beta_i}\epsilon^2} \ln \dfrac{2(2k+2)}{\delta} \\& m_i \geq \frac{3 \cdot 11^2 \left(\beta_0^{k-2} \min_i {\beta_i}^{1-k} - 1\right)^2}{p_i \epsilon^2 }\ln\frac{2(2k+2)}{\delta}
\end{align*}
for all groups $G_i$, then with probability $1-\delta$,
\[
\lvert \frac{m_\beta}{\sum_{i=1}^{m_{\beta}} \hat{w}(\textbf{x}_i, y_i)}L_{S_\beta \widehat{\beta^{-1}}}(h) - L_\cD(h)\rvert \leq \epsilon
\]
 
\end{restatable}

\subsection{Implications for Agnostic PAC Learning}

Now that we have proved that Algorithm~\ref{algo:biased_learning} can produce a model $h$ such that the reweighted loss of $h$ on the biased distribution is within an $\epsilon$ factor of the loss on the true distribution, it remains to show that this implies that the model class $\mathcal{H}$ is agnostic PAC learnable by Algorithm~\ref{algo:biased_learning} when training primarily on the biased distribution. We begin with a formal statement of the classic definition of agnostic PAC learning:

\begin{definition}[Agnostic PAC Learning \cite{shalev-shwartz_ben-david_2014}]
A hypothesis class \(\mathcal{H}\) is agnostic PAC learnable if there exists a function $m_{\mathcal{H}}:\left(0,1\right)^2 \rightarrow \mathbb{N}$ and a learning algorithm with the following property: For every $\epsilon, \delta \in (0,1)$ and for every distribution \(\mathcal{D}\) over $\mathcal{X} \times \mathcal{Y}$, when running the learning algorithm on $m\geq m_{\mathcal{H}}\left( \epsilon, \delta\right)$ i.i.d. examples generated by $\mathcal{D}$, the algorithm returns a hypothesis $h$ such that, with probability of at least $1-\delta$ (over the choice of the $m$ training examples)
\[L_{\mathcal{D}}(h) \leq \min_{h' \in \mathcal{H}} L_{\mathcal{D}}\left(h'\right) + \epsilon
\]
\end{definition}
 
Notice that this definition only includes one data-generating distribution. Therefore, to formally discuss agnostic PAC learning in the presence of underrepresentation bias, we define a new notion of PAC learnability that takes in both a biased and unbiased distribution and requires sufficient samples from each:
 
\begin{definition}[Agnostic PAC Learning with Underrepresentation and Intersectional Bias]
A hypothesis class \(\mathcal{H}\) is agnostic PAC learnable with underrepresentation and intersectional bias if there exists a function $m_{\mathcal{H}}:\left(0,1\right)^2 \rightarrow \mathbb{N}$, a function $m_{\beta,\mathcal{H}}:\left(0,1\right)^2 \rightarrow \mathbb{N}$ and a learning algorithm with the following property: For every $\epsilon, \delta \in (0,1)$ and for every unbiased distribution \(\mathcal{D}\) over $\mathcal{X} \times \mathcal{Y}$ and biased distribution \(\mathcal{D_{\beta}}\) over $\mathcal{X} \times \mathcal{Y}$ (defined as in Section~\ref{sec:prelims}), when running the learning algorithm on $m\geq m_{\mathcal{H}}\left( \epsilon, \delta\right)$ i.i.d. examples generated by $\mathcal{D}$ and $m_\beta \geq m_{\beta, \mathcal{H}}\left( \epsilon, \delta\right)$ i.i.d. examples generated by $\mathcal{D_\beta}$, the algorithm returns a hypothesis $h$ such that, with probability of at least $1-\delta$ over the choice of the $m$ and $m_\beta$ training examples
\[\frac{m_\beta}{\sum_{i=1}^{m_{\beta}} \hat{w}(\textbf{x}_i, y_i)}L_{S_\beta \widehat{\beta^{-1}}}(h) \leq \min_{h' \in \mathcal{H}} L_{\mathcal{D}}\left(h'\right) + \epsilon
\]
\end{definition}

\begin{restatable}{theorem}{secondTheorem}
\label{theorem:secondTheorem}
If the conditions of Theorem~\ref{theorem:mainTheorem} are satisfied, the class $\mathcal{H}$ is agnostically PAC learnable with underrepresentation and intersectional bias.
 
\end{restatable}

To summarize, recall that Algorithm~\ref{algo:biased_learning} employs the Empirical Risk Minimization (ERM) principle on the reweighted sample \( S_\beta \) to find \( h \), the hypothesis that minimizes the empirical error. Thus, \( h \) is optimal on the reweighted sample. Given its optimality on the reweighted sample and the proximity of its true error on \( \mathcal{D} \) to its reweighted empirical error on \( S_\beta \), it follows that \( h \)'s true error on \( \mathcal{D} \) is close to the best possible error of any hypothesis in \( \mathcal{H} \) on \( \mathcal{D} \). Therefore, the intersectional bias learning algorithm satisfies the Agnostic PAC learning definition, as the learned hypothesis \( h \) adheres to the error bound relative to the best hypothesis in \( \mathcal{H} \) on \( \mathcal{D} \).

\section{Experiments}
%\subsection{Objectives}
%\label{subsec:objectives}
In this section, we perform a thorough empirical investigation of our method. Through this investigation, we aim to address several distinct, though interrelated, objectives. 
\begin{enumerate}
\item First, we investigate the validity of the independence assumption that we made on group membership. We measure the correlation between sensitive attributes of interest and conduct $\chi^2$ tests to determine the degree of confidence we have in potentially rejecting the null that these attributes are pairwise independent. We do this both for the overall group membership and group membership conditional on positive labels. Importantly, we realize that this does not test for mutual independence, which is the true assumption. However, pairwise independence is necessary for mutual independence and easier to test for, so we utilize it as a useful proxy. We keep these results in mind during the next two parts of our investigation to see if there is a noticeable degradation in Algorithm~\ref{algo:biased_learning}'s performance in domains with higher dependence between the groups.

\item Next, we validate the effectiveness of our proposed reweighting scheme. In particular, we demonstrate the model's ability to approximate the loss of hypotheses accurately on the true distribution using a biased sample, and therefore, according to Theorem~\ref{theorem:secondTheorem}, find a model close to that produced by training on the full unbiased data set. This showcases how the model addresses underrepresentation and intersectional biases in various scenarios. In this section, we compare the efficacy of our approach to the results obtained from using other popular approaches in the literature for adjusting group representation in data.

\item Finally, we separate our analysis to study how well our reweighting approach fares on individual groups in the data. Although our theory does not provide guarantees of any improvement in downstream fairness when using our method, we would hope to see an improvement in underrepresented groups, at least in most cases.
\end{enumerate}
\subsection{Data Description}

In this section, we describe the data sets used for evaluation, including the nature and source of the data sets and the characteristics of intersectional biases present. We focus on the following three data sets for our experiments, which we selected because they are standard fairness benchmark datasets with flexible options for sensitive feature selection. The first is the Adult data set\cite{misc_adult_2}, which predicts whether an individual's income exceeds \$50K/yr based on census data. The second is the COMPAS data set \cite{compas}, which consists of arrest data from Broward County, Florida, originally compiled by ProPublica. Finally, we utilize the American Community Survey (ACS) Employment data set \cite{DBLP:journals/corr/abs-2108-04884}, which is used to predict individual employment status.

\begin{table}[ht]
\footnotesize
\centering
\begin{tabular}{|l|c|c|l|l|l|}
\hline
Dataset   & \multicolumn{1}{l|}{Samples} & \multicolumn{1}{l|}{Features} & Group Types & Label                                \\  
\hline

Adult & 48842 & 14 & Sex, Race & Income > \$50K/year \\
\hline
COMPAS & 4904 & 9 & Sex, Race & Two year recidivism \\
\hline
ACS Employment & 196104& 12 & Sex, Race & Binary employment status\\
\hline
\end{tabular}
\label{tab:data}
\end{table}
%\ed{Clarify how missing values are dealt with.}
%\ed{Could also add different base models.}
%\ed{Synthetic data?}

In the Adult data set, the possible racial attributes are \textit{White, Pacific Islander, Eskimo, Other, and Black}, and the possible values for sex are \textit{Male} and \textit{Female}. Because sex is encoded as a binary variable, those two attributes are fully dependent. Therefore, we remove the attribute \textit{Male} from our study and only consider the dropout bias for \textit{Female}. Note that this aligns with our earlier assumption about having samples that do not belong to any of the defined groups. We do the same for the race attribute, removing \textit{White} from our list of groups suffering from underrepresentation bias. We follow a similar procedure for the other two datasets. For the COMPAS data set, our final sensitive feature list is \textit{Black, Hispanic, Other, and Female}, and for the ACS Employment data set, our categories are \textit{Black, American Indian, Alaska Native, Other Native, Asian, Native Hawaiian, Female}.\footnote{We chose to use the combination of ``race'' and ``sex'' as protected attributes in all experiments in order to have more consistent comparisons. In addition, these attributes are ones typically examined in terms of statistical group fairness metrics for these datasets which, while not the subject of this paper, is a related area of interest. However, we equally well could have chosen any other attributes or functions of attributes to use as the group labels. However, it is important to note that some choices of groups might be more appropriate to the modeling framework than others in terms of the degree to which the independence assumptions are violated. For example, a choice of sensitive attributes that would not be good choices based on the independence assumption would be ``race'' and ``zip code.''}
\subsection{Baselines}

We now provide a brief description of comparative models for baseline measurements. First, we compare the results obtained with Algorithm~\ref{algo:biased_learning} to those obtained by training a model exclusively on unbiased data or exclusively on biased data. Then, we compare to results obtained by employing an over-sampling technique and an under-sampling technique, respectively. For the over-sampler, we selected the Synthetic Minority Oversampling Technique (SMOTE) \cite{smote}, which functions by generating new, synthetic samples from minority groups based on combinations of neighboring samples. Our approach involves categorizing each combination of race and sex into distinct product groups, followed by upsampling to equalize the size of these groups. \footnote{We use the publicly available implementation for SMOTE which is available in the \textit{imbalanced-learn} package at \url{https://imbalanced-learn.org/stable/references/generated/imblearn.over_sampling.SMOTE.html}, and the Random Under Sampler is available at \url{https://imbalanced-learn.org/stable/references/generated/imblearn.under_sampling.RandomUnderSampler.html}} For the under-sampler, we select the Random Under Sampler, available in the \textit{imbalanced-learn} package, which randomly downsamples majority classes so that all classes have the same number of samples. For all methods, we sample without replacement. We assume that there is no access to the unbiased data when using both SMOTE and the Random Under Sampler, which precludes the estimation of drop-out rates and thereby provides a very naive bias correction baseline (modeling a logical approach someone might take in practice without our framework).\footnote{One exception is if, when using the Random Under Sampler, we emerge with an aggregate data set smaller than the biased dataset. In this case, we must re-sample with replacement to augment the size.} Because the focus of this work is on the data curation and not downstream model class selection, we use logistic regression as the model class for the classification model trained on these data.
\subsection{Experimental Setup}
%Next, we detail the configuration of our experiments.
\subsubsection{Data Partitioning and Preparation}
Each data set is divided into training and testing sets with a ratio of 80\% to 20\%. This division is applied across all experiments to maintain consistency. For our reweighting approach only, a further 20\% of the training data is isolated to serve as an unbiased dataset from which the reweighting function is estimated. For each data set, rows with missing values were removed, and categorical features were one-hot encoded.

\subsubsection{Reweighting Algorithm Implementation}

    We generate a random vector $\mathbf{\beta}$ to represent the underrepresentation bias in each group. The biased dataset is then created by selectively filtering positive examples from each group in the training data, based on the $\mathbf{\beta}$ vector. This process models the scenario where certain groups' positive examples are underrepresented. Using both the biased dataset and the unbiased holdout set, we then estimate the bias parameters. Finally, the reweighted loss is (approximately) minimized on the biased dataset. %This step is key to aligning the model's performance with our theoretical predictions about handling underrepresented data.

\subsubsection{Additional Model Training Configurations}

For the approach we label \textit{Downsampled}, a model is trained on the unbiased data, downsampled to match the size of the biased dataset. This model serves as a comparator for evaluating the effectiveness of the reweighting algorithm. For the approach that we label \textit{Biased}, the model is trained solely on the biased dataset without any reweighting. Additional models are trained using SMOTE, to upsamples the positive class of each group to equalize group representations, and the random under-sampler (which we refer to as \textit{Under}), to downsample the positive classes to achieve uniform group sizes. For SMOTE and under-sampling, after equalizing the positive group sizes, we then resample the aggregate data so that the size is that of the biased dataset. %These methods allow for the examination of alternative approaches to addressing data imbalance.
\begin{comment}

\begin{itemize}
    \item \textbf{Downsampled Model:} A model is trained on the unbiased data, downsampled to match the size of the biased dataset. This model serves as a comparator for evaluating the effectiveness of the reweighting algorithm.
    \item \textbf{Biased Model:} Another model is trained solely on the biased dataset without any reweighting, providing insights into the baseline performance under biased conditions.
    \item \textbf{SMOTE and RandomUnderSampler:} Additional models are trained using SMOTE, to upsamples the positive class of each group to equalize group representations, and RandomUnderSampler, to downsamples the positive classes to achieve uniform group sizes. After equalizing the positive group sizes, we then resample the aggregate data so that the size is that of the biased dataset. These methods allow for the examination of alternative approaches to addressing data imbalance.
\end{itemize}
\end{comment}
\subsubsection{Performance Evaluation}
 
    The performance of each model is evaluated on the test set. This evaluation focuses on assessing the model's generalization capabilities and its effectiveness in addressing underrepresentation bias, as postulated in our theoretical framework. The experiments are carried out in a Python3 Jupyter Notebook environment on a MacBook Pro equipped with an Apple MC Pro chip and 36 GB of memory. To ensure the robustness of our results, each experimental procedure is repeated with 100 different random seeds. 

\subsubsection{Evaluation Metrics}

To evaluate our approach, we assess several performance metrics. We compare each rebalancing strategy in terms of the \textit{population accuracy} obtained by the learned classifier, as well as the \textit{group-specific accuracies}.\footnote{We also compare \textit{precision}, \textit{recall}, and \textit{F1} scores of the model, but relegate this to the Appendix due to space constraints.} As discussed earlier, our theoretical bounds on error gaps only apply directly to the overall loss -- however, we might expect them to extend to the other error metrics. We investigate this further empirically. We also use several statistical tests to assess the degree of correlation between the sensitive attributes we study and the confidence with which we may be able to reject the null hypothesis that the attributes are pairwise independent.  
\section{Results}
\subsection{Independence Assumption}
We begin by visualizing the correlation matrices -- shown as heat maps in the top row of Figure~\ref{fig:heatmaps} -- for the two categories of sensitive variables chosen for each dataset: sex and race. While this is of course only a proxy for statistical dependence, it does suggest that many attributes have fairly weak pairwise relationships. Importantly, there is also little difference in the correlations observed between sensitive attributes over the entire dataset and sensitive attributes when the dataset is restricted to positive samples. Next, we use the $\chi^2$ test of independence to test for pairwise dependencies between the attributes, and we display heat maps of the $p$-values for each pairwise test in the bottom row of Figure~\ref{fig:heatmaps}. We find the results to be more mixed, with the COMPAS dataset displaying the highest proportion of attribute pairs with significant $p$-values. Interestingly, all data sets seemed to show a slight decrease in the number of attributes with significant $p$-values for the $\chi^2$ test when calculated just over positive samples. With these observations in mind, we continue to analyze the overall performance of models trained on biased data, unbiased data, and reweighted or resampled data.

\begin{figure}[h]
\begin{subfigure}{0.3\textwidth}
\begin{subfigure}{0.49\textwidth}
\begin{subfigure}{0.99\textwidth}
\includegraphics[width=1.25\textwidth]{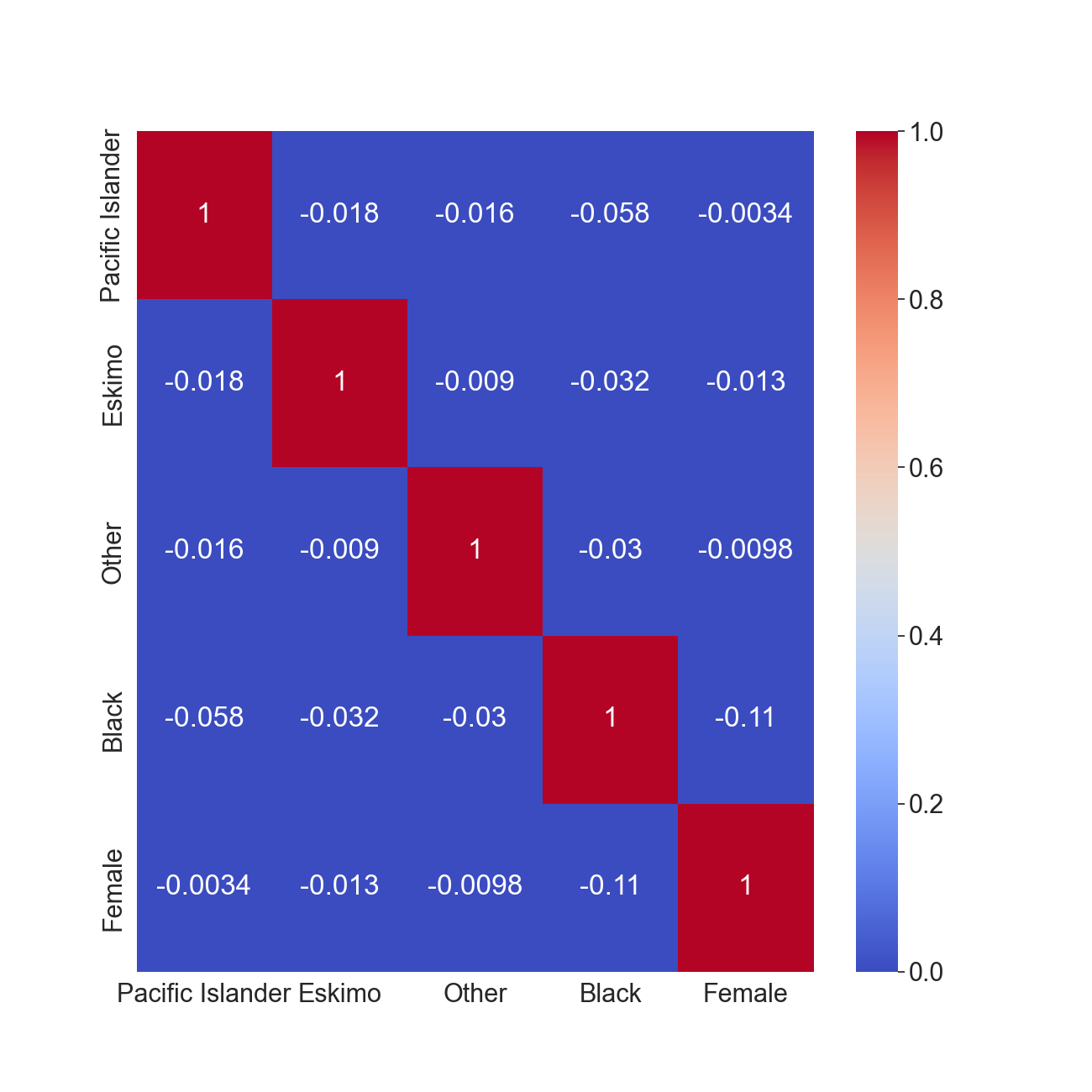}
%\caption{Corr}
\end{subfigure}
\vfill
\begin{subfigure}{0.99\textwidth}
\includegraphics[width=1.25\textwidth]{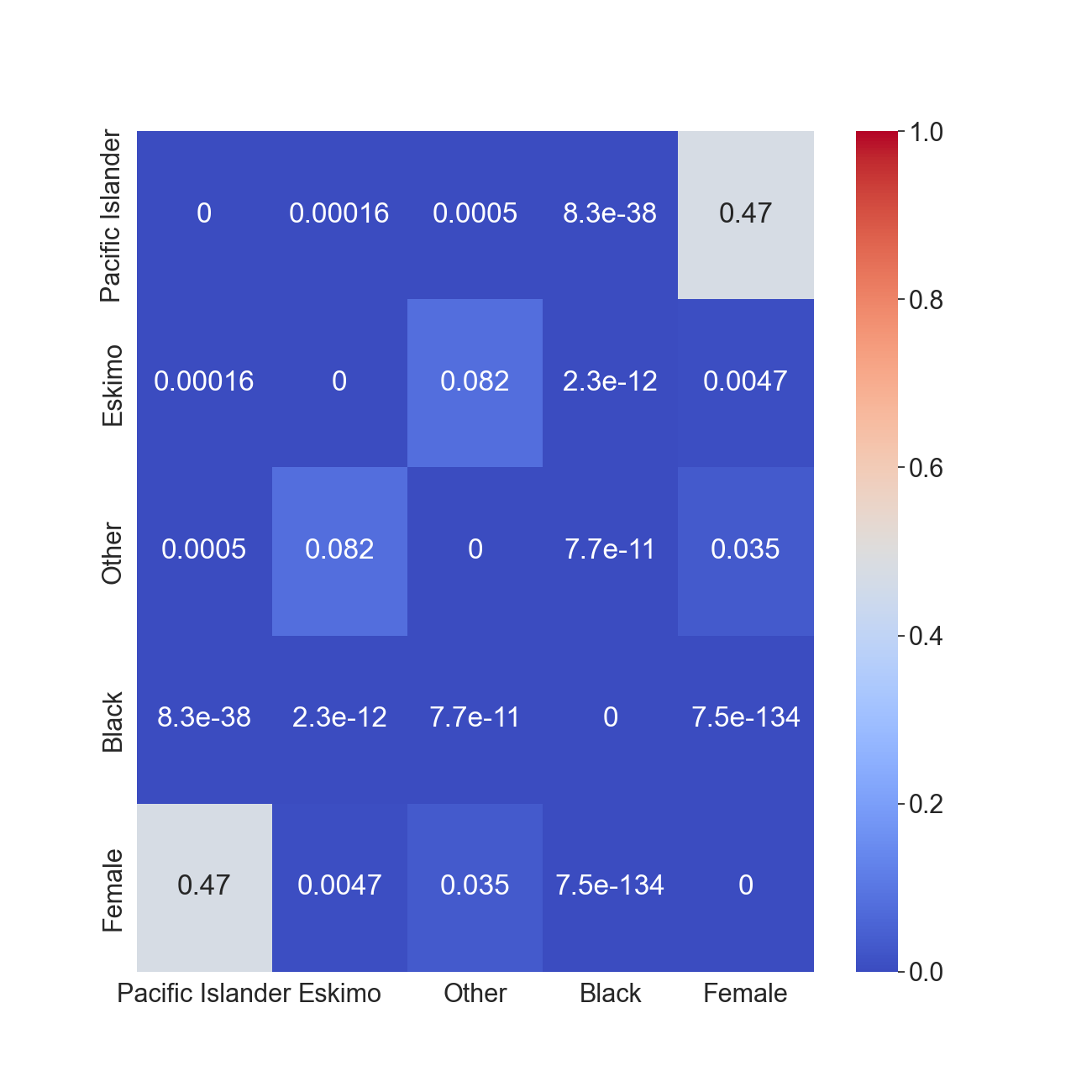}
%\caption{$\chi^2$}
\end{subfigure}
\caption*{Marginal}
\end{subfigure}
\hfill
\begin{subfigure}{0.49\textwidth}
\begin{subfigure}{0.99\textwidth}
\includegraphics[width=1.25\textwidth]{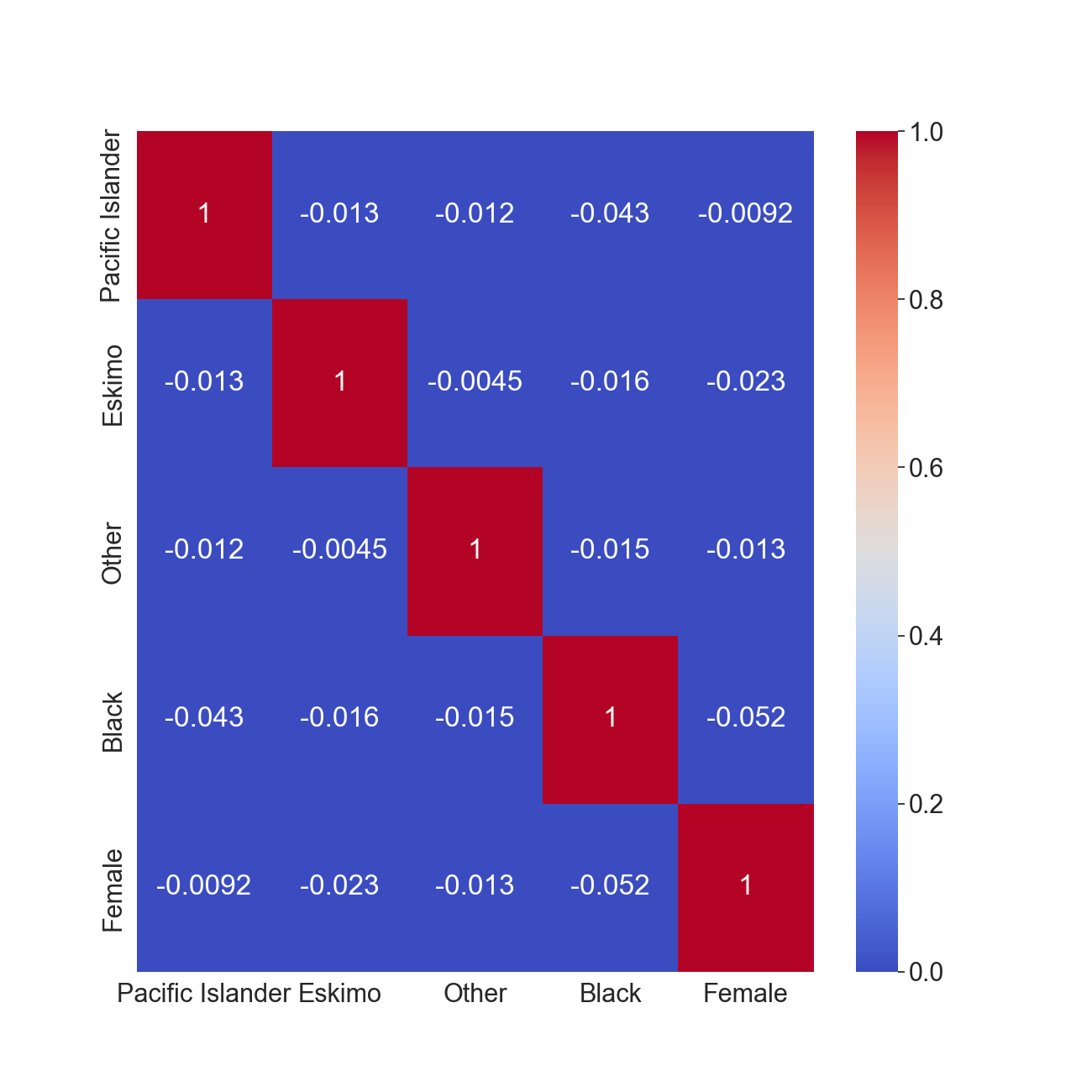}
%\caption{Corr}
\end{subfigure}
\vfill
\begin{subfigure}{0.99\textwidth}
\includegraphics[width=1.25\textwidth]{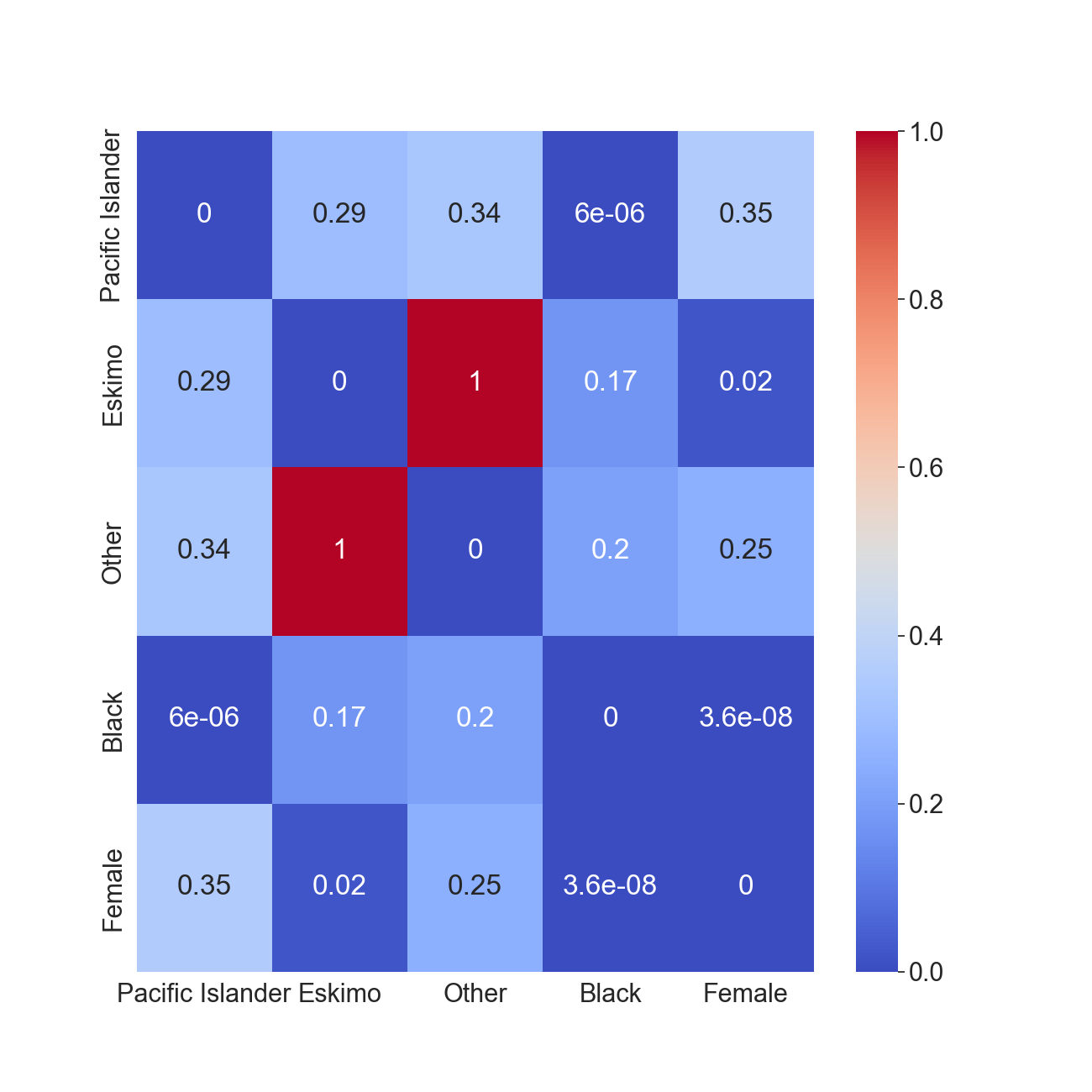}
%\caption{$\chi^2$}
\end{subfigure}
\caption*{Conditional}
\end{subfigure}
\caption*{Adult}
\end{subfigure}
\begin{subfigure}{0.3\textwidth}
\begin{subfigure}{0.49\textwidth}
\begin{subfigure}{0.99\textwidth}
\includegraphics[width=1.25\textwidth]{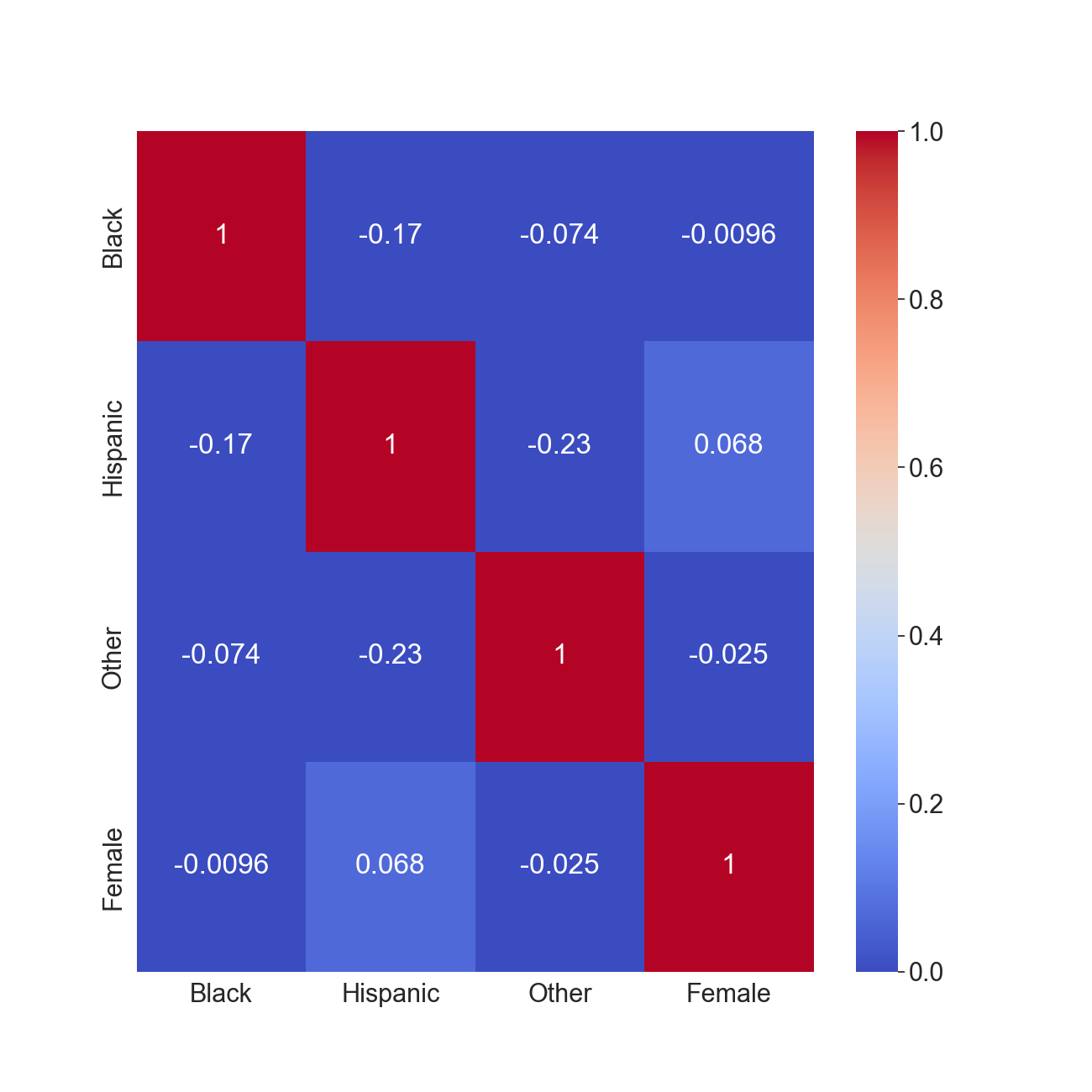}
%\caption{Corr}
\end{subfigure}
\vfill
\begin{subfigure}{0.99\textwidth}
\includegraphics[width=1.25\textwidth]{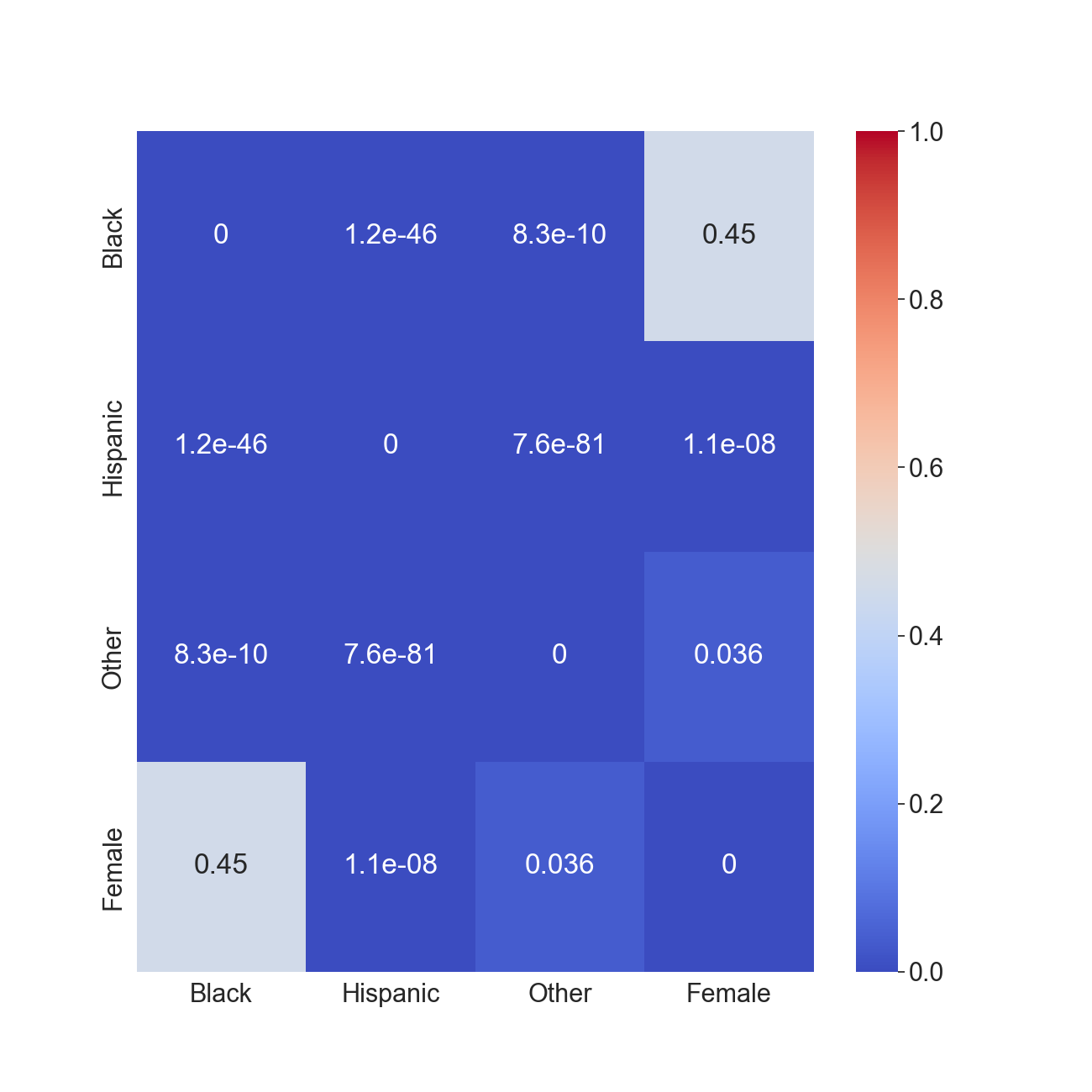}
%\caption{$\chi^2$}
\end{subfigure}
\caption*{Marginal}
\end{subfigure}
\hfill
\begin{subfigure}{0.49\textwidth}
\begin{subfigure}{0.99\textwidth}
\includegraphics[width=1.25\textwidth]{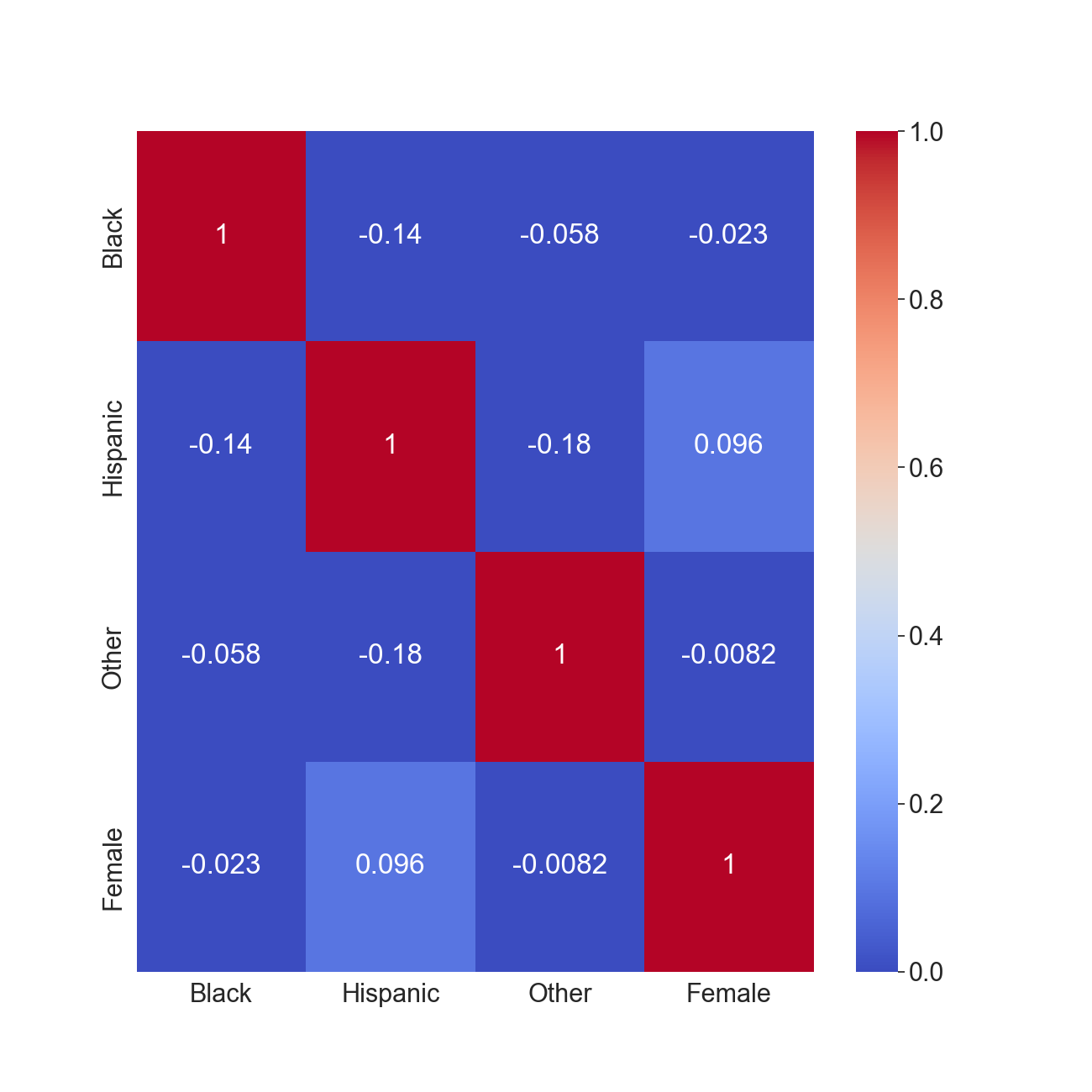}
%\caption{Corr}
\end{subfigure}
\vfill
\begin{subfigure}{0.99\textwidth}
\includegraphics[width=1.25\textwidth]{Figures/heatmap_compas_all_chi2.png}
%\caption{$\chi^2$}
\end{subfigure}
\caption*{Conditional}
\end{subfigure}
\caption*{COMPAS}
\end{subfigure}
\begin{subfigure}{0.3\textwidth}
\begin{subfigure}{0.49\textwidth}
\begin{subfigure}{0.99\textwidth}
\includegraphics[width=1.25\textwidth]{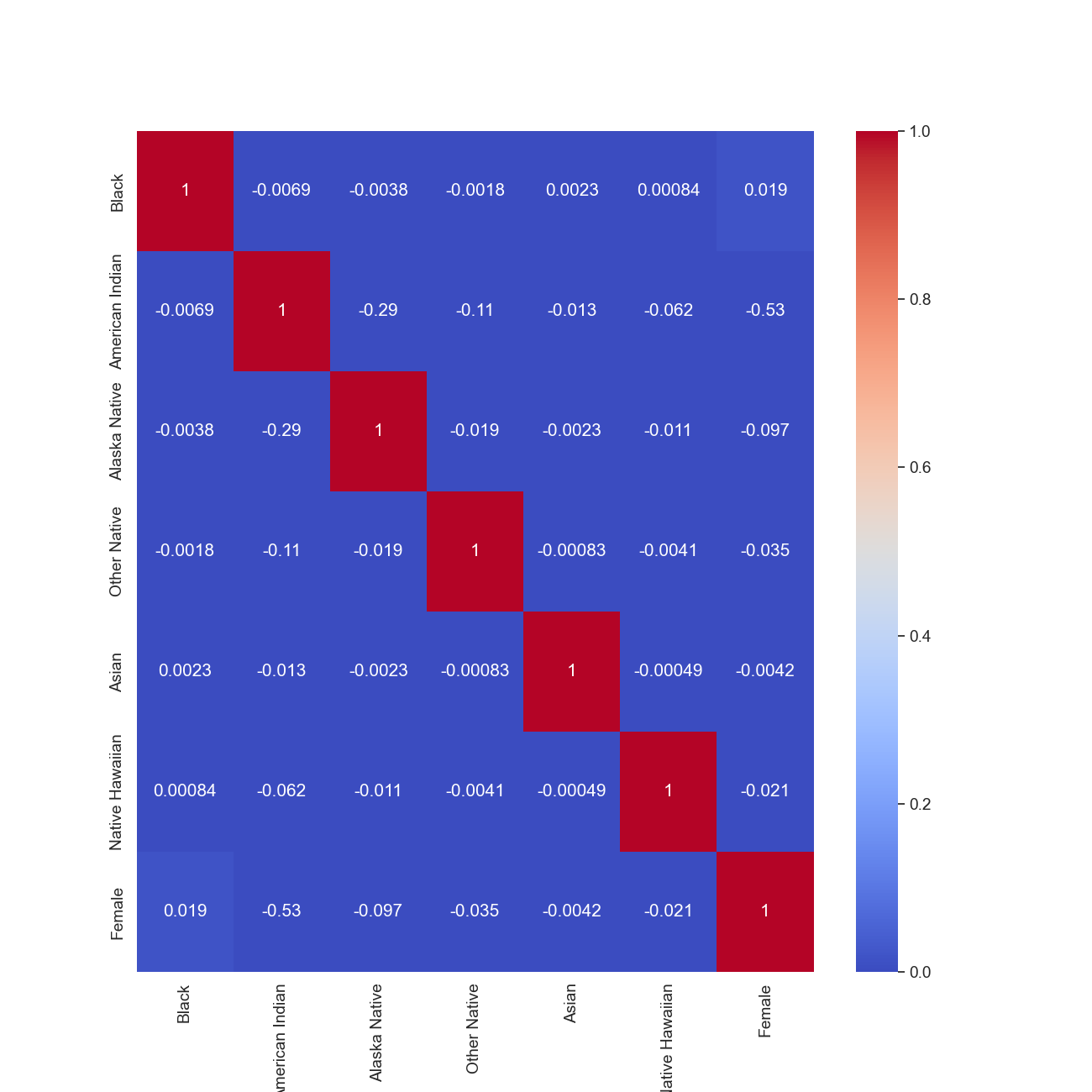}
%\caption{Corr}
\end{subfigure}
\vfill
\begin{subfigure}{0.99\textwidth}
\includegraphics[width=1.25\textwidth]{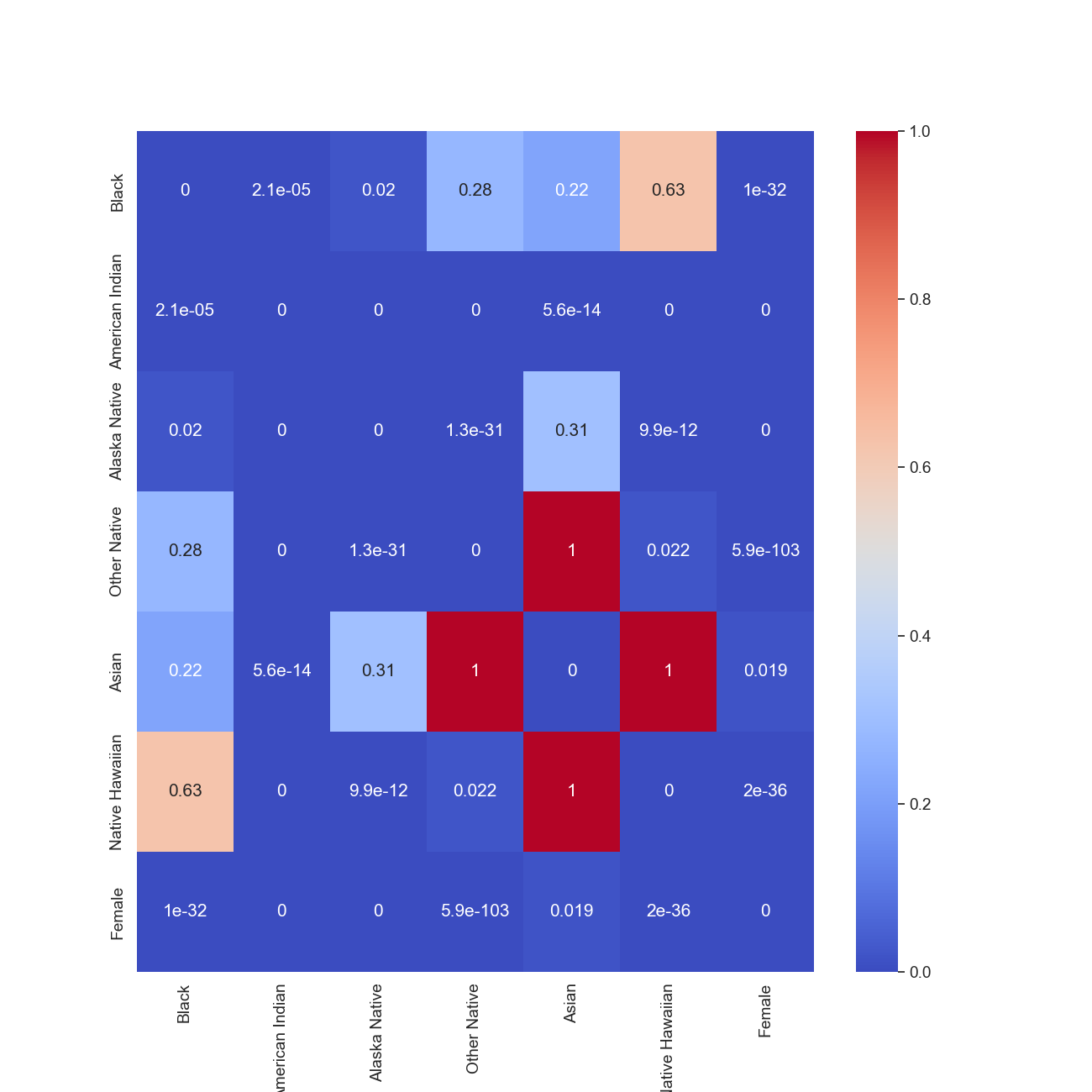}
%\caption{$\chi^2$}
\end{subfigure}
\caption*{Marginal}
\end{subfigure}
\hfill
\begin{subfigure}{0.49\textwidth}
\begin{subfigure}{0.99\textwidth}
\includegraphics[width=1.25\textwidth]{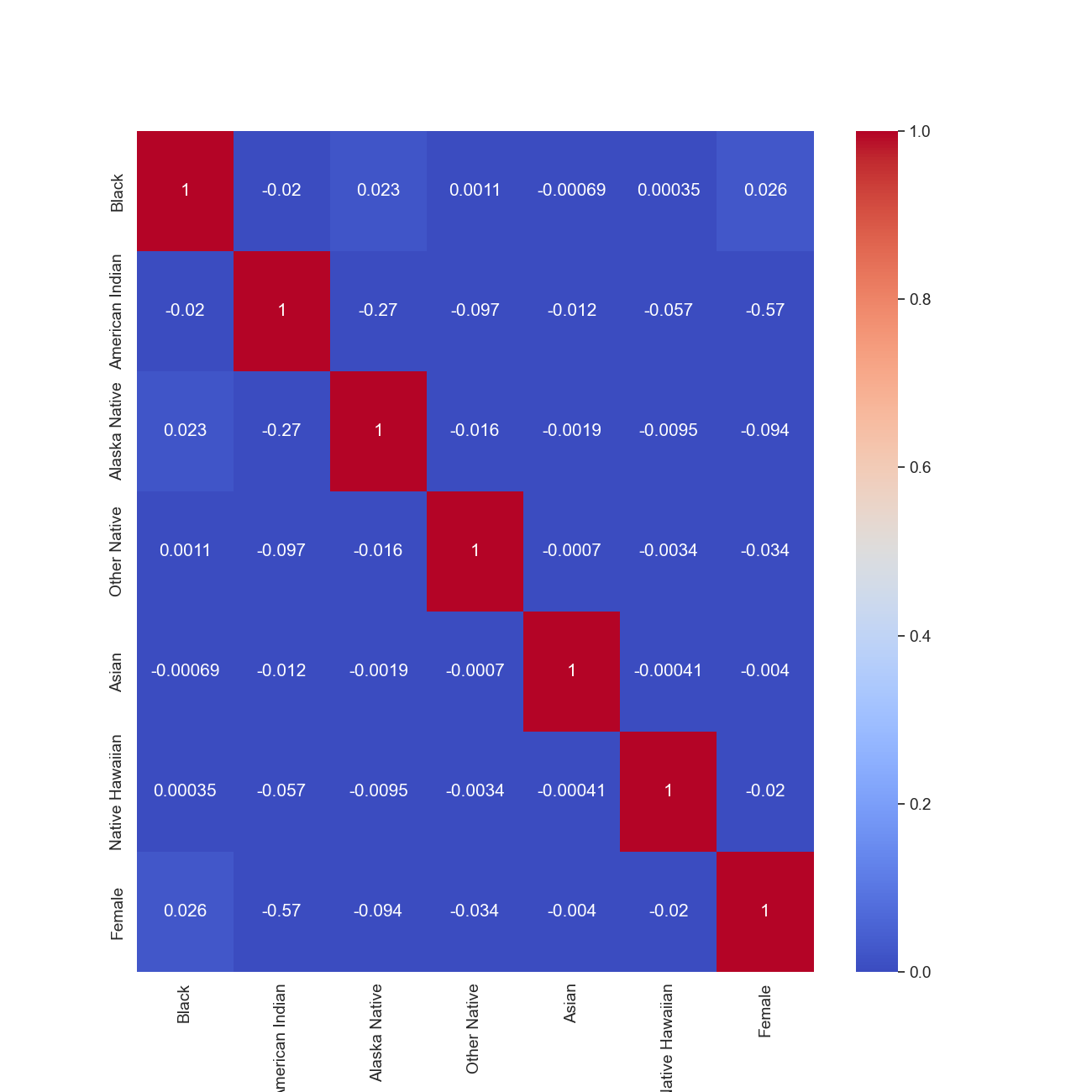}
%\caption{Cor}
\end{subfigure}
\vfill
\begin{subfigure}{0.99\textwidth}
\includegraphics[width=1.25\textwidth]{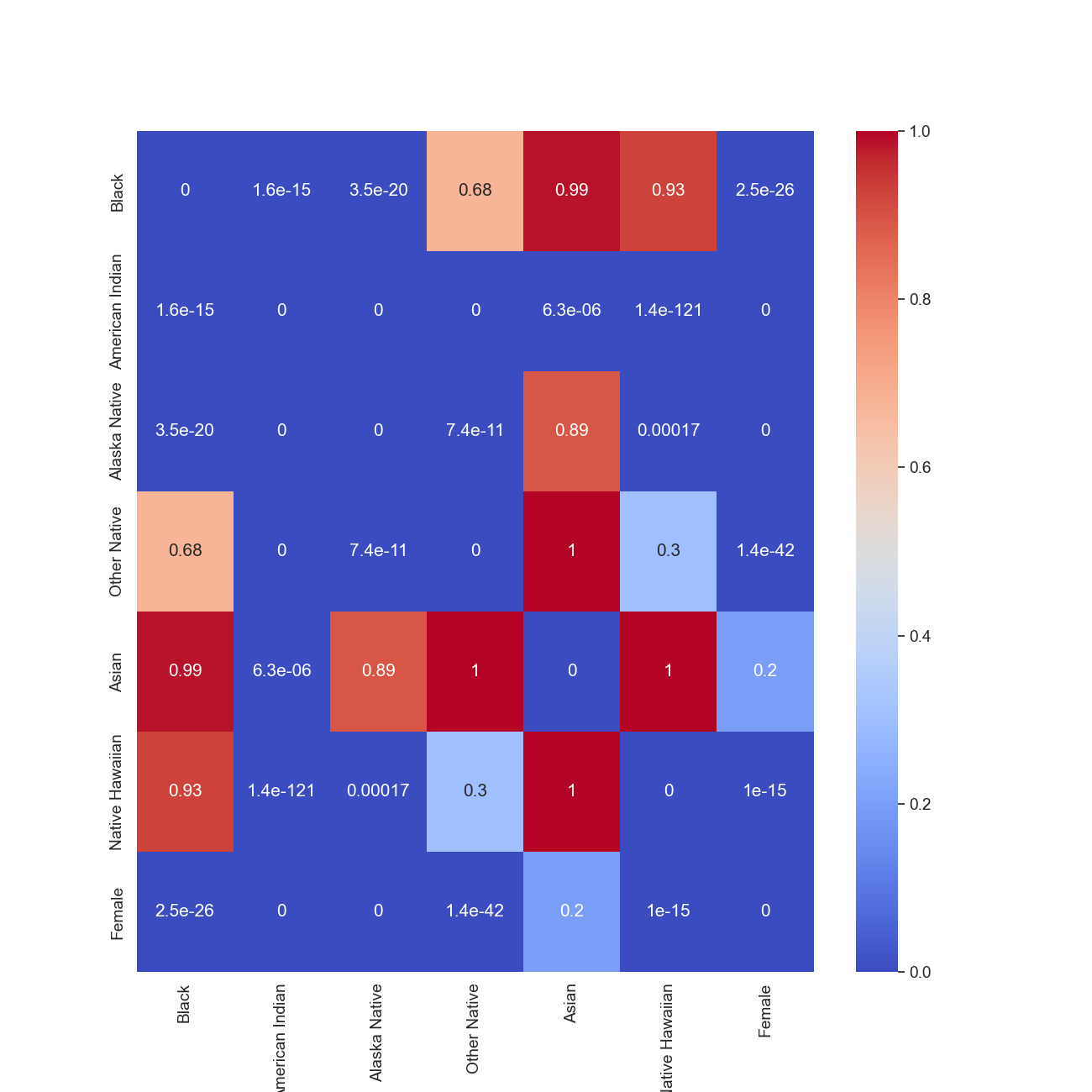}
%\caption{$\chi^2$}
\end{subfigure}
\caption*{Conditional}
\end{subfigure}
\caption*{ACS Employment}
\end{subfigure}
\caption{Heatmaps of Pairwise Correlation (Top Row) and $p$-values for $\chi^2$ Test (Bottom Row)}
\label{fig:heatmaps}
\end{figure}

\subsection{Evaluation of Population Accuracy Across Models}
In the quest to address and mitigate biases in predictive modeling, we conduct a comparative analysis of population accuracy across several models. Figure~\ref{fig:population} displays the distribution of accuracy for five different pre-processing and training methodologies: Unbiased, Reweighted, Biased, SMOTE, and RandomUnderSampler (Under), as applied to three distinct data sets: Adult, COMPAS, and ACS Employment. 
%The results are aggregated over 100 independent training iterations, or "seeds," to ensure statistical robustness.
%\subsection{Trained Model Performance}
\label{fig:population}
\begin{figure}[h]
\begin{subfigure}{0.33\textwidth}
\includegraphics[width=\textwidth]{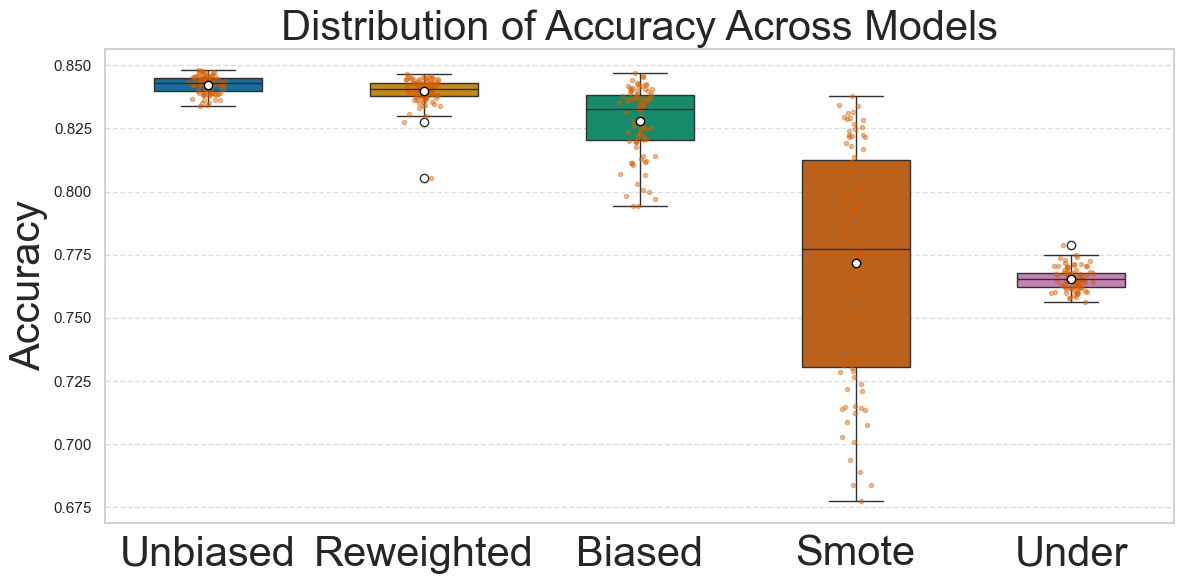}
\caption{Adult}
\end{subfigure}
\begin{subfigure}{0.33\textwidth}
\includegraphics[width=\textwidth]{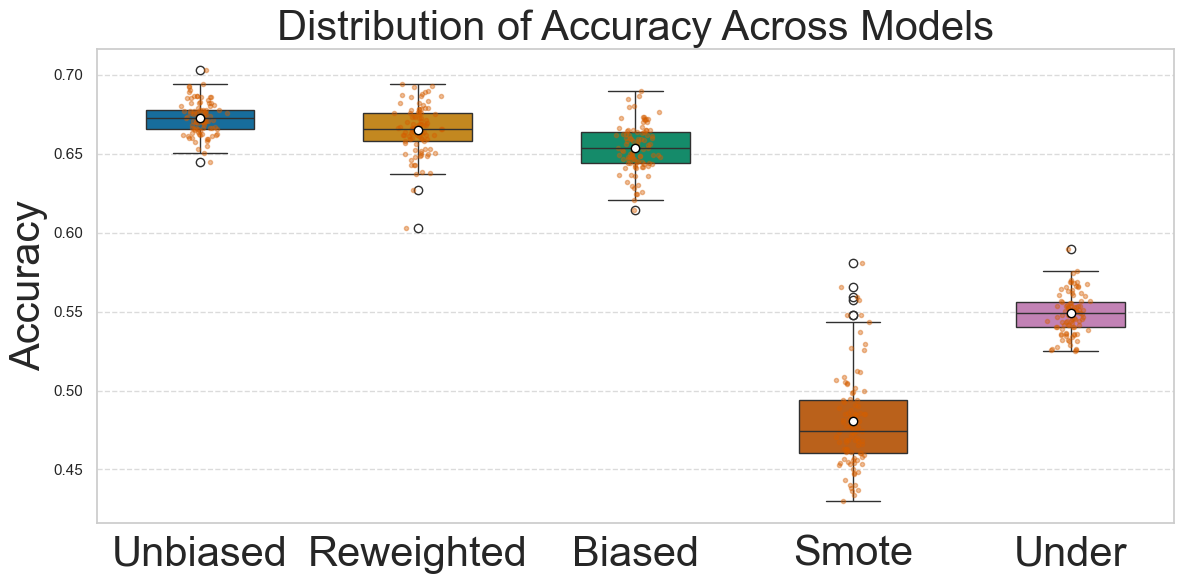}
\caption{COMPAS}
\end{subfigure}
\begin{subfigure}{0.33\textwidth}
\includegraphics[width=\textwidth]{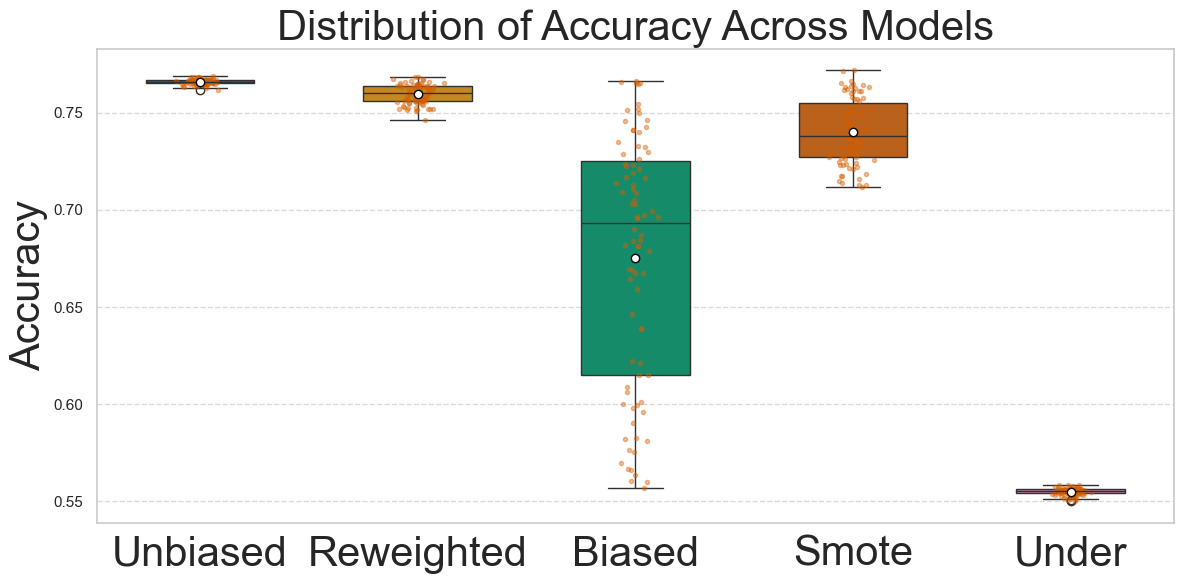}
\caption{ACS Employment}
\end{subfigure}
\caption{Population Accuracy Observed over 100 Seeds for Each Model}
\end{figure}

%\paragraph{\textbf{Adult Dataset}}
For the Adult dataset (a), the Reweighted model displays a central tendency towards higher accuracy, closely mirroring the Unbiased model. This suggests that the reweighting technique is effective in approximating the ideal scenario where the model is trained on an unbiased dataset. The Biased model demonstrates a noticeably lower median accuracy with a wider interquartile range, indicating less reliable performance across runs. SMOTE and the Under-sampling technique exhibit median accuracies that are inferior to the Reweighted model, with SMOTE showing significant variance as evidenced by a long tail of outliers. %, which raises concerns over its consistency.

%\paragraph{\textbf{COMPAS Dataset}}
The COMPAS dataset (b) presents a different pattern. The Unbiased model outperforms other methodologies with the highest median accuracy and the smallest interquartile range, highlighting its stable and superior predictive ability. The Reweighted model, while competitive, falls short of the Unbiased model's performance, as indicated by a slightly lower median accuracy and a broader interquartile range. Both the Biased and SMOTE models exhibit reduced performance with wider variability, and the Under-sampling model, in particular, shows a pronounced spread in accuracy scores, suggesting inconsistent model behavior across different iterations.\footnote{Due to space constraints, we include the group accuracy plots for COMPAs in the Appendix.}

%\paragraph{\textbf{ACS Employment Dataset}}
For the ACS Employment data set (c), the Reweighted model's performance is commendable, achieving a median accuracy comparable to that of the Unbiased model, which indicates the effectiveness of the reweighting approach in this context. The Biased model's performance is markedly lower, reinforcing the impact of biases on model accuracy. Notably, the SMOTE method appears to improve upon the Biased model but does not achieve parity with the Reweighted or Unbiased models. Under-sampling shows the least favorable performance with the lowest median accuracy and greatest variability, suggesting that this approach may not be suitable for this particular dataset. The variability in performance of the biased model on the ACS datasets, as opposed to others, is a noteworthy observation. We conjecture that this may stem from the ACS dataset containing a larger number of sensitive groups. The filtering process, applied to generate the biased dataset, likely introduces additional uncertainty due to the increased complexity of handling multiple sensitive groups. We believe this would be an interesting avenue for further investigation.

\subsection{Analysis of Group Accuracies}
\subsubsection{Adult Data Set}
Figure~\ref{fig:groups_adult} displays the distribution of group accuracies across different data preprocessing methods within the Adult data set. The accuracies are reported for five distinct demographic groups: \textit{Pacific Islander, American Indian and Eskimo, Other Race, Black,} and \textit{Female}. Across all groups, the Reweighted model consistently demonstrates accuracies that closely match those of the Unbiased model, indicating that this approach can effectively mitigate biases in the dataset. The Biased model's lower performance across all groups confirms the presence of bias, while the SMOTE and Under-sampling methods show varying degrees of effectiveness, with neither consistently reaching the performance level of the reweighted or unbiased models. This analysis suggests that reweighting serves as a potent tool for correcting underrepresentation and intersectional bias, thus fostering fairness in classification. %These findings underscore the necessity for carefully calibrated data preprocessing as a means to ensure equitable machine learning practices.  %The results provide insight into how each method performs in terms of classification accuracy within each subgroup, which is a crucial aspect of fairness in machine learning.

\begin{figure}[h]
\begin{subfigure}{0.33\textwidth}
\includegraphics[width=\textwidth]{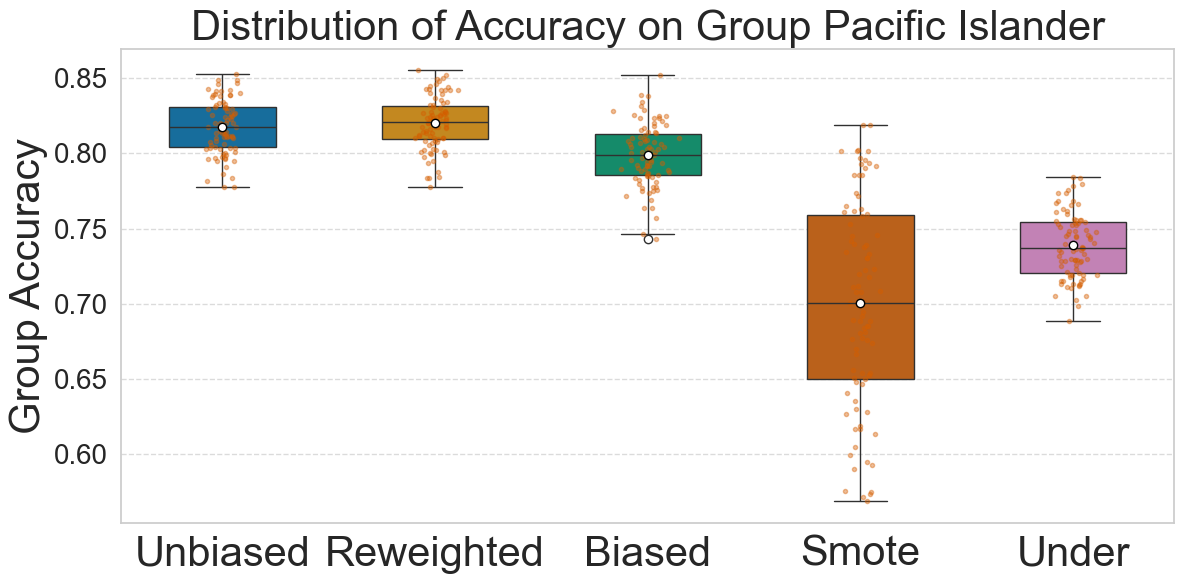}
\caption{Pacific Islander}
\end{subfigure}
\begin{subfigure}{0.33\textwidth}
\includegraphics[width=\textwidth]{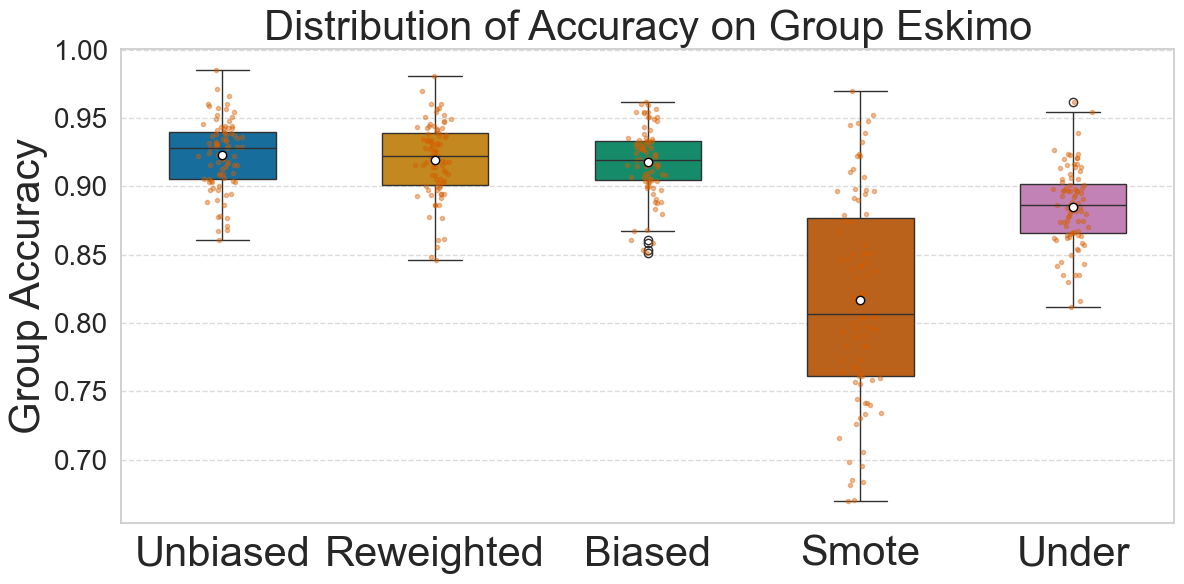}
\caption{American Indian \& Eskimo}
\end{subfigure}
\begin{subfigure}{0.33\textwidth}
\includegraphics[width=\textwidth]{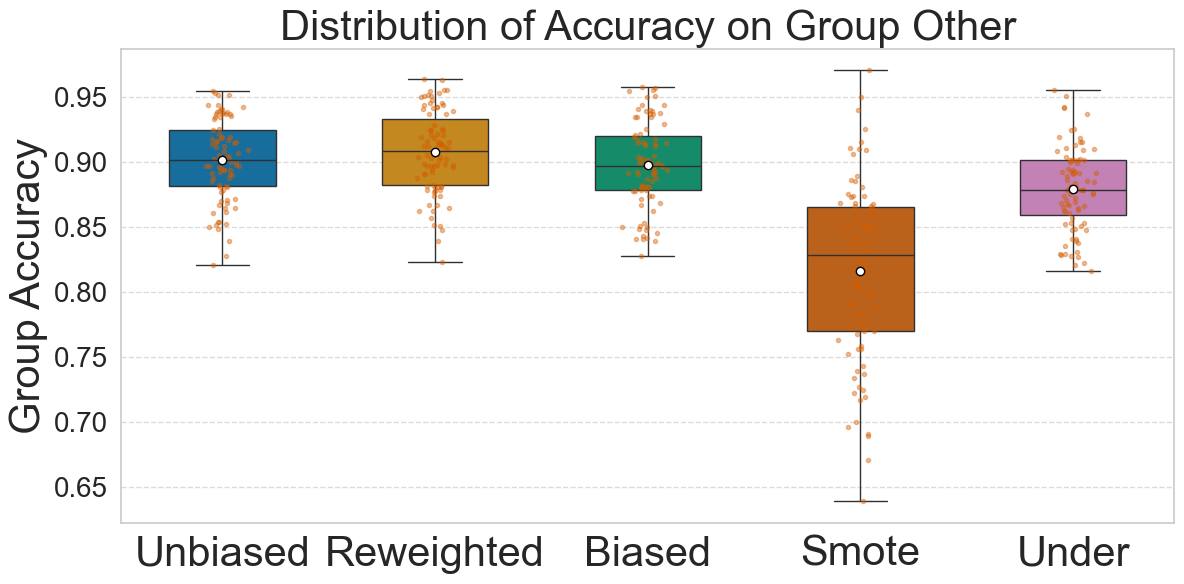}
\caption{Other Race}
\end{subfigure}
\begin{subfigure}{0.33\textwidth}
\includegraphics[width=\textwidth]{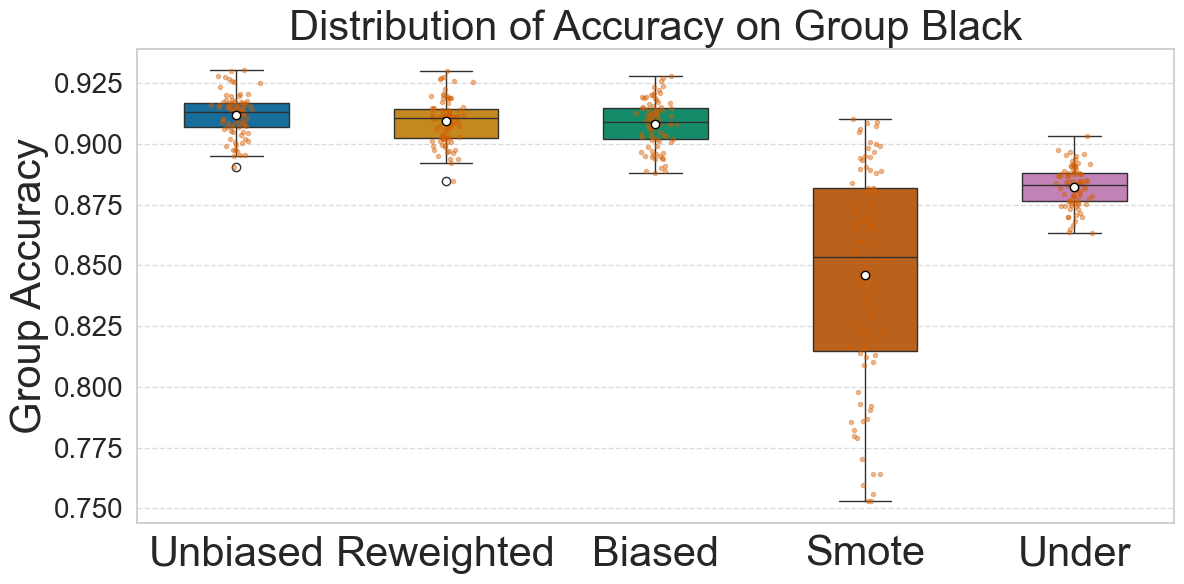}
\caption{Black}
\end{subfigure}
\begin{subfigure}{0.33\textwidth}
\includegraphics[width=\textwidth]{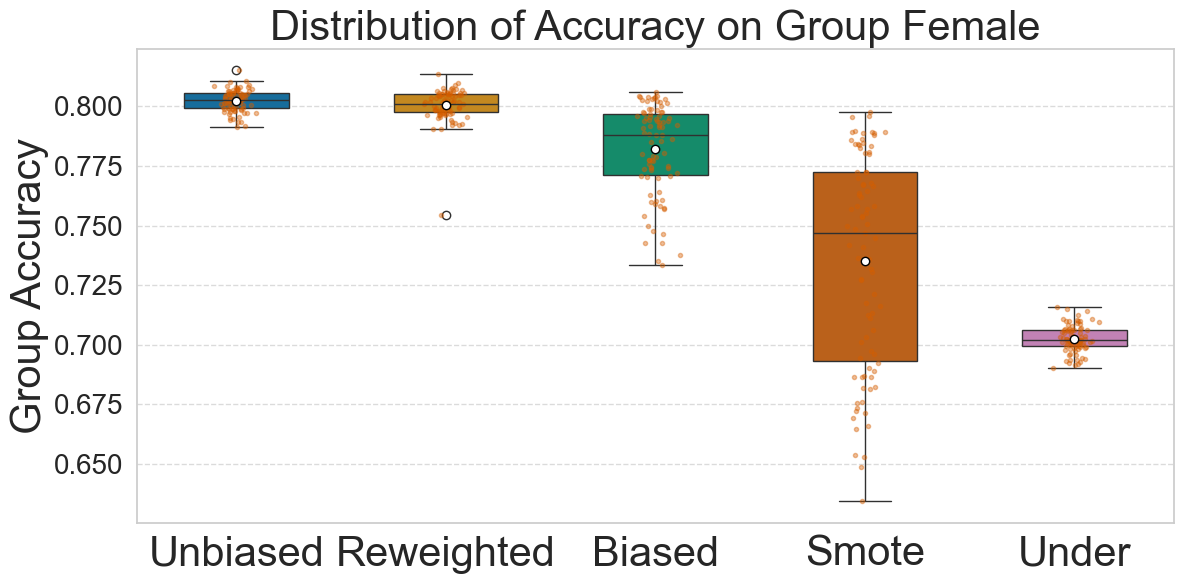}
\caption{Female}
\end{subfigure}
\caption{Group Accuracies on Adult Data Set}
\end{figure}
\label{fig:groups_adult}

\subsubsection{ACS Employment Data Set}
Figure~\ref{fig:groups_acs} delves into the performance of various data preparation methods across different demographic groups within the ACS Employment data set. It compares the accuracies of candidate models across seven demographic groups: \textit{Black, American Indian, Alaska Native, Other Native, Asian, Native Hawaiian}, and \textit{Female}. Across all demographic groups in the ACS Employment data set, the Reweighted model consistently achieves accuracies close to the Unbiased model, supporting its use as a robust bias mitigation strategy. The Biased model consistently shows lower accuracies, emphasizing the presence of bias in the original dataset. SMOTE and Under-sampling techniques exhibit variability and generally fail to provide the stability and accuracy of the Reweighted approach. This analysis underlines the importance of implementing effective data preprocessing methods to ensure fair and accurate machine learning outcomes across diverse demographic groups. %The Reweighting approach, in particular, appears to be a potent strategy for achieving equitable performance in employment-related predictive modeling.

\begin{figure}[h]
\begin{subfigure}{0.33\textwidth}
\includegraphics[width=\textwidth]{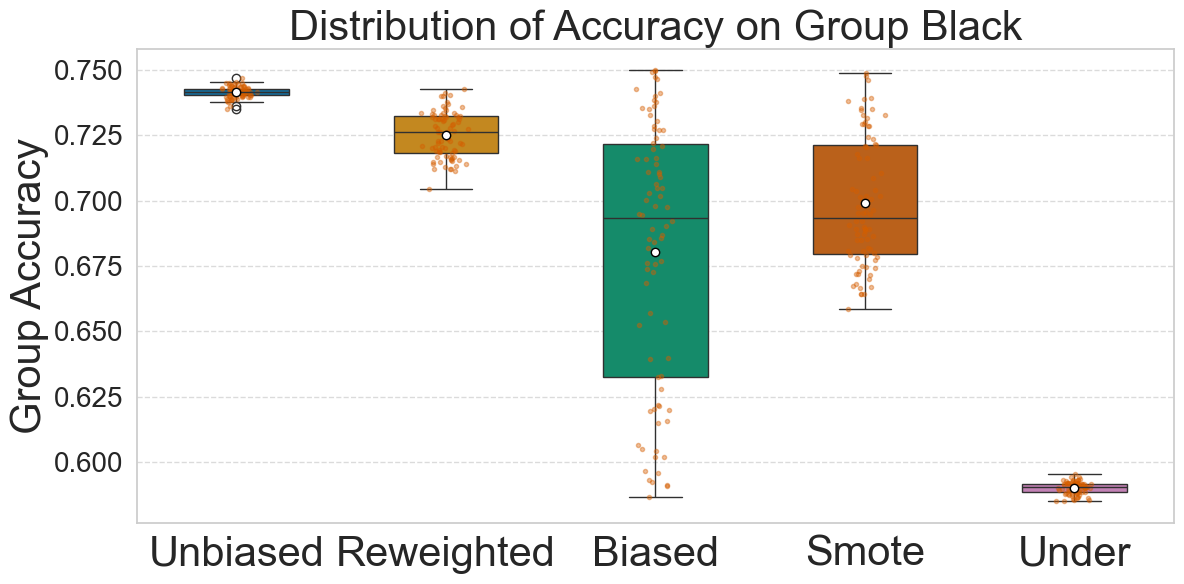}
\caption{Black}
\end{subfigure}
\begin{subfigure}{0.33\textwidth}
\includegraphics[width=\textwidth]{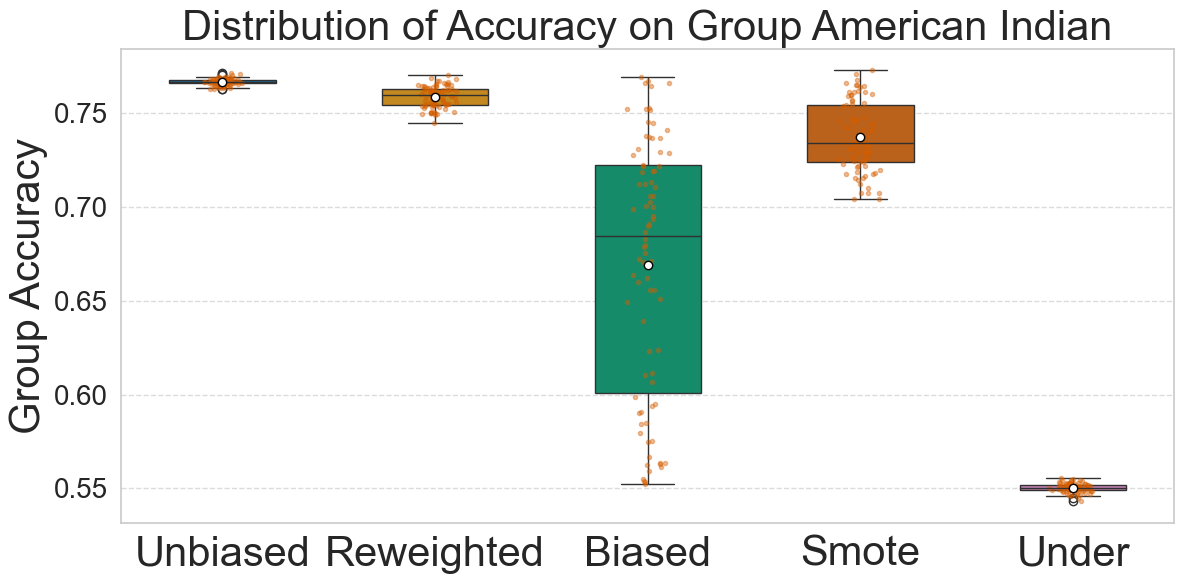}
\caption{Amererican Indian}
\end{subfigure}
\begin{subfigure}{0.33\textwidth}
\includegraphics[width=\textwidth]{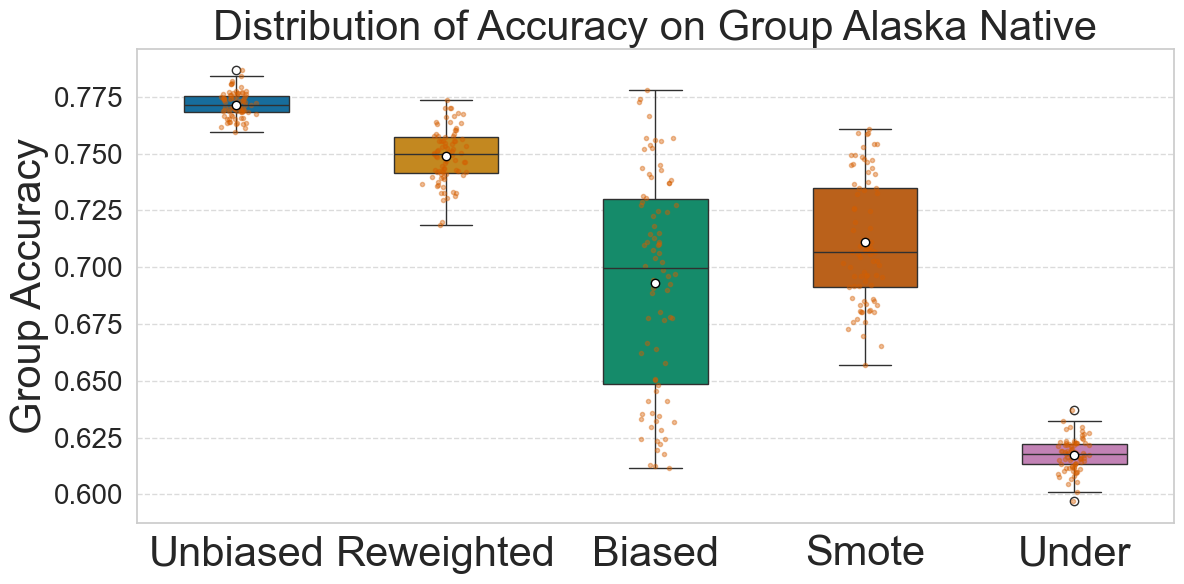}
\caption{Alaska Native}
\end{subfigure}
\begin{subfigure}{0.24\textwidth}
\includegraphics[width=\textwidth]{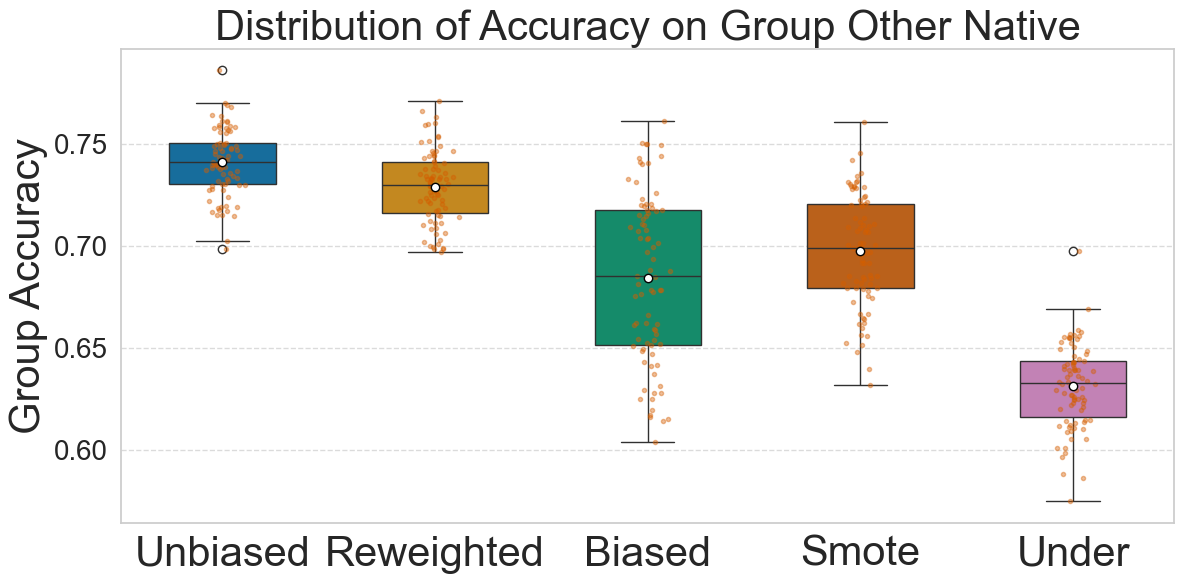}
\caption{Other Native}
\end{subfigure}
\begin{subfigure}{0.24\textwidth}
\includegraphics[width=\textwidth]{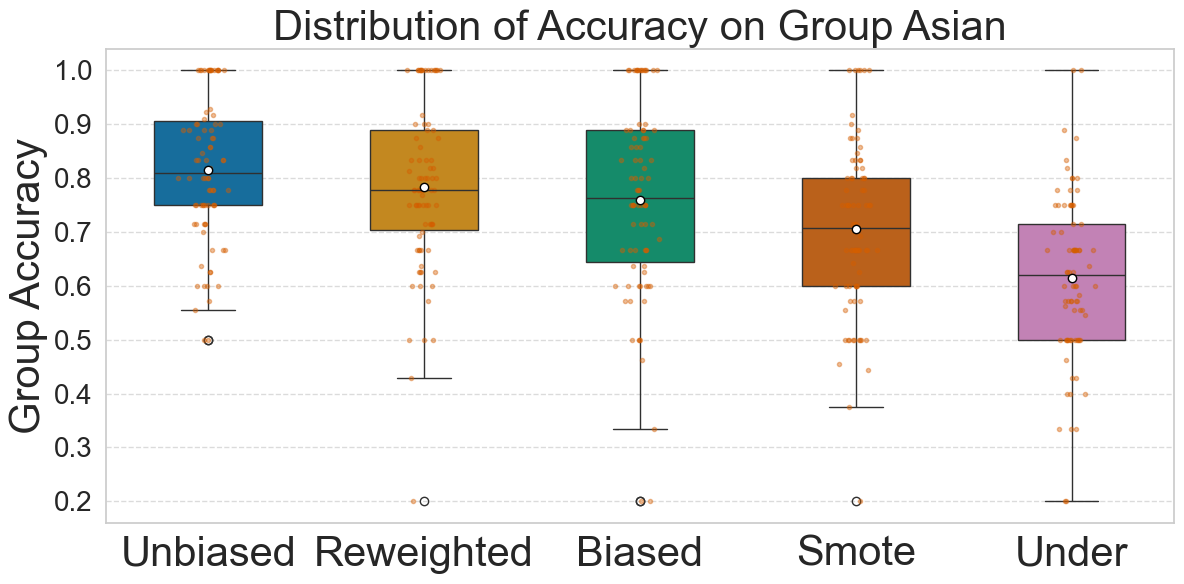}
\caption{Asian}
\end{subfigure}
\begin{subfigure}{0.24\textwidth}
\includegraphics[width=\textwidth]{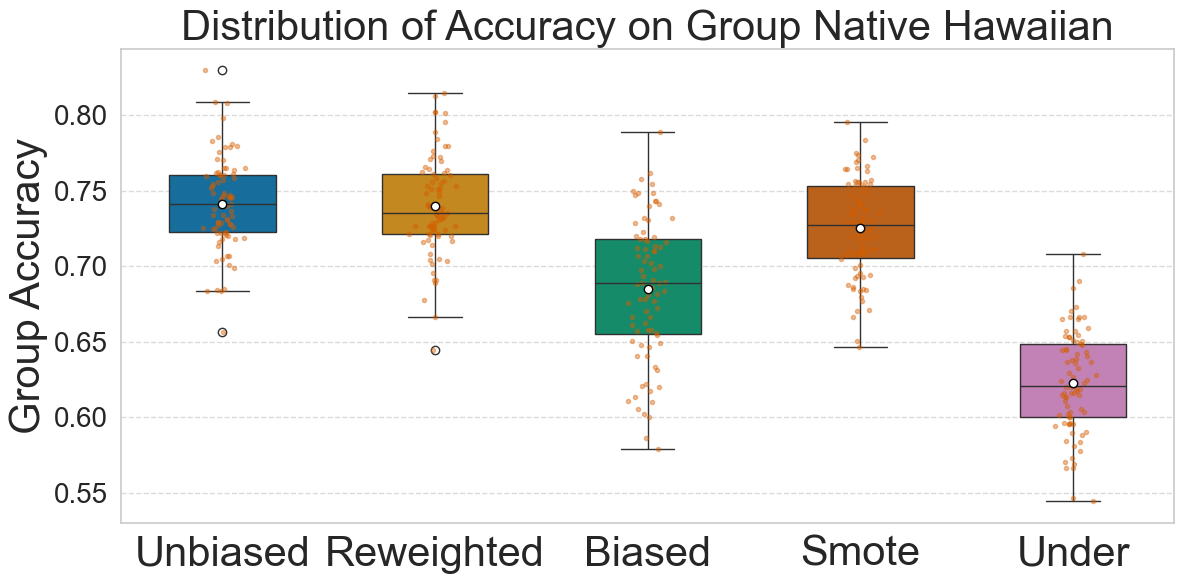}
\caption{Native Hawaiian}
\end{subfigure}
\begin{subfigure}{0.24\textwidth}
\includegraphics[width=\textwidth]{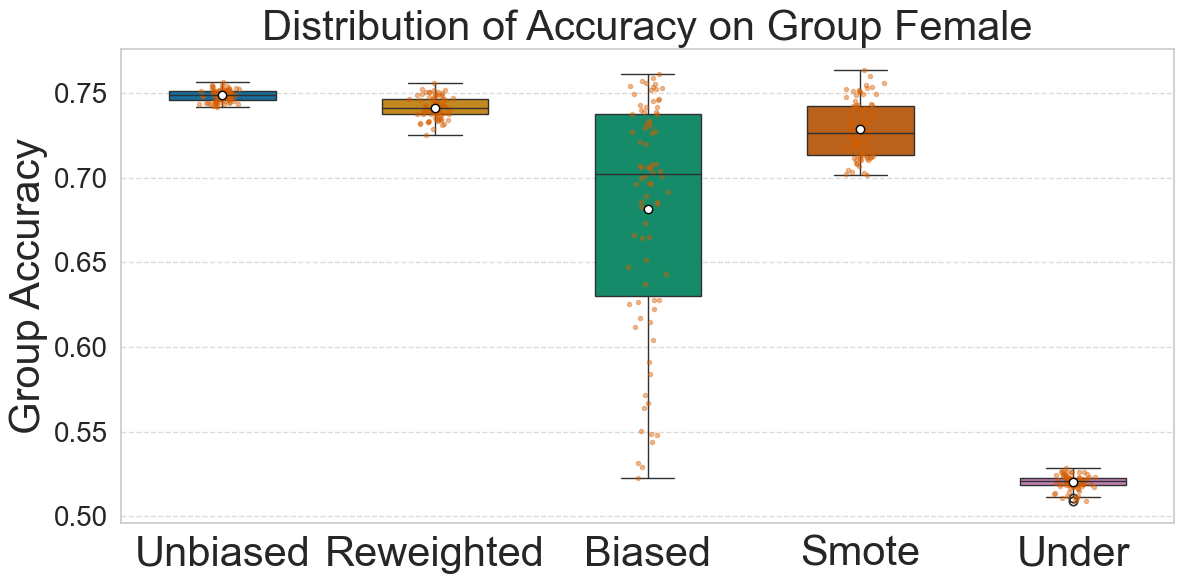}
\caption{Female}
\end{subfigure}
\caption{Group Accuracies on ACS Employment Data Set}
\label{fig:groups_acs}
\end{figure}

%\subsection{Discussion}
%\paragraph{\textbf{Summary}}
%Across all datasets, the Reweighted model consistently demonstrates robust performance, closely approximating or even matching the Unbiased model, and surpassing the Biased, SMOTE, and Under-sampling models in most cases. The variability in performance across different datasets underscores the necessity of dataset-specific approaches and the importance of extensive testing across multiple runs to ensure the reliability of predictive models. 
 %Perhaps the primary limitation in our approach is the assumption of group independence that went into the theoretical guarantees. Another interesting observation is that the performance of our re-weighting algorithm matches the performance of the model trained on unbiased data most closely on the Adult dataset, which is also the dataset on which we observed the smallest proportion of significant p-values for the $\chi^2$ test. The two performances were most different on the COMPAS dataset, which in turn had the highest portion of significant p-values for the $\chi^2$ test.

\section{Conclusion}
Our model and accompanying algorithm provide a flexible and generalizable method of estimating bias across multiple intersectional categories without any prior assumptions about the degree of bias in each group. The model also has the strength that once the bias parameters are estimated, they can be used on future draws of biased data to find a hypothesis approximately minimizing prediction risk with respect to the true distribution with high probability. The experiments demonstrate that with only a small amount of unbiased data, we can use Algorithm~\ref{algo:biased_learning} to consistently produce a more accurate model than would have been produced by training on the biased data or rebalanced the data without an understanding of the drop-out rates. We also find empirically that our reweighting approach typically obtains group accuracy very close to that of a model trained on the unbiased dataset. Limitations include the assumption that group membership is independent, which may oversimplify the complexities of social power structures and how they intersect in real-world contexts. In response, we demonstrated contexts in which this is a reasonable assumption in practice and found that, even when this assumption was not consistent with our data, our method still performed admirably. Finally, we only consider underrepresentation bias but our model in principle can be extended to characterize other types of biases.

 \bibliographystyle{plainnat}
\bibliography{bibliography}
%%%%%%%%%%%%%%%%%%%%%%%%%%%%%%%%%%%%%%%%%%%%%%%%%%%%%%%%%%%%
\appendix
%\newpage
%$ $
%\newpage
\appendix
\section{Appendix}
\subsection{Additional Discussion on Independence Assumptions}
In our initial exploration of modeling intersectionality and bias, we posited that only the bias parameters ($\beta$) were independent, not group membership ($G$). The core challenge we faced was determining how to effectively bound the error on the estimates of $(\hat{\beta}_i)$, especially since each individual might belong to a diverse array of groups, thereby preventing straightforward estimation from samples due to unique bias multipliers for each person. Traditional approaches, such as those employed to derive confidence intervals around linear regression coefficients, prompted us to investigate if analogous techniques could be applicable in our context. The analogy drawn likened $\log (\beta_i)$ to linear regression coefficients, framing our task as an estimation problem. (To potentially utilize linear regression methodologies directly, we considered restructuring our problem as a conventional linear regression model, where examples are set up as $(x,y)$ pairs such that $\mathbb{E}[y]=\langle \theta,x \rangle$ for each $x$ for some vector $\theta$. This formulation would necessitate transforming our model, likely treating the $x$ as an indicator vector of group membership for individuals, with $\theta$ representing the vector of $\log (\beta_i)$, though the precise nature of the $y$ values remained unclear.) The crux of the approach hinged on the expectation that $\mathbb{E}[y]$ could be expressed as a linear combination of the $\log(\beta_i)$ coefficients, akin to $\mathbb{E}[y]=\langle \theta, x\rangle$ with $\theta$ embodying the log-transformed bias parameters.

However, transitioning this conceptual framework into a practical estimation technique encountered significant challenges. Firstly, the independence required between each $(x,y)$ pair in classical linear regression analyses was not inherently present in our setting, where $y$'s were derived from a singular dataset, introducing dependencies that complicated the application of off-the-shelf regression techniques. Additionally, the problem space was characterized by heteroskedasticity, implying that the variance of our outcome variable ($y$) was not constant across observations but varied in accordance with group membership and the associated bias multipliers. The ``standard error" bounds typically assume that the outcome variable $y$ can be expressed as a linear combination of the predictors $x$, augmented by an additive noise term $N$, where this noise is distributed independently of $x$. This is formalized as $y=\langle \theta,x \rangle+N$, with $N$ representing the noise distribution that's assumed to be independent from the predictors $x$. However, this assumption does not hold in our context, since the "noise" we encounter arises from sampling errors within each group, compounded by the fact that these groups vary in size. This deviation from standard assumptions necessitates a tailored approach to estimating and bounding errors in our model.

Given these challenges, we shifted our model to assume independence not only in the bias parameters $(\beta)$ but also in group membership $(G)$. Under this assumption, the independence of group membership enables us to accurately estimate ratios using Chernoff bounds for both the numerator and denominator, which in turn allows for precise estimation of $\hat{p}_i$ and $\hat{\beta}_i$ with a solid level of confidence. Specifically, by utilizing the unbiased training set $S$ and the biased training set $S_\beta$, we can calculate $\hat{p}_i$ for each group $i$ as the proportion of positive outcomes within that group in $S$, and similarly $\hat{p}_{\beta_i}$ as the proportion of positive outcomes within group $i$ in $S_\beta$. These proportions give us a direct method to estimate the inherent bias $\hat{\beta}$ for each group by comparing their ratios in the unbiased and biased scenarios. This methodological pivot towards leveraging statistical properties of independence and employing Chernoff bounds for error estimation facilitated a tractable solution to estimate our parameters of interest within the constraints of intersectionality modeling. 
\subsection{Missing Proofs}

\prodPosRate*
\begin{proof}
By Bayes' rule, 
\begin{align*}
    &\Pr[y=1 | \textbf{x} \in \cap_{i \in I} G_i] = \frac{\Pr[ \textbf{x} \in \cap_{i \in I} G_i | y = 1]}{\Pr[\textbf{x} \in \cap_{i \in I} G_i]} \Pr[y=1]
\end{align*}

Applying Assumption~\ref{ass:indep}, it follows that 
\begin{align*}
    &\Pr[y=1 | \textbf{x} \in \cap_{i \in I} G_i] = \frac{\Pr[\textbf{x} \in \cap_{i \in I} G_i | y = 1]}{\prod_{i \in I} \Pr[\textbf{x} \in G_i]}Pr[y=1] 
\end{align*}

Applying Assumption~\ref{ass:cond_indep} yields 
\begin{align*}
    &\Pr[y=1 | \textbf{x} \in \cap_{i \in I} G_i] = \Pr[y=1]\frac{\prod_{i \in I} \Pr[x \in G_i | y=1]}{\prod_{i \in I} \Pr(\textbf{x} \in G_i)} 
\end{align*}

Then, using Bayes' rule once again: 
\begin{align*}
&\Pr[y=1 | \textbf{x} \in \cap_{i \in I} G_i] = \Pr[y=1]\frac{\prod_{i \in I} \frac{\Pr[y=1 | \textbf{x} \in G_i] \Pr[\textbf{x} \in G_i]}{\Pr[y=1]}}{\prod_{i \in I} \Pr[\textbf{x} \in G_i]}
\end{align*}
and thus by simplifying the terms, the lemma is proved: \[\Pr[y=1|\textbf{x} \in \bigcap_{i \in I} G_i] = \Pr[y=1]^{1-\lvert I\vert } \prod_{i \in I} p_i \]
\end{proof}

\inverseBeta*

\begin{proof}
We begin by considering the quantity \( p_i \cdot \beta_i \), the probability that an element of group $G_i$ is both positive and included in the biased dataset \( S_\beta \). Since the bias parameter \( \beta_i \) represents the probability that a positive sample from group \( G_i \) is retained in \( S_\beta \), the probability of a positive outcome within group \( G_i \) in \( S_\beta \) is the product of the original base positive rate \( p_i \) and the bias parameter \( \beta_i \):

\begin{align*}
    p_i \cdot \beta_i &= \Pr[y = 1 | \textbf{x} \in G_i] \Pr[(\textbf{x},y) \in S_\beta| \textbf{x} \in G_i, y=1] = \Pr[y=1, (\textbf{x},y) \in S_\beta| \textbf{x} \in G_i]
\end{align*}
 
Expanding further, we see that 
%We will write \( \mathbf{p \beta} = (p_1 \beta_1, p_2 \beta_2, \ldots, p_k \beta_k) \) to indicate the vector consisting of individual group base positive rates in the biased dataset. Finally, we represent the probability that a sample is positive given its group membership and presence in the biased dataset by $p^*$:
%\begin{align*}
%p_i^* =  \Pr[y=1 | (\textbf{x},y) \in S_\beta, \textbf{x} \in G_i]
%\end{align*}

    %= \Pr[y=1, (\textbf{x},y) \in S_\beta| \textbf{x} \in G_i] \\
    %&= \Pr[y=1 | (\textbf{x},y) \in S_\beta, \textbf{x} \in G_i] \Pr[(\textbf{x},y) \in S_\beta | \textbf{x} \in G_i]\\
    \begin{align*}
    p_i \cdot \beta_i &= \Pr[y=1 | (\textbf{x},y) \in S_\beta, \textbf{x} \in G_i] \left(\Pr[(\textbf{x},y) \in S_\beta, y=1 | \textbf{x} \in G_i] + \Pr[(\textbf{x},y) \in S_\beta, y=0 | \textbf{x} \in G_i]\right)\\
    &= \Pr[y=1 | (\textbf{x},y) \in S_\beta, \textbf{x} \in G_i] \left(\Pr[(\textbf{x},y) \in S_\beta| y=1, \textbf{x} \in G_i]\Pr[y=1|\textbf{x} \in G_i] + \Pr[(\textbf{x},y) \in S_\beta| y=0, \textbf{x} \in G_i]\Pr[y=0|\textbf{x} \in G_i]\right)\\
    &= \Pr[y=1 | (\textbf{x},y) \in S_\beta, \textbf{x} \in G_i]\cdot \left(p_i \cdot \beta_i + 1 - p_i\right)\\
\end{align*}

Substituting $p_{\beta_i} =  \Pr[y=1 | (\textbf{x},y) \in S_\beta, \textbf{x} \in G_i]$, we have:
\begin{align*}
 &p_i \cdot \beta_i = p_{\beta_i} \left(p_i \cdot \beta_i + 1 - p_i \right) = p_{\beta_i} \cdot p_i \cdot \beta_i + p_{\beta_i}(1 - p_i) \\
 \implies &p_i \cdot \beta_i - p_{\beta_i} \cdot p_i \cdot \beta_i = p_{\beta_i}(1 - p_i) \\
 \implies &\frac{1}{\beta_i} = \frac{p_i \left( 1 - p_{\beta_i} \right)}{p_{\beta_i}(1 - p_i)}
\end{align*}
\end{proof}

\prodBeta*

\begin{proof}
We may follow the structure of the proof to Lemma~\ref{lem:prod_pos_rate}:
\begin{align*}
\Pr[(\textbf{x},y) \in S_\beta|\textbf{x} \in \cap_{i \in I} G_i, y=1] &= \frac{\Pr[(\textbf{x},y) \in S_\beta,\textbf{x} \in \cap_{i \in I} G_i, y=1]}{\Pr[\textbf{x} \in \cap_{i \in I} G_i, y=1]} \\
&= \frac{\Pr[\textbf{x} \in \cap_{i \in I} G_i | (\textbf{x},y) \in S_\beta, y=1] \Pr[(\textbf{x},y) \in S_\beta, y=1]}{\Pr[\textbf{x} \in \cap_{i \in I} G_i, y=1]}\\
&= \frac{\Pr[\textbf{x} \in \cap_{i \in I} G_i | (\textbf{x},y) \in S_\beta, y=1] \Pr[(\textbf{x},y) \in S_\beta | y=1]}{\Pr[\textbf{x} \in \cap_{i \in I} G_i | y=1]}\\
&=\frac{\prod_{i \in I} \Pr[\textbf{x} \in G_i|(\textbf{x},y) \in S_\beta, y=1] \Pr[(\textbf{x},y) \in S_\beta | y=1]}{\prod_{i \in I} \Pr[x \in G_i | y=1]} \\
&= \Pr[(\textbf{x},y) \in S_\beta | y=1] \prod_{i \in I} \frac{\frac{\Pr[(\textbf{x},y) \in S_\beta | \textbf{x} \in G_i, y=1] \Pr[\textbf{x} \in G_i| y=1]}{\Pr[(\textbf{x},y) \in S_\beta| y=1]}}{\Pr[x \in G_i | y=1]} \\
&=\beta_0^{1-\lvert I\rvert} \prod_{i \in I} \beta_i
\end{align*}
\end{proof}

\subsection{Proof of Theorem~\ref{thm:biased_pmf}}

\thmDistributions*
\begin{comment}
\begin{proof}
\begin{align*}
    &\Pr_{(\textbf{x},y) \sim \cD_\beta}[(\textbf{x},y)] \propto \Pr_{(\textbf{x},y) \sim \cD}[(\textbf{x}=\textbf{X}, y=Y)]\frac{1}{w(\textbf{x},y)}\\
    \implies &\Pr_{(\textbf{x},y) \sim \cD}[(\textbf{x}=\textbf{X}, y=Y)] \propto \Pr_{\cD_\beta}[(\textbf{x},y)] w(\textbf{x},y)
\end{align*}

Normalize and expand.
\end{proof}

%\begin{comment}
\begin{proof}
\begin{align*}
    %&p_{\cD_\beta}(\textbf{x}, y) \\:=&
    &\Pr_\cD[(\textbf{x}=X, y=Y) \in S_\beta] \\= &\Pr_\cD[(\textbf{x} = X, y=Y) \in S] \\ 
    \cdot &\Pr_{\cD}[(\textbf{x} = X, y=Y) \in S_\beta | (\textbf{x} = X, y=Y) \in S] \\
    = &p_\cD (\textbf{x},y)\Pr_{\cD}[(\textbf{x} = X, y=Y) \in S_\beta | (\textbf{x} = X, y=Y) \in S] \\
    = &p_\cD (\textbf{x},y) \left(\mathbb{I}(Y=0) + 
    \mathbb{I}(Y=1) \right.\\\cdot &\left.\Pr[(\textbf{x} = X, y=Y) \in S_\beta | G(\textbf{x}), Y=1, (\textbf{x} = X, y=Y) \in S] \right)\\
    &= \frac{p_\cD (\textbf{x},y)}{w(\textbf{x},y)}
\end{align*}

This implies that
\begin{align*}
    p_\cD (\textbf{x},y) = w(\textbf{x},y) p_{\cD_\beta}(\textbf{x},y)
\end{align*}
%If $y=0$, any sample included in $S$ will be included in $S_\beta$. Therefore, $p_{\cD_\beta}(\textbf{x},y) \propto p_\cD (\textbf{x},y).$

%If $y=1$,
\end{proof}
\end{comment}

\begin{proof}

First, we prove that $\mathbb{E}_\cD[\frac{1}{w(\textbf{x},y)}]=\frac{1}{\mathbb{E}_{\cD_\beta}[w(\textbf{x},y)]}$. To do so, we simply use the fact that the total probability must sum to one:
\begin{align*}
1 &= \sum_{(\textbf{x},y)}p_\cD(\textbf{x},y) \\
&= \sum_{(\textbf{x},y)} \mathbb{E}_\cD [\frac{1}{w(\textbf{x},y)}]w(\textbf{x},y)p_{\cD_\beta}{(\textbf{x},y)} \\
&= \mathbb{E}_\cD [\frac{1}{w(\textbf{x},y)}] \sum_{(\textbf{x},y)} w(\textbf{x},y)p_{\cD_\beta}{(\textbf{x},y)} \\
&= \mathbb{E}_\cD [\frac{1}{w(\textbf{x},y)}] \mathbb{E}_{\cD_\beta} w(\textbf{x},y) 
\end{align*} 
which implies $\mathbb{E}_\cD[\frac{1}{w(\textbf{x},y)}]=\frac{1}{\mathbb{E}_{\cD_\beta}[w(\textbf{x},y)]}$. Using this relationship and the definition of the re-weighting function, $w(\textbf{x},y)$, we can immediately rewrite the definition of $p_{\cD_\beta}(\textbf{x},y)$ as follows:

\begin{align*}
p_{\cD_{\beta}}(\textbf{x},y) &=\frac{\Pr[(\textbf{x},y) \in S_{\beta}|(\textbf{x},y)=(\textbf{X},Y)]p_\cD(\textbf{x},y)}{\mathbb{E}_\cD [\Pr[(\textbf{x},y) \in S_{\beta}|(\textbf{x},y)=(\textbf{X},Y)]]} \\&
=\frac{\frac{p_\cD(\textbf{x},y)}{w(\textbf{x},y)}}{\mathbb{E}_\cD [\frac{1}{w(\textbf{x},y)}]} \\&
=\frac{p_\cD(\textbf{x},y)\mathbb{E}_\cD [\frac{1}{w(\textbf{x},y)}]}{w(\textbf{x},y)}
\end{align*}

\end{proof}

\subsection{Discussion and Proof Sketch of Theorem \ref{theorem:mainTheorem}}
\begin{figure}[h]
\centering
\resizebox{0.75\textwidth}{!}{
\tikzset{every picture/.style={line width=0.75pt}} %set default line width to 0.75pt        

\begin{tikzpicture}[x=0.75pt,y=0.75pt,yscale=-1,xscale=1]
%uncomment if require: \path (0,429); %set diagram left start at 0, and has height of 429

%Shape: Circle [id:dp9304219008821741] 
\draw  [fill={rgb, 255:red, 74; green, 144; blue, 226 }  ,fill opacity=0.7 ][line width=0.75]  (61.58,83.21) .. controls (61.58,56.03) and (83.61,34) .. (110.79,34) .. controls (137.97,34) and (160,56.03) .. (160,83.21) .. controls (160,110.39) and (137.97,132.42) .. (110.79,132.42) .. controls (83.61,132.42) and (61.58,110.39) .. (61.58,83.21) -- cycle ;
\draw   (97,64.21) -- (116,64.21)(106.5,54) -- (106.5,74.42) ;
\draw   (124,73.21) -- (143,73.21)(133.5,63) -- (133.5,83.42) ;
\draw   (74,83.21) -- (93,83.21)(83.5,73) -- (83.5,93.42) ;
%Straight Lines [id:da17427349745305487] 
\draw [color={rgb, 255:red, 208; green, 2; blue, 27 }  ,draw opacity=1 ]   (186,93) -- (341,93.42) ;
\draw [shift={(343,93.42)}, rotate = 180.15] [color={rgb, 255:red, 208; green, 2; blue, 27 }  ,draw opacity=1 ][line width=0.75]    (10.93,-3.29) .. controls (6.95,-1.4) and (3.31,-0.3) .. (0,0) .. controls (3.31,0.3) and (6.95,1.4) .. (10.93,3.29)   ;
\draw   (119,109.21) -- (138,109.21)(128.5,99) -- (128.5,119.42) ;
\draw   (133,90.21) -- (152,90.21)(142.5,80) -- (142.5,100.42) ;
\draw   (95,97.21) -- (114,97.21)(104.5,87) -- (104.5,107.42) ;
%Shape: Circle [id:dp9718901662370412] 
\draw  [fill={rgb, 255:red, 208; green, 2; blue, 27 }  ,fill opacity=0.55 ] (380.58,67.42) .. controls (380.58,31.52) and (409.68,2.42) .. (445.58,2.42) .. controls (481.48,2.42) and (510.58,31.52) .. (510.58,67.42) .. controls (510.58,103.32) and (481.48,132.42) .. (445.58,132.42) .. controls (409.68,132.42) and (380.58,103.32) .. (380.58,67.42) -- cycle ;
\draw   (412,73.21) -- (431,73.21)(421.5,63) -- (421.5,83.42) ;
\draw   (439,42.21) -- (458,42.21)(448.5,32) -- (448.5,52.42) ;
\draw   (414,31.21) -- (433,31.21)(423.5,21) -- (423.5,41.42) ;
\draw   (447,73.21) -- (466,73.21)(456.5,63) -- (456.5,83.42) ;
%Straight Lines [id:da34536302718564804] 
\draw [color={rgb, 255:red, 144; green, 19; blue, 254 }  ,draw opacity=1 ]   (445,388.42) -- (259,387.42) -- (203,387.42) ;
\draw [shift={(201,387.42)}, rotate = 360] [color={rgb, 255:red, 144; green, 19; blue, 254 }  ,draw opacity=1 ][line width=0.75]    (10.93,-3.29) .. controls (6.95,-1.4) and (3.31,-0.3) .. (0,0) .. controls (3.31,0.3) and (6.95,1.4) .. (10.93,3.29)   ;
%Down Arrow [id:dp5203829752654565] 
\draw  [fill={rgb, 255:red, 74; green, 144; blue, 226 }  ,fill opacity=0.77 ] (135.47,264.01) -- (165.16,236.48) -- (113.01,180.22) -- (172.39,125.18) -- (224.54,181.44) -- (254.23,153.92) -- (229.62,246.47) -- cycle ;
%Down Arrow [id:dp862296079172109] 
\draw  [color={rgb, 255:red, 0; green, 0; blue, 0 }  ,draw opacity=1 ][fill={rgb, 255:red, 144; green, 19; blue, 254 }  ,fill opacity=1 ] (260.01,239.14) -- (260.39,265.26) -- (407.75,263.1) -- (408.51,315.33) -- (261.15,317.49) -- (261.54,343.61) -- (162.53,292.82) -- cycle ;
%Down Arrow [id:dp6706942357896564] 
\draw  [fill={rgb, 255:red, 208; green, 2; blue, 27 }  ,fill opacity=0.55 ] (286.61,146.99) -- (314.99,174.13) -- (374.4,111.98) -- (431.18,166.26) -- (371.77,228.4) -- (400.16,255.54) -- (303.78,242.7) -- cycle ;
%Shape: Circle [id:dp8151865177723263] 
\draw  [fill={rgb, 255:red, 74; green, 144; blue, 226 }  ,fill opacity=0.08 ][line width=0.75]  (6.58,306.21) .. controls (6.58,266.88) and (38.46,235) .. (77.79,235) .. controls (117.12,235) and (149,266.88) .. (149,306.21) .. controls (149,345.54) and (117.12,377.42) .. (77.79,377.42) .. controls (38.46,377.42) and (6.58,345.54) .. (6.58,306.21) -- cycle ;
% Text Node
\draw (155,31) node [anchor=north west][inner sep=0.75pt]   [align=left] {{\huge \textbf{Underrepresentation Bias}}};
% Text Node
\draw (201,352) node [anchor=north west][inner sep=0.75pt]   [align=left] {{\huge \textbf{Reweighting Transformation}}};
% Text Node
\draw (136,188) node [anchor=north west][inner sep=0.75pt]   [align=left] {{\large \textbf{Estimate }$\displaystyle p$}};
% Text Node
\draw (299.04,178.74) node [anchor=north west][inner sep=0.75pt]   [align=left] {{\large \textbf{Estimate }$\displaystyle p_{\beta }$}};
% Text Node
\draw (240,277) node [anchor=north west][inner sep=0.75pt]   [align=left] {{\large \textbf{Estimate }$\displaystyle w$\textbf{ }}};
% Text Node
\draw (15,285) node [anchor=north west][inner sep=0.75pt]   [align=left] {{\large \textbf{Model Training}}\\{\large \textbf{Dataset}}};
% Text Node
\draw (7,139) node [anchor=north west][inner sep=0.75pt]   [align=left] {{\large \textbf{Small Unbiased}}\\{\large \textbf{Dataset}}};
% Text Node
\draw (437,143) node [anchor=north west][inner sep=0.75pt]   [align=left] {{\huge \textbf{Large Biased Dataset}}};
\end{tikzpicture}
}
\caption{Reweighting to Approximate $\mathcal{D}$}
\end{figure}
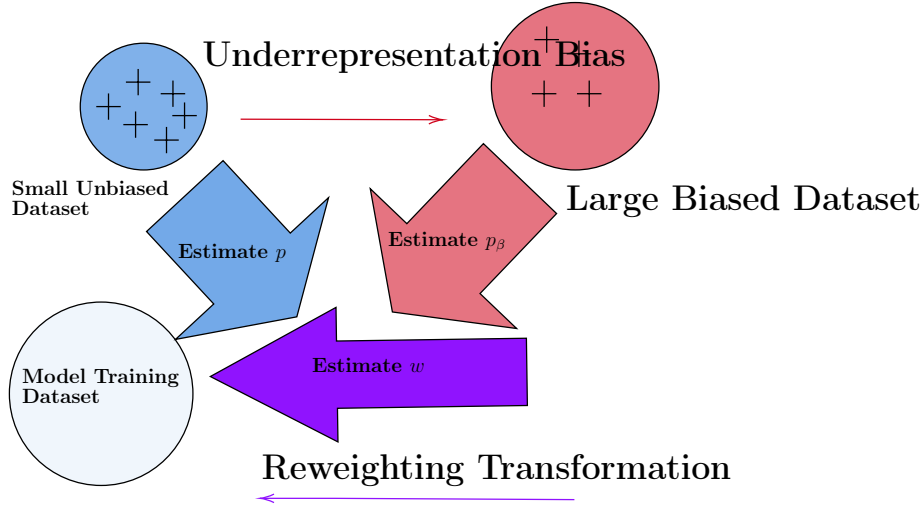

\paragraph{Proof Sketch}
The proof of the main theorem consists of three parts:

\textbf{Part A}: This part focuses on normalizing the weighted empirical loss on the biased sample. The goal is to bound the difference between the sum of weights in the biased sample and the reciprocal of the expected weight. By achieving this normalization, we can account for the bias introduced by the weights and ensure that the algorithm's performance is affected only minimally. See Lemmas~\ref{lemma:A2} and ~\ref{lemma:A3}. %This normalization step is crucial for the subsequent analysis of the proof. 

\textbf{Part B}: Part B of the proof is concerned with estimating the reweighting function. It involves first obtaining reliable estimates of the probabilities of positive examples in each group of the unbiased data and biased data, respectively, which is used to estimate $\beta$ and, finally, the maximum value the reweighting function can take on. The key lemma established here specifies the sample size requirements to achieve reliable estimation with high probability, showing that we can estimate the probabilities of positive examples in each group with reasonable accuracy by having only a small number of unbiased samples in addition to our larger set of biased samples. See Lemmas~\ref{lemma:estimating_p} - \ref{lemma:beta_0_inverse} for details.
%\beta^{-1}$ and, finally, $\prod_{i=1}^{k} \beta_i^{-1}$
 
\textbf{Part C}: Part C combines the previous results to establish the sample complexity guarantee for Algorithm~\ref{algo:biased_learning}. It utilizes the bounds obtained from the lemmas established in Parts A and B, which cover both the intersectional case where the groups in the biased sample have overlapping instances and the computationally simpler case where they are disjoint. By combining these results, the proof demonstrates that running Algorithm~\ref{algo:biased_learning} with appropriate sample sizes outputs a hypothesis $h$ with a low error on the true distribution $\cD$ with high probability. 
See Lemmas~\ref{lemma:C1}~ - \ref{lemma:C2} for details.

\textbf{Part D}: Part D, Lemma~\ref{lemma:D1} and Theorem~\ref{theorem:mainTheorem}, tie everything together and establish the algorithm's sample complexity, providing the desired guarantee for its performance. This is novel in that a small amount of unbiased data can take us from ``impossible" to ``nearly optimal," at least in the context of underrepresentation and intersectional bias. The core observation and benefit is that only a small fraction of your data needs to be unbiased. The unbiased data can be used to estimate the base rate. We then can use a large quantity of biased data to learn the model.

The intermediate lemmas are as follows:
\begin{enumerate}

   \item[] \textbf{Lemma \ref{lemma:A2}}: The absolute difference between the estimate of the normalizing factor and the true normalizing factor is bounded with probability $1-\frac{\delta}{2k+2}$. The proof uses the Hoeffding bound.
   \item[] \textbf{Lemma \ref{lemma:A3}}: The reweighted loss $L_{S_\beta}(h)$ normalized by the empirical estimate is always less than or equal to 1. The proof uses the sum of weighted indicator functions.
   \item[] \textbf{Lemma \ref{lemma:estimating_p}}: The sample size required to estimate $p$ within a certain level of precision is given by the multiplicative Chernoff Bound. The proof uses the concept of Bernoulli trials and the multiplicative Chernoff Bound.
   \item[] \textbf{Lemma \ref{lemma:estimating_p_beta}}: The sample size required to estimate $p_{\beta_i}$ within a certain level of precision is given by the multiplicative Chernoff Bound. The proof uses the concept of indicator variables and the multiplicative Chernoff bound. 
    \item[] \textbf{Lemma \ref{lemma:estimating_beta_inverse}}: The sample size required to estimate $\mathbf{\frac{1}{\beta}}$ within a certain level of precision is given by the multiplicative Chernoff Bound. The proof directly uses the results from Lemmas~\ref{lemma:estimating_p} and ~\ref{lemma:estimating_p_beta}.
   \item[] \textbf{Lemma \ref{lemma:B9}}: If the sizes of the samples satisfy certain conditions, then with high probability, the estimated inverse of the product of biases, $\prod_{i=1}^{k} \widehat{\beta_i^{-1}}$, is close to the true inverse of the product of biases, $\prod_{i=1}^{k} \beta_i^{-1}$, within a margin of $\frac{3\epsilon}{5}$. The proof uses the bounds from Lemma B8 and applies these bounds to the product of inverse biases.
   \item[] \textbf{Lemma \ref{lemma:C1}}: The expected loss of hypothesis $h$ on the distribution $\cD$ is equal to the expected loss of $h$ on the weighted distribution $\cD_\beta$ scaled by the inverse of the expected weight $\mathbb{E}_{\cD_\beta}[w(\textbf{x},y)]$ of the weighted distribution. The proof follows from the definition of expected value.
   \item[] \textbf{Lemma \ref{lemma:C2}}: A sample size of at least $m_\beta \geq \frac{11^2}{{2\epsilon^2 \prod_{j=1}^{k} \beta_j ^2}}\ln{\frac{4|H|(k+1)}{\delta}}$ is required to estimate the true expected loss $L_D(h)$ using a reweighting of the biased sample loss $L_{S_\beta}(h)$. The proof uses a multiplicative Hoeffding bound.
   \item[]  \textbf{Lemma \ref{lemma:D1}}: The risk difference between the learned hypothesis $h$ on the biased sample $S_\beta$ and the risk of $h$ on the true distribution $\cD$ is bounded with probability $1-\delta$. The proof uses triangle inequality and standard concentration inequalities.
\end{enumerate}

By relying on these lemmas, Theorem \ref{theorem:mainTheorem} is proven, demonstrating the reliability and accuracy of the proposed method for empirical risk minimization in biased datasets.  
%\subsection{Proof of Theorem~\ref{theorem:mainTheorem}}

\subsubsection{Part A}
 
\begin{restatable}{lemma}{normalizing} \label{lemma:A2}
With $m_\beta \geq \frac{\left(\beta_0^{k-2} \min_i {\beta_i}^{1-k} - 1\right)^2}{2\epsilon^2\mathbb{E}_{\cD_\beta}[w(\textbf{x},y)]^2}\ln{\frac{4(k+3)}{\delta}}$ samples from $\cD_\beta$: \\
$\Pr \left[ \lvert \frac{m_\beta}{\sum_{i=1}^{m_\beta} w(\textbf{x}_i, y_i)}  - \frac{1}{\mathbb{E}_{\cD_\beta}[w(\textbf{x},y)]} \rvert > \frac{\epsilon m_\beta}{\sum_{i=1}^{k} w(\textbf{x}_i, y_i)} \right] \leq~\frac{\delta}{2(k+3)} $ 
\end{restatable}

\begin{proof}
We can use a Hoeffding bound, as $w(\textbf{x},y) \in [0,\beta_0^{k-2} \min_i {\beta_i}^{1-k}]$

\begin{align*}
& \Pr \lvert \sum_{i=1}^{m_\beta} w(\textbf{x}_i, y_i)  - m_\beta \mathbb{E}_{\cD_\beta}[w(\textbf{x},y)] \rvert > \epsilon m_\beta \mathbb{E}_{\cD_\beta}[w(\textbf{x},y)] \leq 2e^{\frac{-2\epsilon^2 m_\beta \mathbb{E}_{\cD_\beta}[w(\textbf{x},y)]^2}{\left(\beta_0^{k-2} \min_i {\beta_i}^{1-k} - 1\right)^2}}  \\
\implies &\Pr \left[\lvert \frac{1}{\mathbb{E}_{\cD_\beta}[w(\textbf{x},y)] }  - \frac{m_\beta}{\sum_{i=1}^{m_\beta} w(\textbf{x}_i, y_i)}\rvert > \frac{\epsilon m_\beta}{\sum_{i=1}^{m_\beta} w(\textbf{x}_i, y_i)} \right] \leq 2e^{\frac{-2\epsilon^2 m_\beta \mathbb{E}_{\cD_\beta}[w(\textbf{x},y)]^2}{\left(\beta_0^{k-2} \min_i {\beta_i}^{1-k} - 1\right)^2}} 
\end{align*}

Setting this to be less than $\delta$, we see that this requires $m_\beta \geq \frac{\left(\beta_0^{k-2} \min_i {\beta_i}^{1-k} - 1\right)^2}{-2\epsilon^2 \mathbb{E}_{\cD_\beta}[w(\textbf{x},y)]^2}\ln{\frac{\delta}{4(k+3))}}$ samples from the biased distribution:

\begin{align*}
&2e^{\frac{-2\epsilon^2 m \mathbb{E}_{\cD_\beta}[w(\textbf{x},y)]^2}{\left(\beta_0^{k-2} \min_i {\beta_i}^{k-1} - 1\right)^2}} \leq \frac{\delta}{2(k+3)} \\
\implies &\frac{-2\epsilon^2 m_\beta \mathbb{E}_{\cD_\beta}[w(\textbf{x},y)]^2}{\left(\beta_0^{k-2} \min_i {\beta_i}^{1-k} - 1\right)^2} \leq \ln{\frac{\delta}{4(k+3)}} \\
\implies &-2\epsilon^2 m_\beta \mathbb{E}_{\cD_\beta}[w(\textbf{x},y)]^2\leq \left(\beta_0^{k-2} \min_i {\beta_i}^{1-k} - 1\right)^2\ln{\frac{\delta}{4(k+3)}} \\
\implies & m_\beta \geq \frac{\left(\beta_0^{k-2} \min_i {\beta_i}^{1-k} - 1\right)^2}{-2\epsilon^2 \mathbb{E}_{\cD_\beta}[w(\textbf{x},y)]^2}\ln{\frac{4(k+3)}{\delta}} \\
\end{align*}

Finally, we can upper bound the right-hand side to say that if we have at least 
$\frac{\left(\beta_0^{k-2} \min_i {\beta_i}^{1-k} - 1\right)^2}{-2\epsilon^2 \mathbb{E}_{\cD_\beta}[w(\textbf{x},y)]^2}\ln{\frac{4(k+3)}{\delta}}$ samples, we achieve the desired bound.
\end{proof}

\begin{restatable}{lemma}{adjustedLoss} \label{lemma:A3}
$\frac{m_\beta}{\sum_{i=1}^{m_\beta} w(\textbf{x}_i, y_i)} L_{{S_\beta}\beta^{-1}} \leq 1$
\end{restatable}

\begin{proof}
\begin{align*}
&\frac{m_\beta}{\sum_{i=1}^{m_\beta} w(\textbf{x}_i, y_i)} L_{{S_\beta}\beta^{-1}} = \frac{m_\beta}{\sum_{i=1}^{m_\beta} w(\textbf{x}_i, y_i)}\frac{1}{m_\beta}\sum_{i=1}^{m_\beta} w(\textbf{x}_i, y_i) \mathbb{I}(h(\textbf{x}_i) \neq y_i) \leq \frac{\sum_{i=1}^{m_\beta} w(\textbf{x}_i, y_i)}{\sum_{i=1}^{m_\beta} w(\textbf{x}_i, y_i)} = 1 
\end{align*}
\end{proof}

\subsubsection{Part B}

\begin{restatable}{lemma}{lemmaEstimatingP}  \label{lemma:estimating_p}
Given $\epsilon, \delta \in (0,1)$, we require $m_i \geq \frac{3}{\epsilon^2p_i}\ln\frac{4(k+3)}{\delta}$ samples of unbiased data to estimate $\widehat{p_i}$ such that with probability $1-\frac{\delta}{2(k+3)}$,  
\begin{align}
\label{eq:p_error}
    \left| \widehat{p}_i-p_i\right| <p_i\epsilon
\end{align}
\end{restatable}

\begin{proof}
Let $X_1, X_2, \dots, X_m$ be $m$ independent Bernoulli trials with success probability $p_i$. Let $S = X_1 + X_2 + \dots + X_m$ be the number of successes, and $\widehat{p_i} = S/m_i$ be the empirical estimate of $p_i$.

By the multiplicative Chernoff Bound, for any $\epsilon > 0$ and $\delta > 0$, we have:
\begin{align}
\Pr\left(\left|\widehat{p_i} - p_i\right| \geq p_i\epsilon\right) \leq 2e^{-\frac{m\epsilon^2p_i}{3}}.
\end{align}

Thus, by setting $\frac{\delta}{2(k+3)} \geq 2e^{-\frac{m_i\epsilon^2p_i}{3}}$ and solving for $m_i$, we get $m_i \geq \frac{3}{\epsilon^2p_i}\ln\frac{4(k+3)}{\delta}$
\end{proof} 

\begin{restatable}{lemma}{lemmaEstimatingP_0}  \label{lemma:estimating_p_0}
Given $\epsilon, \delta \in (0,1)$, we require $m \geq \frac{3}{\epsilon^2p_0}\ln\frac{4(k+3)}{\delta}$ samples of unbiased data to estimate $\widehat{p_0}$ such that with probability $1-\frac{\delta}{2(k+3)}$,  
\begin{align}
\label{eq:p_0_error}
    \left| \widehat{p_0}-p_0\right| <p_0\epsilon
\end{align}
\end{restatable}

\begin{proof}
The proof follows identically to that of Lemma~\ref{lemma:estimating_p}.
\end{proof}

\begin{restatable}{lemma}{lemmaEstimatingPBeta}
\label{lemma:estimating_p_beta}
Given $\epsilon, \delta \in (0,1)$, 
we require $m_{\beta_i} \geq \frac{3}{\epsilon^2 p_{\beta_i}} \ln \dfrac{4(k+3)}{\delta}$ samples of biased data to estimate $p_{\beta_i}$ such that with probability $1-\frac{\delta}{2(k+3)}$,

$$\vert \widehat{p_{\beta_i}} - p_{\beta_i}\rvert < p_{\beta_i} \epsilon$$
\end{restatable}

\begin{proof}
Let $Z_j$ be the indicator that sample $j$ in group $i$ is a positive sample, and let $R = \sum_{j=1}^{m_{\beta_i}} Z_j$ be their sum. Then, $\mathbb{E}[Z] = p_{\beta_i}$ and $\mathbb{E}[S]=m_{\beta_i}p_{\beta_i}$, where $m_{\beta_i}$ is the number of samples in group $i$ of the biased data. Applying a multiplicative Chernoff bound gives:
\begin{align}
\mathbb{P}\left(\left|\widehat{p_{\beta_i}} - p_{\beta_i}\right| \geq \epsilon p_{\beta_i}\right) &\leq 2\exp\left(-\frac{\epsilon^2 m_{\beta_i} p_{\beta_i}}{3}\right)
\end{align}
Setting this probability to be less than $\delta$ gives
\begin{align*}
    &2\exp\left(-\frac{\epsilon^2 m_{\beta_i}p_{\beta_i}}{3}\right) \leq \frac{\delta}{2(k+3)} \\
    \implies 
    &-\epsilon^2 m_{\beta_i} p_{\beta_i} \leq 3\ln{\frac{\delta}{4(k+3)}} \\
    \implies &m_{\beta_i} \geq \frac{3\ln{\frac{4(k+3)}{\delta}}}{\epsilon^2 p_{\beta_i}}
\end{align*}
\end{proof}

\begin{restatable}{lemma}{lemmaEstimatingPBeta_0}
\label{lemma:estimating_p_beta_0}
Given $\epsilon, \delta \in (0,1)$, 
we require $m_{\beta} \geq \frac{3}{\epsilon^2 p_{\beta_0}} \ln \dfrac{4(k+3)}{\delta}$ samples of biased data to estimate $p_{\beta_0}$ such that with probability $1-\frac{\delta}{2(k+3)}$,

$$\vert \widehat{p_{\beta_0}} - p_{\beta_0}\rvert < p_{\beta_0} \epsilon$$
\end{restatable}

\begin{proof}
The proof follows identically to in Lemma~\ref{lemma:estimating_p_beta}
\end{proof}

\begin{restatable}{lemma}{BetaInverse}\label{lemma:estimating_beta_inverse}
\label{lemma:intersectingBetaInverse}
Set $\delta > 0$ and $0 < \epsilon < \frac{1}{3}$. If  $m_{\beta_i} \geq \frac{3}{p_{\beta_i}\epsilon ^2}\ln \dfrac{4(k+3)}{\delta}$ and $m_i \geq \frac{3}{\epsilon^2p_i}\ln\frac{4(k+3)}{\delta}$, then
\[
      Pr\left[\lvert \widehat{\frac{1}{\beta_i}} - \frac{1}{\beta_i}\rvert > 3\epsilon \frac{1}{\beta_i}\right] \leq \frac{\delta}{2(k+3)}
      \]
\end{restatable}

\begin{proof}

With probability $1-\frac{\delta}{2(k+3)}$,
\begin{align*}
    \frac{\left(1-\epsilon \right)^2}{\left(1+\epsilon\right)^2}\frac{p_i \left( 1 - p_{\beta_i} \right)}{p_{\beta_i}\left(1 - p_i\right)}  \leq \frac{\widehat{p_i} \left( 1 - \widehat{p_{\beta_i}}\right)}{\widehat{p_{\beta_i}} \left(1 - \widehat{p_i} \right)} \leq \frac{\left(1+\epsilon \right)^2}{\left(1-\epsilon\right)^2}\frac{p_i \left( 1 - p_{\beta_i} \right)}{p_{\beta_i}\left(1 - p_i\right)} 
\end{align*}
%\begin{align*}
% \frac{(1-\epsilon)p_i}{(1+\epsilon)p_i \beta_i}\leq  \frac{\hat{p_i}}{\hat{p_i \beta_i}}\leq  \frac{(1+\epsilon)p_i}{(1-\epsilon)p_i \beta_i}
% \end{align*}

 For $\epsilon < \frac{1}{3}$,
 \begin{align*}
    \left(1-3\epsilon \right)^2\frac{p_i \left( 1 - p_{\beta_i} \right)}{p_{\beta_i}\left(1 - p_i \right)}  <  \frac{\widehat{p_i}\left( 1 - \widehat{p_{\beta_i}} \right)}{\widehat{p_{\beta_i}} \left(1 - \widehat{p_i} \right)} < \left(1+3\epsilon \right)^2 \frac{p_i \left( 1 - p_{\beta_i} \right)}{p_{\beta_i}\left(1 - p_i \right)} 
 \end{align*}

Note that for $\epsilon < \frac{1}{3}$, $9\epsilon^2 < 3\epsilon$, so $1+ 6\epsilon + 9\epsilon^2 \leq 1+9\epsilon$ and $1 - 6\epsilon + 9\epsilon^2>1-9\epsilon$
 
 %$(1+3\epsilon)^2 = 1 + 6\epsilon + 9\epsilon^2 \leq $
 %\begin{align*}
%(1-3\epsilon) \frac{p_i}{p_i \beta_i}\leq  \frac{\hat{p_i}}{\hat{p_i \beta_i}}\leq (1+3\epsilon) \frac{p_i}{p_i \beta_i}
% \end{align*}
% \begin{align*}
%    (1-3\epsilon + 9\epsilon^2)\frac{p_i \left( 1 - p_i^* \right)}{p_i^*(1 - p_i)}  \leq  \frac{\widehat{p_i} \left( 1 - \widehat{p_i^*} \right)}{\widehat{p_i^*}(1 - \widehat{p_i})} \leq (1+3\epsilon + 9\epsilon^2) \frac{p_i \left( 1 - p_i^* \right)}{p_i^*(1 - p_i)} 
% \end{align*}
%\frac{1}{\beta_i} = \frac{p_i \left( 1 - p_i^* \right)}{p_i^*(1 - p_i)}

 Then
 \begin{align*}
  &Pr\left[\lvert  \frac{\widehat{p_i} \left(1 - \widehat{p_{\beta_i}}\right)}{\widehat{p_{\beta_i}}\left(1 - \widehat{p_i}\right)} -\frac{p_i\left(1 - p_{\beta_i}\right)}{p_{\beta_i}\left(1 - p_i\right)}\rvert > 9\epsilon \frac{p_i\left(1 - p_{\beta_i}\right)}{p_{\beta_i} \left(1 - p_i\right)}\right] \leq \frac{\delta}{2(k+3)}
  \end{align*}
 %&Pr\left[\lvert \frac{\hat{p_i}}{\hat{p_i \beta_i}} - \frac{p_i}{p_i \beta_i}\rvert > 3\epsilon \frac{p_i}{p_i \beta_i}\right] \leq \frac{\delta}{2k+2} \\
 Recalling that $\frac{1}{\beta_i} = \frac{p_i\left(1 - p_{\beta_i}\right)}{p_{\beta_i}\left(1 - p_i\right)}$, this gives us:
 \begin{align*}
  \Pr\left[\lvert \widehat{\frac{1}{\beta_i}} - \frac{1}{\beta_i} \rvert > 9\epsilon \frac{1}{\beta_i} \right] \leq \frac{\delta}{2(k+3)}
 \end{align*}
\end{proof}

\begin{restatable}{lemma}{betaProducts}
\label{lemma:B9}
Set $\delta > 0$ and $0 < \epsilon < \frac{1}{9}$. If $m_{\beta_i} \geq \frac{3}{p_{\beta_i}\epsilon ^2}\ln \dfrac{4(k+3)}{\delta}$ and $m_i \geq \frac{3}{\epsilon^2p_i}\ln\frac{42(k+3))}{\delta}$ for all groups $G_i$
\[ \Pr \left[ \lvert {\prod_{i=1}^{k} \widehat{\frac{1}{\beta_i}}} - \prod_{i=1}^{k} \frac{1}{\beta_i} \rvert \geq 9\epsilon \prod_{i=1}^{k} \frac{1}{\beta_i} \right] \leq \frac{\delta}{2(k+3)}
\]
\end{restatable}
%\ed{Double check whether hat should be over full product or individual terms.}
\begin{proof}
    Let $\epsilon < \frac{1}{9}$. We can use our bound on the marginal $\widehat{1}{\beta_i}$ to bound the difference of this product from its expectation. %\ed{Need to say anything about independence assumption?} Lemma~\ref{lemma:intersectingBetaInverse} tells us that with probability $1-\delta$, $ \lvert \widehat{\frac{\beta_0}{\beta_i}} - \frac{\beta_0}{\beta_i} \rvert < 3\epsilon \frac{\beta_0}{\beta_i}$. Then, with probability $1-\delta$ and $\epsilon < 1/3$
    \begin{align*}
    \vert \prod_{i=1}^{k} \widehat{\frac{1}{\beta_i}} - \prod_{i=1}^{k} \frac{1}{\beta_i} \rvert &\leq \vert \prod_{i=1}^{k} \frac{1}{\beta_i} (1-9\epsilon) - \prod_{i=1}^{k} \frac{1}{\beta_i} \rvert \\&\leq \prod_{i=1}^{k} \frac{1}{\beta_i} \left(9\epsilon \right)^k \\&\leq 9\epsilon \prod_{i=1}^{k} \frac{1}{\beta_i}
    \end{align*}
\end{proof}

\begin{restatable}{lemma}{beta0}  \label{lemma:beta_0}
Given $\epsilon, \delta \in (0,1)$, we require $m_\beta \geq \frac{3}{p_{\beta_0} \epsilon ^2}\ln \dfrac{4(k+3)}{\delta}$ samples of biased data and $m \geq \frac{3}{\epsilon^2 p_0 }\ln\frac{4(k+3)}{\delta}$ samples of unbiased data to estimate $\widehat{\beta_0}$ such that with probability $1-\frac{\delta}{2(k+3)}$,  
\begin{align}
\label{eq:b_0_error}
    \left|\widehat{\beta_0}-\beta_0 \right| < \epsilon \beta_0
\end{align}
\end{restatable}

\begin{proof}
With probability $1-\frac{\delta}{2k+2}$,
\begin{align*}
    \frac{\left(1-\epsilon \right)}{\left(1+\epsilon\right)}\frac{p_{\beta_0}}{p_0}  \leq \frac{\widehat{p_{\beta_0}}}{\widehat{p_0}} \leq \frac{\left(1+\epsilon \right)}{\left(1-\epsilon\right)}\frac{p_{\beta_0}}{p_0} 
\end{align*}

Then for $\epsilon < \frac{1}{3}$
\begin{align*}
    \left(1-3\epsilon \right)\frac{p_{\beta_0}}{p_0}  \leq \frac{\widehat{p_{\beta_0}}}{\widehat{p_0}} \leq \left(1+3\epsilon \right)\frac{p_{\beta_0}}{p_0} 
\end{align*}

This implies

\begin{align*}
\Pr\left[\lvert \widehat{\beta_0} - \beta_0 \rvert > 3\epsilon \beta_0 \right] \leq \frac{\delta}{2(k+3)}
\end{align*}
 
\end{proof}
\begin{restatable}{lemma}{beta0_inverse}  \label{lemma:beta_0_inverse}
Given $\epsilon, \delta \in (0,1)$, we require $m_\beta \geq \frac{3}{p_{\beta_0} \epsilon ^2}\ln \dfrac{4(k+3)}{\delta}$ samples of biased data and $m \geq \frac{3}{\epsilon^2 p_0 }\ln\frac{4(k+3)}{\delta}$ samples of unbiased data to estimate $\widehat{\beta_0}$ such that with probability $1-\frac{\delta}{2(k+3)}$,  
\begin{align}
\label{eq:b_0_inverse_error}
    \left|\widehat{\frac{1}{\beta_0}}-\frac{1}{\beta_0} \right| < \epsilon \frac{1}{\beta_0}
\end{align}
\end{restatable}

\begin{proof}
The proof is identical to that of Lemma~\ref{lemma:beta_0}, where the quantity $\frac{\widehat{p_0}}{\widehat{p_{\beta_0}}}$ is bounded.
\end{proof}
\begin{comment}
\begin{restatable}{lemma}{beta0_Products}
\label{lemma:B9}
Set $\delta > 0$ and $0 < \epsilon < \frac{1}{3}$. If $m_{\beta} \geq \frac{3}{p_{\beta}\epsilon ^2}\ln \dfrac{2(2k+2)}{\delta}$ and $m \geq \frac{3}{\epsilon^2p}\ln\frac{2(2k+2)}{\delta}$
\[ \Pr \left[ \lvert {\prod_{i=1}^{k} \widehat{\beta_0}} - \prod_{i=1}^{k} \beta_0 \rvert \geq 3\epsilon \prod_{i=1}^{k} \beta_0 \right] \leq \frac{\delta}{2k+2}
\]
\end{restatable}
%\ed{Double check whether hat should be over full product or individual terms.}
\begin{proof}
    We can use our bound on the marginal $\widehat{\frac{\beta_0}{\beta_i}}$ to bound the difference of this product from its expectation. %\ed{Need to say anything about independence assumption?} Lemma~\ref{lemma:intersectingBetaInverse} tells us that with probability $1-\delta$, $ \lvert \widehat{\frac{\beta_0}{\beta_i}} - \frac{\beta_0}{\beta_i} \rvert < 3\epsilon \frac{\beta_0}{\beta_i}$. Then, with probability $1-\delta$ and $\epsilon < 1/3$
    \begin{align*}
    \vert \prod_{i=1}^{k} \widehat{\beta_0} - \prod_{i=1}^{k} \beta_0 \rvert &\leq \vert \prod_{i=1}^{k} \beta_0 (1-3\epsilon) - \prod_{i=1}^{k} \beta_0 \rvert \\&\leq \prod_{i=1}^{k} \beta_0 \left(3\epsilon \right)^k \\&\leq 3\epsilon \prod_{i=1}^{k} \beta_0
    \end{align*}
    \ed{Could keep this $k$ to tighten bound.}
\end{proof}
\end{comment}

\subsubsection{Part C}
 
\begin{restatable}{lemma}{expectations} \label{lemma:C1}
\label{lemma:expectations}
$\frac{1}{E_{D_\beta}[w(\textbf{x},y)]}\mathbb{E}_{\cD_\beta}[L_{S_{\beta} \beta^{-1}}(h)] = L_\cD(h)$
\end{restatable}

\begin{proof}
\begin{align*}
&\frac{1}{\mathbb{E}_{\cD_\beta}[w(\textbf{x},y)]}\mathbb{E}_{\cD_\beta}[L_{S_{\beta} \beta^{-1}}(h)] \\&= \frac{1}{\mathbb{E}_{\cD_\beta}[w(\textbf{x},y)]}\mathbb{E}_{\cD_\beta}[\frac{1}{m_\beta}\sum_{i=1}^{m_\beta} w(\textbf{x}_i, y_i) \mathbb{I}(h(\textbf{x}_i) \neq y_i)]\\
&=\frac{1}{\mathbb{E}_{\cD_\beta}[w(\textbf{x},y)]} \mathbb{E}_{\cD_\beta}[w(\textbf{x}, y) \mathbb{I}(h(\textbf{x}) \neq y)]\\
&= \sum_{x,y \sim \cD_\beta}w(\textbf{x}, y) \mathbb{I}(h(\textbf{x}) \neq y  )\frac{p_{D_\beta}(\textbf{x},y)}{\mathbb{E}_{\cD_\beta}[w(\textbf{x},y)]} \\
&= \sum_{x,y \sim \cD}\mathbb{I}(h(\textbf{x}_i) \neq y_i)p_D(\textbf{x},y)\\
&=L_\cD(h)
\end{align*}
\end{proof}

\begin{restatable}{lemma}{lossOnD}\label{lemma:C2}
With $m_\beta \geq \frac{1}{{2\epsilon^2}}\ln{\frac{2|H|(2k+2)}{\delta}}\left( \frac{\beta_0^{k-2} \min_i {\beta_i}^{1-k}}{\mathbb{E}_{\cD_\beta}[w(\textbf{x}_i, y_i)]}\right)^2$ samples,
\begin{align*}
&\Pr\left[\lvert \frac{1}{\mathbb{E}_{\cD_\beta}[w(\textbf{x}_i, y_i)]} L_{S_\beta \beta^{-1}}(h) - L_\cD(h) \rvert > \epsilon L_\cD(h)\right] \\ &\leq \frac{\delta}{2(k+3)}
\end{align*}
\end{restatable}

\begin{proof}
From Lemma~\ref{lemma:expectations}, we know that $\frac{1}{\mathbb{E}_{\cD_\beta}[w(\textbf{x}_i, y_i)]} L_{S_\beta \beta^{-1}}(h)$ is bounded between 0 and $\frac{\beta_0^{k-2} \min_i {\beta_i}^{1-k}}{\mathbb{E}_{\cD_\beta}[w(\textbf{x}_i, y_i)]}$. Applying a multiplicative Hoeffding bound gives 
\begin{align*}
&\Pr\left[\lvert \frac{1}{\mathbb{E}_{\cD_\beta}[w(\textbf{x}_i, y_i)]} L_{S_\beta \beta^{-1}}(h) - L_\cD(h) \rvert > \epsilon L_\cD(h)\right] \\&\leq 2e^{-2\epsilon^2(\mathbb{E}_{\cD_\beta}[w(\textbf{x}_i, y_i)])^2 m_\beta \left( \beta_0^{k-2} \min_i {\beta_i}^{1-k} \right)^{-2}}
\end{align*}

Applying a union bound over the VC dimension of the class $H$, $|H|$, gives:
\begin{align*}
&\Pr\left[\lvert \frac{1}{\mathbb{E}_{\cD_\beta}[w(\textbf{x}_i, y_i)]} L_{S_\beta \beta^{-1}}(h) - L_\cD(h) \rvert > \epsilon L_\cD(h)\right] \\&\leq 2|H|e^{-2\epsilon^2(\mathbb{E}_{\cD_\beta}[w(\textbf{x}_i, y_i)])^2 m_\beta \left( \beta_0^{k-2} \min_i {\beta_i}^{1-k} \right)^{-2}}
\end{align*}

Upper bounding by $\frac{\delta}{2(k+3)}$ gives:
\begin{align*}
&2|H|e^{-2\epsilon^2(\mathbb{E}_{\cD_\beta}[w(\textbf{x}_i, y_i)])^2 m_\beta \left( \beta_0^{k-2} \min_i {\beta_i}^{1-k} \right)^{-2}} \leq \frac{\delta}{2(k+3)} \\
&\implies e^{-2\epsilon^2(\mathbb{E}_{\cD_\beta}[w(\textbf{x}_i, y_i)])^2 m_\beta \left( \beta_0^{k-2} \min_i {\beta_i}^{1-k} \right)^{-2}} \leq \frac{\delta}{4|H|(k+3)} \\
&\implies -2\epsilon^2(\mathbb{E}_{\cD_\beta}[w(\textbf{x}_i, y_i)])^2 m_\beta \left( \beta_0^{k-2} \min_i {\beta_i}^{1-k} \right)^{-2} \leq \ln{\frac{\delta}{4|H|(k+3)}} \\
&\implies (\mathbb{E}_{\cD_\beta}[w(\textbf{x}_i, y_i)])^2 m_\beta \left( \beta_0^{k-2} \min_i {\beta_i}^{1-k} \right)^{-2} \geq\frac{1}{-2\epsilon^2}\ln{\frac{\delta}{4|H|(k+3)}} \\
&\implies m_\beta \geq\frac{1}{2\epsilon^2}\ln{\frac{4|H|(k+3)}{\delta}}\left( \frac{\beta_0^{k-2} \min_i {\beta_i}^{1-k}}{\mathbb{E}_{\cD_\beta}[w(\textbf{x}_i, y_i)]}\right)^2 \\
\end{align*}
\end{proof}
%Finally, because $\mathbb{E}_{\cD_\beta}[w(\textbf{x}_i, y_i)] \leq \beta_0^{k-2} \min_i {\beta_i}^{1-k}$, choosing 
%$m_\beta \geq \frac{1}{2\epsilon^2}\ln{\frac{2|H|(2k+2)}{\delta}}$ suffices.
%\ed{Is this logic right?}

\subsubsection{Part D}
In Lemma \ref{lemma:D1}, we want to prove that with probability $1-\delta$ and a certain number of samples, the difference between the reweighted risk of the learned hypothesis $h$ on the biased sample $S_\beta \widehat{\beta^{-1}}$ (normalized by $\sum_{i=1}^{m_{\beta}} \hat{w}(\textbf{x}_i, y_i)$ and reweighted with $\widehat{\beta^{-1}}$) and the risk of $h$ on the true distribution $\cD$ is bounded by an epsilon multiplicative factor of the risk on the true distribution.

To prove this, we start by expanding the expression using the triangle inequality. We split it into three terms: A, B, and C. A represents the difference between the normalized risk on the biased sample using $\hat{w}(\textbf{x},y)$ and the normalized risk using the expected weights $\mathbb{E}_{\cD_\beta}[w(\textbf{x},y)]$. B represents the difference between the normalized risk on the biased sample reweighted with $\widehat{\beta^{-1}}$ and the normalized risk on the biased sample reweighted with $\beta^{-1}$. C represents the difference between the normalized risk on the biased sample using $\beta^{-1}$ and the risk on the true distribution $\cD$.

Next, we simplify the expression further. We use the absolute value to ensure non-negativity and introduce additional terms to manipulate the expression. We apply the Hoeffding bound to bound the first term, which involves the difference in weights. We also introduce the empirical risk $L_S(h)$ and manipulate the terms to arrive at the final inequality.

The final inequality shows that the difference between the two risks is bounded by $2\epsilon +  9\epsilon \beta_0^{k-2} \min_i {\beta_i}^{k-1}$. This inequality holds with probability $1-\delta$ and depends on the number of samples used.

By proving this lemma, we establish a bound on the difference between the reweighted risk of the learned hypothesis on the biased sample and the risk on the true distribution. This helps us understand the generalization performance of the learned hypothesis in the presence of bias.

\begin{lemma}
\label{lemma:D1}
Let $\delta > 0$, $\frac{1}{9} > \epsilon > 0$. If $m_\beta \geq \frac{\left(\beta_0^{k-2} \min_i {\beta_i}^{1-k} - 1\right)^2}{{2\epsilon^2}}\ln{\frac{4|H|(k+3)}{\delta}}$, $m_{\beta_i} \geq \frac{3}{p_{\beta_i}\epsilon ^2} \ln \dfrac{4(k+3)}{\delta}$ and $m_i \geq \frac{3}{\epsilon^2p_i}\ln\frac{4(k+3)}{\delta}$ for all groups $G_i$, then with probability $1-\delta$,
\[
\lvert \frac{m_\beta}{\sum_{i=1}^{m_{\beta}} \hat{w}(\textbf{x}_i, y_i)}L_{S_\beta \widehat{\beta^{-1}}}(h) - L_\cD(h)\rvert \leq 2\epsilon + 9\epsilon \prod_{i=1}^{k} \beta_i^{-1}
\]
\end{lemma}

\begin{proof}
We can use the triangle inequality to expand this expression into three terms:
\begin{align*}
\lvert \frac{m_\beta}{\sum_{i=1}^{m_\beta} \hat{w}(\textbf{x}_i, y_i)}L_{S_\beta \widehat{\beta^{-1}}}(h) - L_\cD(h)\rvert &\leq \lvert \frac{m_\beta}{\sum_{i=1}^{m_\beta} \hat{w}(\textbf{x}_i, y_i)}L_{S_\beta \widehat{\beta^{-1}}}(h) - \frac{1}{\mathbb{E}_{\cD_\beta}[w(\textbf{x},y)]}L_{S_\beta \widehat{\beta^{-1}}}(h)\rvert \\& + \lvert \frac{1}{\mathbb{E}_{\cD_\beta}[w(\textbf{x},y)]}L_{S_\beta \widehat{\beta^{-1}}}(h) - \frac{1}{\mathbb{E}_{\cD_\beta}[w(\textbf{x},y)]}L_{S_\beta \beta^{-1}}(h)\rvert + \lvert \frac{1}{\mathbb{E}_{\cD_\beta}[w(\textbf{x}_i, y_i)]} L_{S_\beta \beta^{-1}}(h) - L_\cD(h) \rvert
\end{align*}  

Simplifying the expression above, we have (with probability $1-\delta$)
\begin{align*}
&\lvert \frac{m_\beta}{\sum_{i=1}^{m_\beta} \hat{w}(\textbf{x}_i, y_i)}L_{S_\beta \widehat{\beta^{-1}}}(h) - L_\cD(h)\rvert 
\\&\leq \lvert \frac{m_\beta}{\sum_{i=1}^{m_\beta} \hat{w}(\textbf{x}_i, y_i)} - \frac{1}{\mathbb{E}_{\cD_\beta}[w(\textbf{x},y)]} \rvert L_{S_\beta \widehat{\beta^{-1}}}(h)
+ \lvert L_{S_\beta \widehat{\beta^{-1}}}(h) - L_{S_\beta \beta^{-1}}(h)\rvert \frac{1}{\mathbb{E}_{\cD_\beta}[w(\textbf{x},y)]} 
+ \lvert \frac{1}{\mathbb{E}_{\cD_\beta}[w(\textbf{x}, y)]} L_{S_\beta \beta^{-1}}(h) - L_\cD(h) \rvert \\
&\leq \frac{\epsilon m_\beta}{\sum_{i=1}^{m_\beta} w(\textbf{x}_i, y_i)} L_{S_\beta \widehat{\beta^{-1}}}(h)
+ \lvert \frac{1}{m_\beta} \sum_{i=1}^{m_\beta} \mathbb{I}[h(\textbf{x}_i)\neq y_i] (\hat{w}(\textbf{x}_i, y_i) - w(\textbf{x}_i, y_i)\rvert \frac{1}{\mathbb{E}_{\cD_\beta}[w(\textbf{x},y)]} + \lvert \frac{1}{\mathbb{E}_{\cD_\beta}[w(\textbf{x}, y)]} L_{S_\beta \beta^{-1}}(h) - L_\cD(h) \rvert
\end{align*}

Now we need to bound each term in this sum, which is done using the results from lemmas~\ref{lemma:A2}-\ref{lemma:C2}.
%&= \frac{\epsilon m_\beta}{\sum_{i=1}^{m_\beta} w(\textbf{x}_i, y_i)} L_{S_\beta \widehat{\beta^{-1}}}(h)
%+ \frac{L_S(h)}{\mathbb{E}_{\cD_\beta}[w(\textbf{x},y)]}\lvert \widehat{\beta_0^{-1}}\prod_{i=1}^{k}\widehat{\frac{\beta_0}{\beta_i^}} - \beta_0^{-1} \prod_{i=1}^{k} \beta_i^{-1}\rvert + \lvert \frac{1}{\mathbb{E}_{\cD_\beta}[w(\textbf{x},y)]} L_{S_\beta \beta^{-1}}(h) - L_\cD(h) \rvert \\
\begin{align*}
&\leq 2\epsilon + 9\epsilon \beta_0^{k-2} \min_i {\beta_i}^{k-1}
\end{align*}  
%\ed{Use Cauchy Schwart here to separate for new $\beta_0$ term.}
%Since $\frac{\epsilon \prod_{i=1}^{k} \beta_i}{11} \leq \frac{\epsilon}{2 + \frac{9}{\prod_{i=1}^{k} \beta_i}}$, we can replace $\epsilon$ with $\frac{\epsilon \prod_{i=1}^{k} \beta_i}{11}$ in our sample size bounds from the remaining lemmas, giving us:

%\begin{align*}
%\lvert \frac{m_\beta}{\sum_{i=1}^{m_\beta} \hat{w}(\textbf{x}_i, y_i)}L_{S_\beta \hat{\beta}^{-1}}(h) - L_\cD(h)\rvert \leq \epsilon
%\end{align*}
\end{proof}

\mainTheorem*

\begin{proof}
To prove Theorem~\ref{theorem:mainTheorem}, we employ Lemmas~\ref{lemma:A2}-\ref{lemma:D1}. The result follows from choosing $\epsilon:=\frac{\epsilon'}{11 \beta_0^{k-2} \min_i {\beta_i}^{1-k} }$. To see this, note that $\frac{\epsilon'}{11 \beta_0^{k-2} \min_i {\beta_i}^{1-k} } \leq \frac{\epsilon'}{2 + 9  \beta_0^{k-2} \min_i {\beta_i}^{1-k} {\prod_{i=1}^{k} \beta_i}}$. We can then parameterize Lemma~\ref{lemma:D1} with $\epsilon'$ instead of $\epsilon$ and then replace $\epsilon$ with $\frac{\epsilon'}{11 \beta_0^{k-2} \min_i {\beta_i}^{1-k} }$ in our sample size bounds from the remaining lemmas, giving us:

\begin{align*}
\lvert \frac{m_\beta}{\sum_{i=1}^{m_\beta} \hat{w}(\textbf{x}_i, y_i)}L_{S_\beta \widehat{\beta^{-1}}}(h) - L_\cD(h)\rvert \leq \epsilon
\end{align*}
\end{proof}

\subsection{Proof of Theorem~\ref{theorem:secondTheorem}}
\secondTheorem*

\begin{proof}
We aim to demonstrate that the Intersectional Bias Learning Algorithm adheres to the Agnostic PAC learning definition. Let's begin by revisiting the definition and the theorem's guarantees.

\textbf{Agnostic PAC Learning Definition (Informal):}
A hypothesis class \( \mathcal{H} \) is agnostic PAC learnable if for every distribution over the features and labels (where labels are binary), and for \( \epsilon, \delta > 0 \), there exists a sample size \( m \) such that for any sample \( S \) of size at least \( m \), the algorithm produces a hypothesis \( h \) satisfying:
\[ \mathbb{P}_{(x,y) \sim \mathcal{D}}[h(x) \neq y] \leq \min_{h' \in \mathcal{H}} \mathbb{P}_{(x,y) \sim \mathcal{D}}[h'(x) \neq y] + \epsilon \]
with probability at least \( 1-\delta \).

The left side of the definition represents the error of the hypothesis \( h \) on the entire distribution \( \mathcal{D} \), while the right side represents the error of the best hypothesis in the class \( \mathcal{H} \) on the distribution \( \mathcal{D} \).

From Theorem~\ref{theorem:mainTheorem}, we see that by using the following sample sizes:
\begin{align*}
    &m_\beta \geq \frac{11^2}{{2\epsilon^2 \prod_{j=1}^{k} \beta_j ^2}}\ln{\frac{4|H|(k+3)}{\delta}}\\ 
    &m_{\beta_i} \geq \frac{3 \cdot 11^2}{p_{\beta_i}\epsilon^2 \prod_{j=1}^{k} \beta_j ^2} \ln \dfrac{4(k+3)}{\delta}\\
    &m_i \geq \frac{3 \cdot 11^2}{(\epsilon/2)^2 \prod_{j=1}^{k} \beta_j ^2 p_i}\ln\frac{4(k+3)}{\delta}
\end{align*}
we can assure that with probability $1-\delta$, our estimate of the true loss errs by at most $\epsilon$.

\[ \lvert \frac{m_\beta}{\sum_{i=1}^{m_{\beta}} \hat{w}(\textbf{x}_i, y_i)}L_{S_\beta \widehat{\beta^{-1}}}(h) - L_\cD(h)\rvert \leq \epsilon/2 \]

Let $h^* = \text{argmin}_h \frac{m_\beta}{\sum_{i=1}^{m_{\beta}} \hat{w}(\textbf{x}_i, y_i)}L_{S_\beta \widehat{\beta^{-1}}}(h)$. We will now break our analysis into two cases.

\textit{Case 1:}

If $h^*$ satisfies
\[\frac{m_\beta}{\sum_{i=1}^{m_{\beta}} \hat{w}(\textbf{x}_i, y_i)}L_{S_\beta \widehat{\beta^{-1}}}(h^*) \leq \min_{h' \in \mathcal{H}} L_{\mathcal{D}}\left(h'\right) + \epsilon
\] the proof is complete.

\textit{Case 2:} Assume for contradiction that 
\[\frac{m_\beta}{\sum_{i=1}^{m_{\beta}} \hat{w}(\textbf{x}_i, y_i)}L_{S_\beta \widehat{\beta^{-1}}}(h^*) > \min_{h' \in \mathcal{H}} L_{\mathcal{D}}\left(h'\right) + \epsilon
\]

However, due to the bounds established in Theorem ~\ref{theorem:mainTheorem}, there must exist another $h' \in \mathcal{H}$ such that

\[ 
\lvert \frac{m_\beta}{\sum_{i=1}^{m_{\beta}} \hat{w}(\textbf{x}_i, y_i)}L_{S_\beta \widehat{\beta^{-1}}}(h') - L_\cD(h')\rvert \leq \epsilon 
\]

implying $h'$ would have a lower weighted empirical risk on the biased distribution than $h^*$ contradicting the definition of $h^*$. Therefore, this can only happen with probability $\delta$, completing the proof. Thus under the conditions stated in Theorem ~\ref{theorem:mainTheorem} the Intersectional Bias Learning Algorithm produces a hypothesis $h$ that satisfies the requirements of agnostic PAC learnability with underrepresentation and intersectional bias, establishing the class $\mathcal{H}$ as agnostically PAC learnable under these conditions.
 
\end{proof}

\section{Additional Experiments}
Here, we show the performance of models trained with our reweighting approach in terms of F1 Score, Precision, and Recall. 

\begin{figure}[h]
\centering
\includegraphics[width=0.24\textwidth]{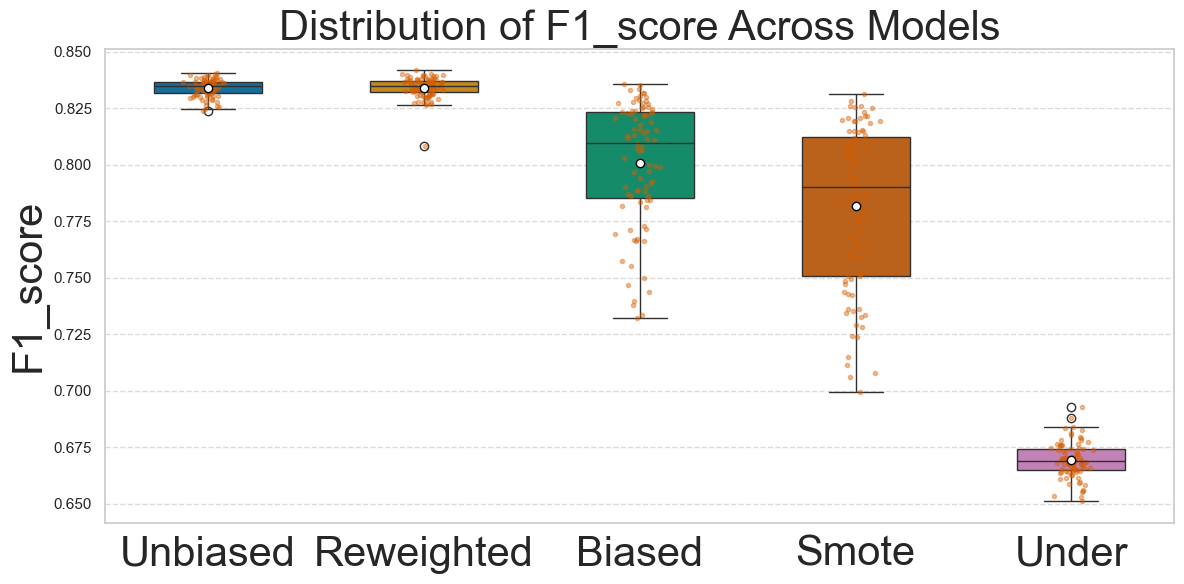}
\includegraphics[width=0.24\textwidth]{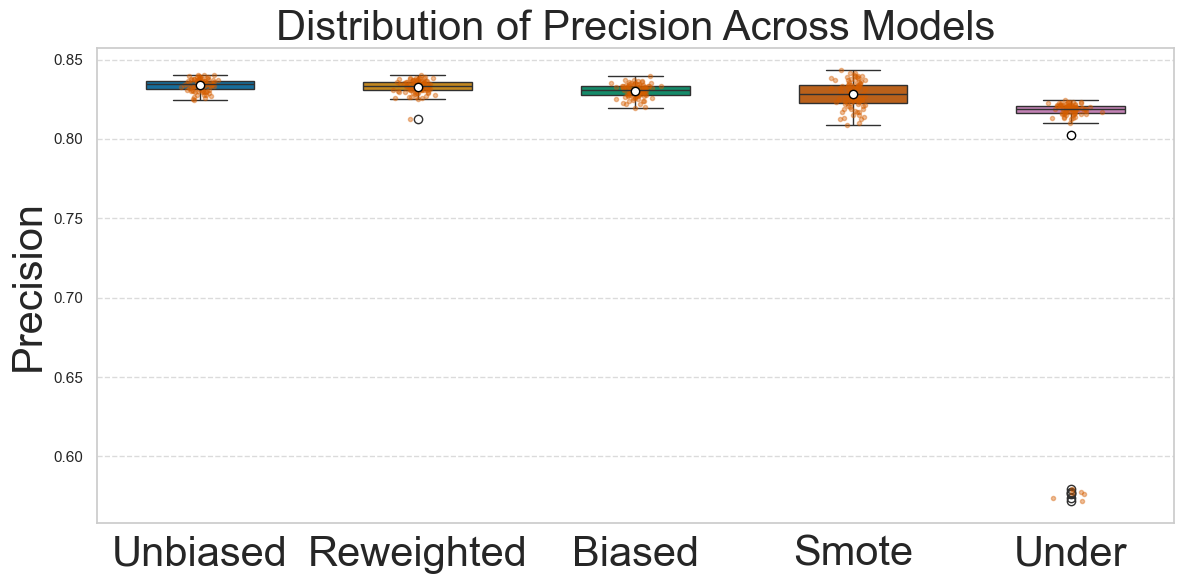}
\includegraphics[width=0.24\textwidth]{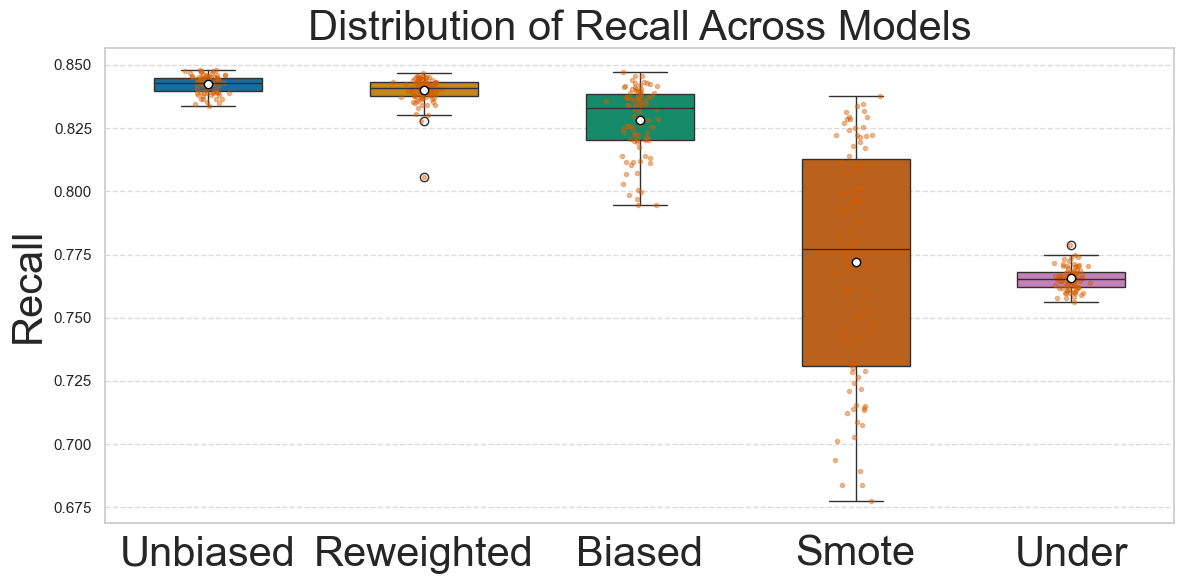}
\caption{Results on Adult Dataset}
\end{figure}

\begin{figure}[h]
\centering
\includegraphics[width=0.24\textwidth]{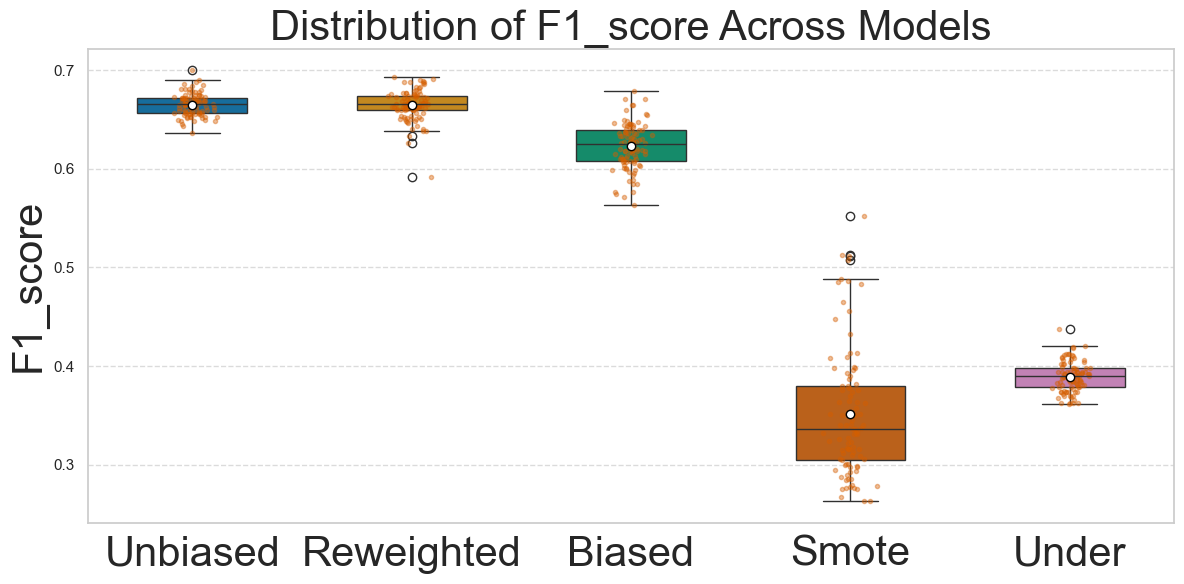}
\includegraphics[width=0.24\textwidth]{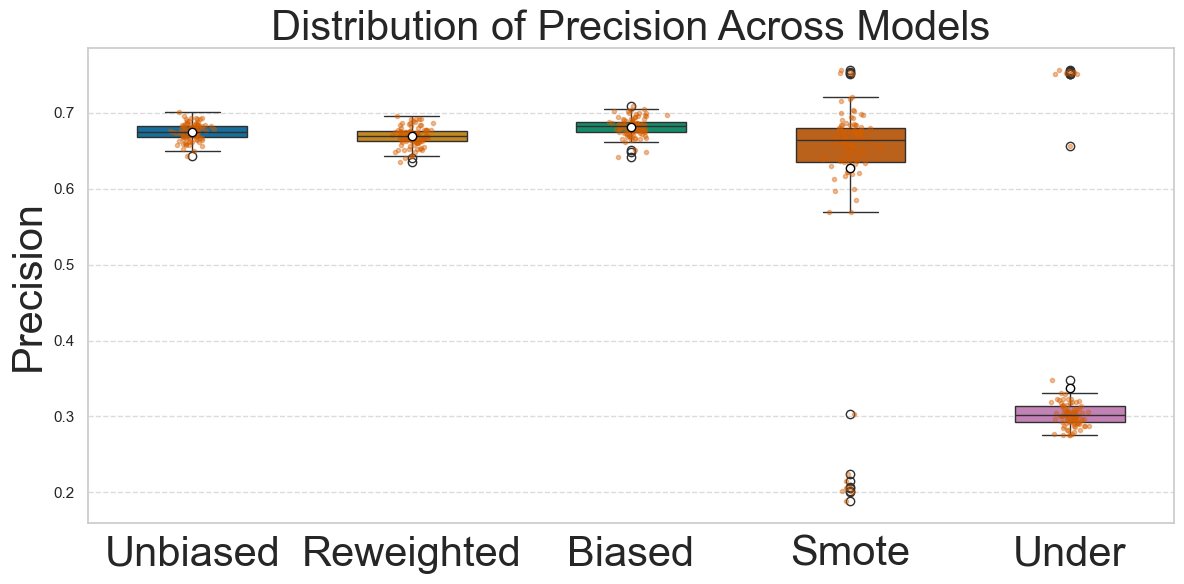}
\includegraphics[width=0.24\textwidth]{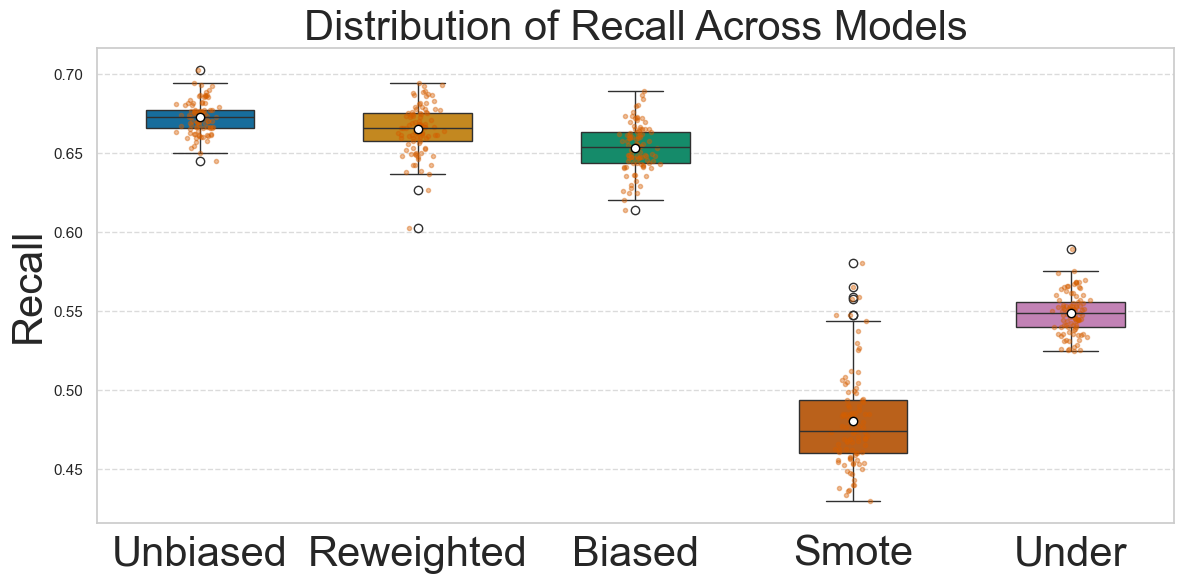}
\caption{Results on COMPAS Dataset}
\end{figure}

\begin{figure}[h]
\centering
\includegraphics[width=0.24\textwidth]{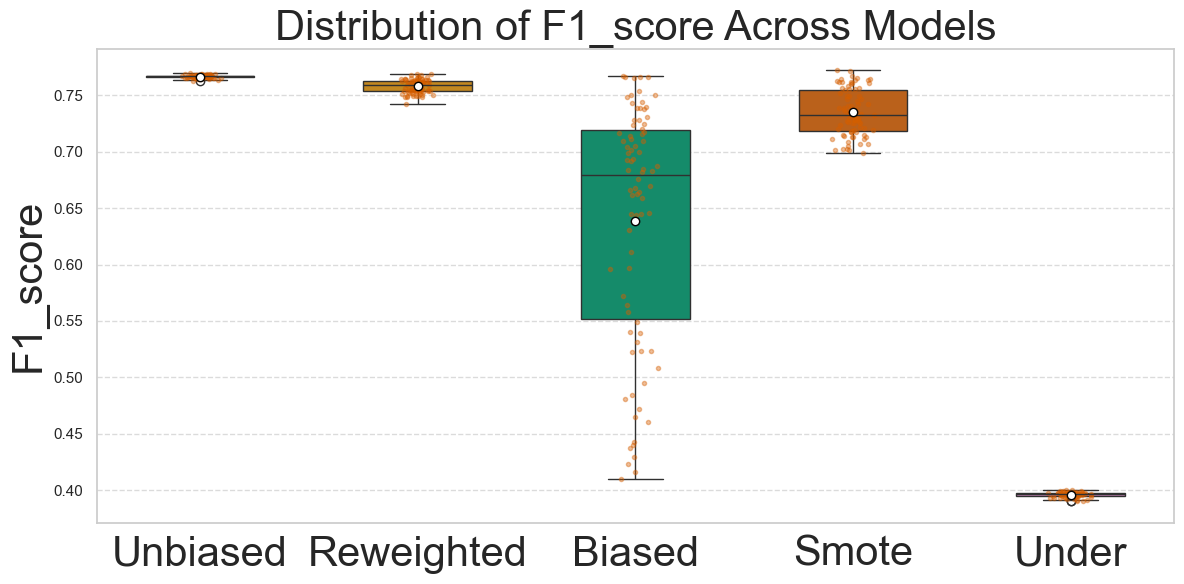}
\includegraphics[width=0.24\textwidth]{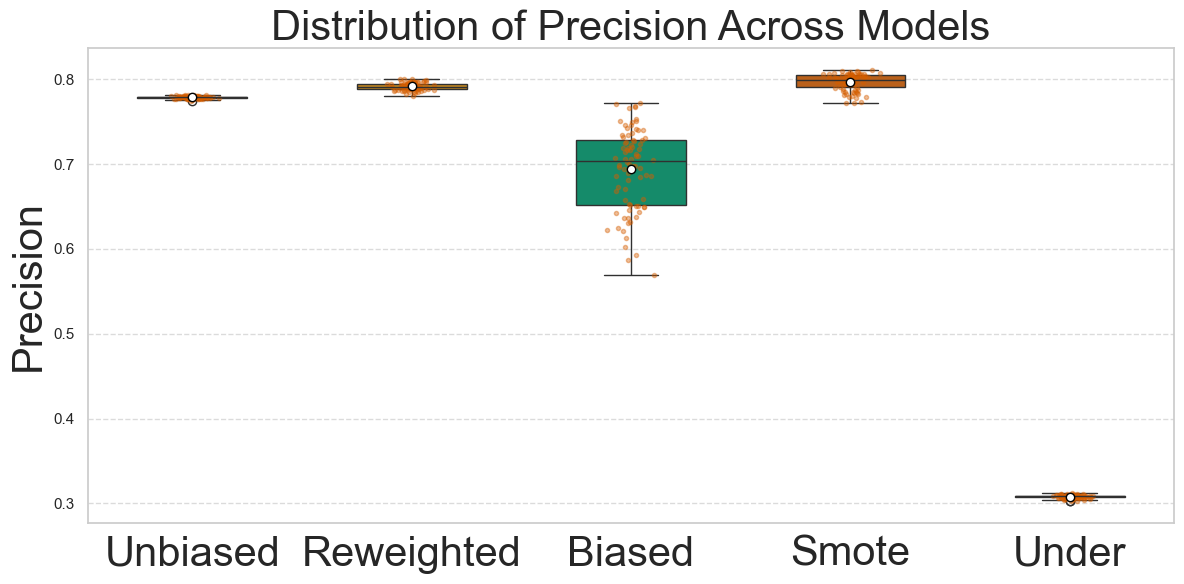}
\includegraphics[width=0.24\textwidth]{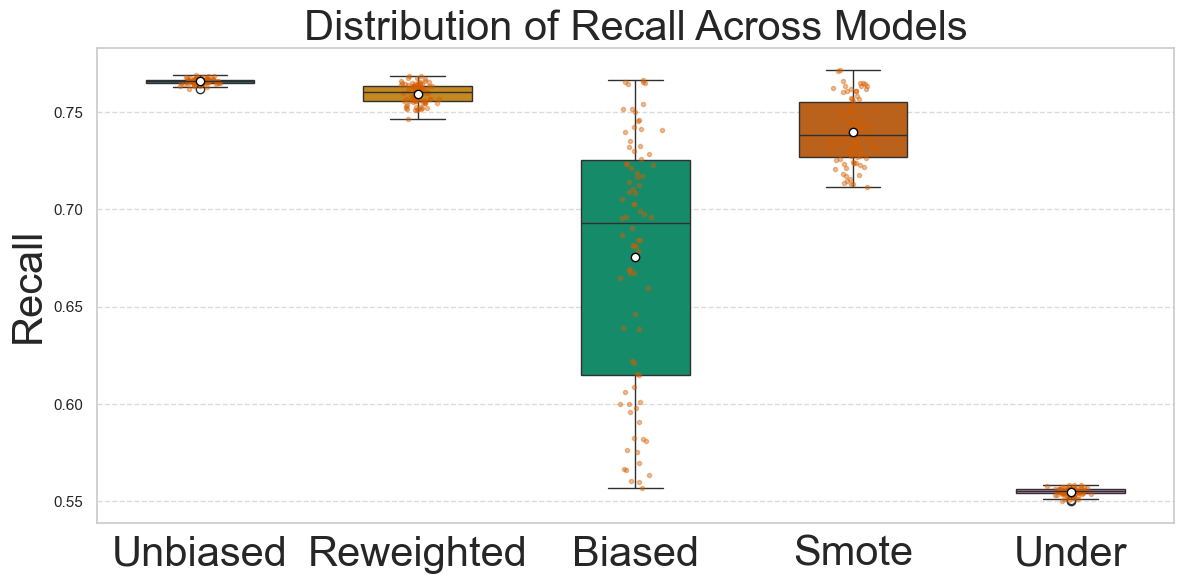}
\caption{Results on ACS Employment Dataset}
\end{figure}

\begin{figure}[h]
\centering
\includegraphics[width=0.24\textwidth]{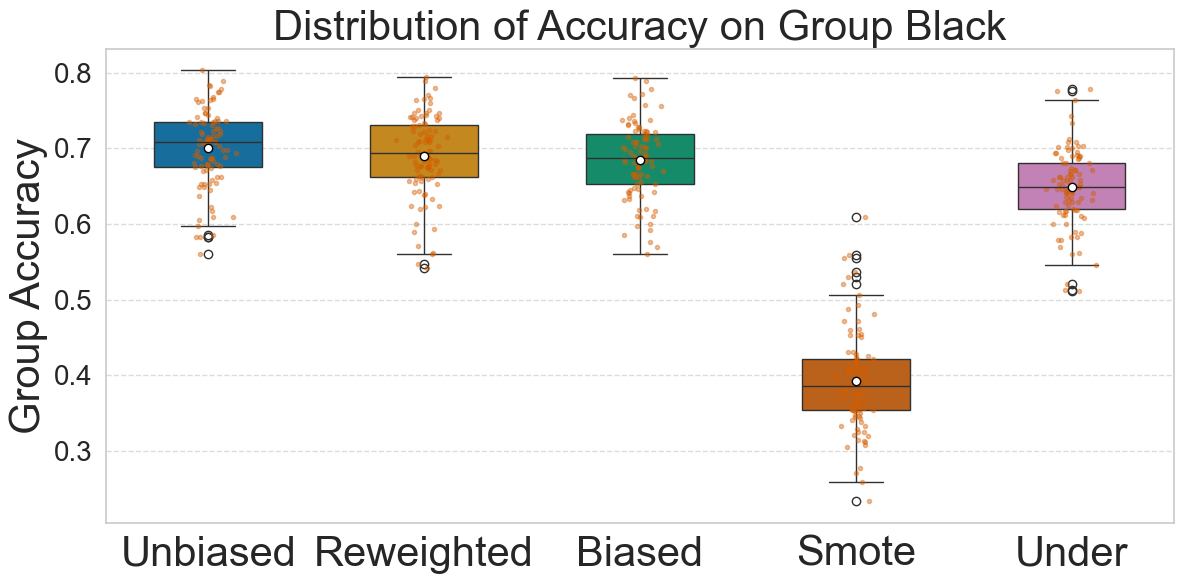}
\includegraphics[width=0.24\textwidth]{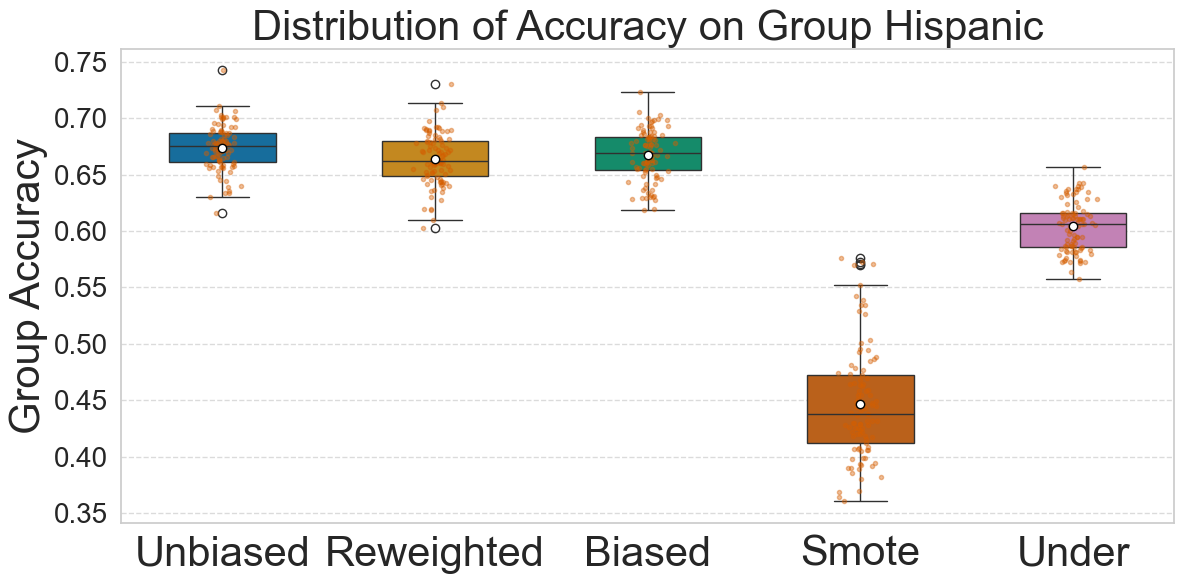}
\includegraphics[width=0.24\textwidth]{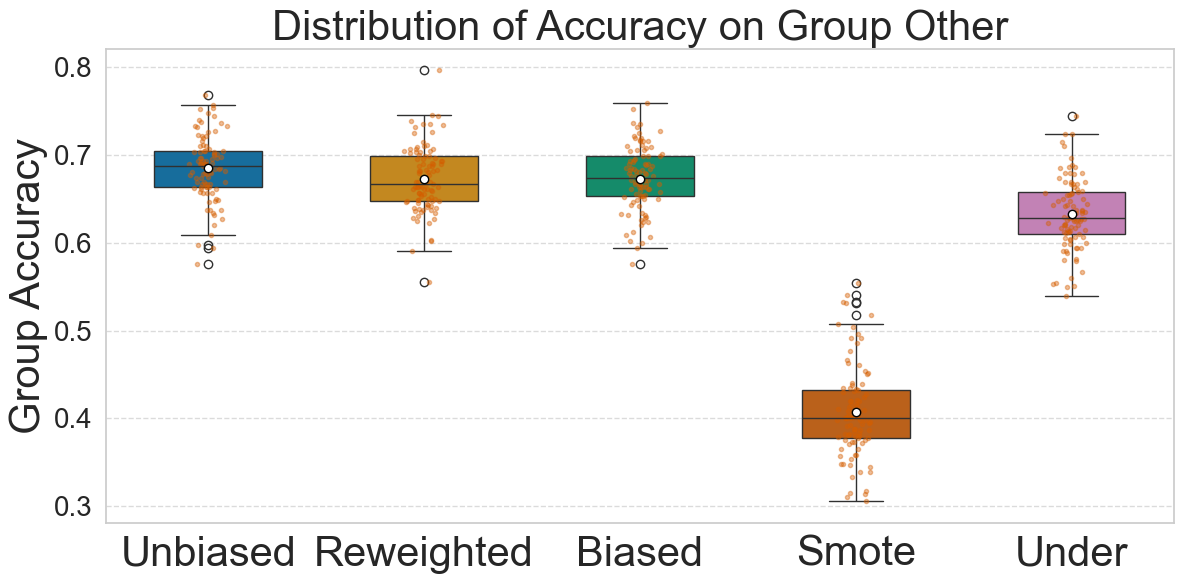}
\includegraphics[width=0.24\textwidth]{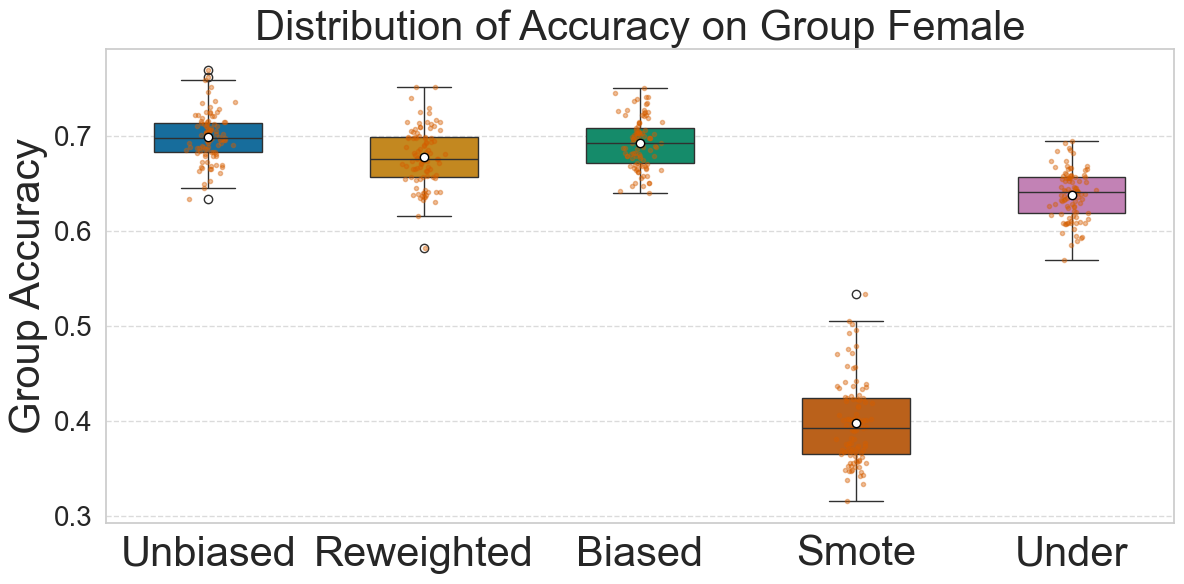}
\caption{Group Accuracies on COMPAS dataset}
\end{figure}

\subsection{Further Discussion of Figure~\ref{fig:groups_adult}}

\paragraph{\textbf{Pacific Islander (a)}}
For the Pacific Islander group, the Reweighted model demonstrates a median accuracy that closely aligns with the Unbiased model, suggesting that reweighting effectively corrects for the biases in this group. The Biased model shows a considerable decrease in median accuracy, pointing to the presence of bias in the original data. The SMOTE approach appears to have a slightly improved median accuracy over the Biased model but does not reach the level of the Reweighted model. The Under-sampling method exhibits a broad range of accuracies, indicated by the extensive spread of the interquartile range, implying inconsistency in the model's performance.

\paragraph{\textbf{Amer. Indian / Eskimo (b)}}
The Amer. Indian / Eskimo group's accuracy distribution reveals that the Reweighted model maintains a high median accuracy, very close to that of the Unbiased model, indicating the method's reliability in this group. The Biased model's performance is visibly lower, underscoring the impact of bias. Both SMOTE and Under-sampling show lower median accuracies compared to the Reweighted model, with SMOTE displaying a wide range of outcomes as indicated by the presence of several outliers.

\paragraph{\textbf{Other Race (c)}}
For individuals classified as Other Race, the Reweighted model again shows a median accuracy that competes closely with the Unbiased model, supporting the effectiveness of the reweighting technique. The Biased model falls short in median accuracy, and the SMOTE method, while somewhat effective, does not provide consistency, as evidenced by outliers. The Under-sampling method has the widest interquartile range, suggesting high variability in model performance.

\paragraph{\textbf{Black (d)}}
In the Black group, the Reweighted model approximates the Unbiased model's median accuracy effectively, suggesting that reweighting is a robust method to correct biases affecting this group. The Biased model underperforms in comparison, and the SMOTE method, despite a few outliers, shows an overall improvement over the Biased model. The Under-sampling method's accuracy distribution is wide, indicating variable outcomes.

\paragraph{\textbf{Female (e)}}
For the Female group, the Reweighted model's median accuracy is on par with the Unbiased model, highlighting the reweighting technique's capability to mitigate gender bias in predictions. The Biased model has a significantly lower median accuracy. The SMOTE method's median accuracy is slightly better than the Biased model but does not achieve the level of the Reweighted model. The Under-sampling method shows the least favorable median accuracy and the most considerable spread, which could indicate a failure to address gender biases effectively.

\subsection{Further Discussion of Figure~\ref{fig:groups_acs}}

\paragraph{\textbf{Black (a)}}
For the Black group, the Reweighted model achieves a median accuracy competitive with the Unbiased model, suggesting effective bias mitigation. The Biased model exhibits reduced accuracy, highlighting the impact of initial data bias. SMOTE and Under-sampling yield lower accuracies, with SMOTE showing a considerable spread, indicating variable performance.

\paragraph{\textbf{Amer. Indian (b)}}
In the Amer. Indian group, the Reweighted model closely approximates the accuracy of the Unbiased model, implying that the reweighting process is proficient in addressing biases in this group. The Biased model’s accuracy is notably lower, while SMOTE and Under-sampling again demonstrate variability and do not reach the performance levels of the Reweighted model.

\paragraph{\textbf{Alaska Native (c)}}
The Alaska Native group's accuracy distribution suggests that the Reweighted model closely mirrors the Unbiased model's performance, with a high median accuracy. The Biased model shows a significant decrease in accuracy. SMOTE and Under-sampling approaches do not achieve the consistency or accuracy of the Reweighted model, as indicated by their wider interquartile ranges.

\paragraph{\textbf{Other Native (d)}}
For the Other Native group, the Reweighted model's median accuracy is very close to that of the Unbiased model. The Biased model underperforms relative to the Reweighted model. Both SMOTE and Under-sampling approaches fail to match the Reweighted model's accuracy level, with SMOTE, in particular, displaying a broad range of outcomes.

\paragraph{\textbf{Asian (e)}}
In the Asian group, the Reweighted model again closely matches the Unbiased model in median accuracy, suggesting that the reweighting technique effectively corrects biases within this demographic. The Biased model lags in performance, and while SMOTE improves upon the Biased model, it does not achieve the consistency of the Reweighted model. Under-sampling shows the most considerable variance in outcomes.

\paragraph{\textbf{Native Hawaiian (f)}}
For Native Hawaiian individuals, the Reweighted model's performance is very close to the Unbiased model, indicating successful bias correction. The Biased model's performance is lower, while SMOTE and Under-sampling show lower and more variable accuracies.

\paragraph{\textbf{Female (g)}}
In the Female group, the Reweighted model's median accuracy is comparable to that of the Unbiased model, illustrating the reweighting technique's ability to mitigate gender bias. The Biased model's lower accuracy highlights initial biases, and neither SMOTE nor Under-sampling consistently reaches the Reweighted model's accuracy level.

\end{document}